\documentclass{article}

\usepackage[preprint, nonatbib]{neurips_2020}

\usepackage[utf8]{inputenc} 
\usepackage[T1]{fontenc}    
\usepackage{hyperref}       
\hypersetup{colorlinks,urlcolor=blue}
\usepackage{url}            
\usepackage{booktabs}       
\usepackage{amsfonts}       
\usepackage{nicefrac}       
\usepackage{microtype}      
\usepackage{graphicx}      
\usepackage{wrapfig}
\usepackage{xspace}
\usepackage{xcolor}
\usepackage{amsthm}
\usepackage{amsmath}
\usepackage{pgfplots}
\usepackage{pgfplotstable}
\usepackage{calc}
\usepackage{minitoc}
\usepackage[toc]{multitoc}
\usepackage{tocloft}
\usepackage{subfig}

\usepackage{amsmath,amsfonts,bm}

\def\eqref#1{equation~\ref{#1}}

\def\1{\bm{1}}

\DeclareMathAlphabet{\mathsfit}{\encodingdefault}{\sfdefault}{m}{sl}
\SetMathAlphabet{\mathsfit}{bold}{\encodingdefault}{\sfdefault}{bx}{n}

\DeclareMathOperator*{\argmax}{arg\,max}

\newcommand{\cA}{\mathcal{A}}
\newcommand{\cN}{\mathcal{N}}
\newcommand{\bE}{\mathbb{E}}
\newcommand{\bfone}{\mathbf{1}}

\newtheorem{theorem}{Theorem}

\newtheorem{lemma}{Lemma}
\newtheorem{definition}{Definition}

\definecolor{cobalt}{rgb}{0.0, 0.28, 0.67} 

\usepackage{floatrow}

\newcommand{\PAR}{\textsc{PAR}\xspace}
\newcommand{\MUN}{\textsc{MUN}\xspace}
\newcommand{\MAR}{\textsc{MAR}\xspace}
\newcommand{\VEN}{\textsc{VEN}\xspace}
\newcommand{\BUR}{\textsc{BUR}\xspace}
\newcommand{\PIE}{\textsc{PIE}\xspace}

\usepackage[noend]{algpseudocode}
\usepackage{algorithm}
\usepackage{algorithmicx}
\usepackage{authblk}
\usepackage[misc]{ifsym}
\usepackage[symbol]{footmisc}

\title{Learning to Play No-Press Diplomacy\\ with Best Response Policy Iteration}

\author[]{\bf Thomas~Anthony$^*$, Tom~Eccles$^*$, Andrea~Tacchetti, J\'anos~Kram\'ar, Ian Gemp, Thomas~C.~Hudson, Nicolas~Porcel, Marc~Lanctot, Julien~P\'erolat, Richard~Everett, Roman~Werpachowski, Satinder~Singh, Thore~Graepel and Yoram Bachrach}
\affil{DeepMind}

\newdimen{\algindent}
\algnewcommand\LeftComment[2]{%
\hspace{#1\algindent}$\triangleright$ {#2} \hfill %
}
\pgfplotsset{compat=1.3}

\begin{document}

\addtocontents{toc}{\protect\setcounter{tocdepth}{0}}

\maketitle

\vspace{-2mm}
\begin{abstract}
Recent advances in deep reinforcement learning (RL) have led to considerable progress in many 2-player zero-sum games, such as Go, Poker and Starcraft.
The purely adversarial nature of such games allows for conceptually simple and principled application of RL methods. 
However real-world settings are many-agent, and agent interactions are complex mixtures of common-interest and competitive aspects.
We consider Diplomacy, a 7-player board game designed to accentuate dilemmas resulting from many-agent interactions.
It also features a large combinatorial action space and simultaneous moves, which are challenging for RL algorithms. 
We propose a simple yet effective approximate best response operator, designed to handle large combinatorial action spaces and simultaneous moves.
We also introduce a family of policy iteration methods that approximate fictitious play.
With these methods, we successfully apply RL to Diplomacy:
we show that our agents convincingly outperform the previous state-of-the-art,
and game theoretic equilibrium analysis shows that the new process yields consistent improvements.

\end{abstract}

\section{Introduction}
\label{l_sect_intro}

Artificial Intelligence methods have achieved exceptionally strong competitive play in board games such as Go, Chess, Shogi~\cite{silver2016mastering,silver2017mastering,Silver18AlphaZero}, Hex~\cite{Anthony17ExIT}, Poker~\cite{moravvcik2017deepstack,brown2018superhuman}
and various video games~\cite{koutnik2013evolving,mnih2013playing,kempka2016vizdoom,resnick2018pommerman,guss2019minerl,vinyals2019grandmaster,jaderberg2019human,openai2019dota}.
Despite the scale, complexity and variety of these domains, a common focus in multi-agent environments is the class of 2-player (or 2-team) zero-sum games: ``1 vs 1'' contests.
There are several reasons: they are polynomial-time solvable, and solutions both grant worst-case guarantees and are interchangeable, so agents can approximately solve them in advance~\cite{vonNeumann1928,von1944theory}. 
Further, in this case conceptually simple adaptations of reinforcement learning (RL) algorithms often have theoretical guarantees.
However, most problems of interest are not purely adversarial: e.g. route planning around congestion, contract negotiations or interacting with clients all involve compromise and consideration of how preferences of group members coincide and/or conflict. 
Even when agents are self-interested, they may gain by coordinating and cooperating, so interacting among diverse groups of agents requires complex reasoning about others' goals and motivations.

We study {\bf Diplomacy}~\cite{calhamer1959diplomacy}, a 7-player board game.
The game was specifically designed to emphasize tensions between competition and cooperation, so it is particularly well-suited to the study of learning in mixed-motive settings. 
The game is played on a map of Europe partitioned into provinces. Each player controls multiple units, and each turn \emph{all} players move \emph{all} their units simultaneously. One unit may support another unit (owned by the same or another player), allowing it to overcome resistance by other units. 
Due to the inter-dependencies between units, players must coordinate the moves of their own units, and stand to gain by coordinating their moves with those of other players.
Figure~\ref{fig:example} depicts interactions among several players (moving and supporting units to/from provinces); we explain the basic rules in Section~\ref{l_sect_diplomacy_intro}. The original game allows cheap-talk negotiation between players before every turn. In this paper we focus on learning strategic interactions in a many-agent setting, so we consider the popular {\it No Press} variant, where no explicit communication is allowed.

\begin{wrapfigure}[15]{r}{0pt}
 \vspace{-10mm}
    \includegraphics[width=0.23\textwidth]{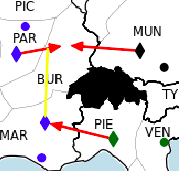}
  \caption{Simple example of interactions between several players' moves.}
  \label{fig:example}
\end{wrapfigure}

Diplomacy is particularly challenging for RL agents. 
First, it is a \textit{many-player} ($n>2$) game, so methods cannot rely on the simplifying properties of 2-player zero-sum games. 
Second, it features {\it simultaneous moves}, with a player choosing an action without knowledge of the actions chosen by others, which highlights reasoning about opponent strategies.
Finally, Diplomacy has a {\it large combinatorial action space}, with an estimated game-tree size of $10^{900}$, and $10^{21}$ to $10^{64}$ legal joint actions {\it per turn}.~\footnote{For comparison, Chess's game tree size is ${10}^{123}$, it has ${10}^{47}$ states, and fewer than 100 legal actions per turn. Estimates for Diplomacy are based on human data; see Appendix~\ref{app:size-estimates} for details.} 

Consequently, although Diplomacy AI has been studied since the 1980s~\cite{kraus1988diplomat,kraus1989automated}, until recently progress has relied on handcrafted rule-based systems, rather than learning.
Paquette et al.~\cite{paquette2019no} achieved a major breakthrough: they collected a dataset of $\sim150,000$ human Diplomacy games, and trained an agent, \textit{DipNet}, using a graph neural network (GNN) to imitate the moves in this dataset. 
This agent defeated previous state-of-the-art agents conclusively and by a wide margin. This is promising, as imitation learning can often be a useful starting point for RL methods.

However, to date RL has not been successfully applied to Diplomacy. For example, Paquette et al.~\cite{paquette2019no} used A2C initialised by their imitation learning agent, but this process did not improve performance as measured by the Trueskill rating system~\cite{herbrich2007trueskill}. This is unfortunate, as without agents able to optimise their incentives, we cannot study the effects of mixed-motives on many-agent learning dynamics, or how RL agents might account for other agents' incentives (e.g. with Opponent Shaping~\cite{foerster2018learning}).

{\bf Our Contribution:} We train RL agents to play No-Press Diplomacy, using a policy iteration (PI) approach.
We propose a simple yet scalable improvement operator, \textit{Sampled Best Responses} (SBR), which effectively handles Diplomacy's large combinatorial action space and simultaneous moves. 
We introduce versions of PI that approximate iterated best response and fictitious play (FP)~\cite{brown1951iterative,robinson1951iterative} methods. 
In Diplomacy, we show that our agents outperform the previous state-of-the-art both against reference populations and head-to-head. A game theoretic equilibrium analysis shows our process yields consistent improvements. 
We propose a few-shot exploitability metric, which our RL reduces, but agents remain fairly exploitable.
We perform a case-study of our methods in a simpler game, Blotto (Appendix~\ref{app:blotto}), and prove convergence results on FP in many-player games (Appendix~\ref{app:theory}). 

\section{Background and Related Work}

Game-playing has driven AI research since its inception: work on games delivered progress in search, RL and computing equilibria~\cite{samuel1959some,greenblatt1967greenblatt,knuth1975analysis,findler1977studies,berliner1980backgammon,rosenbloom1982world,schaeffer1992world,tesauro1994td,schraudolph1994temporal,genesereth2005general,ferrucci2012introduction}
leading to
prominent successes in Chess~\cite{campbell2002deep}, Go~\cite{silver2016mastering,silver2017mastering}, Poker~\cite{moravvcik2017deepstack,brown2018superhuman}, multi-agent control domains~\cite{bansal2017emergent,liu2019emergent,baker2019emergent,song2019arena,wang2019poet} and video games~\cite{koutnik2013evolving,mnih2013playing,kempka2016vizdoom}.
Recent work has also used deep
RL in many-player games. Some, such as Soccer, Dota and Capture-the-Flag, focus on two teams engaged in a zero-sum game but are cooperative between members of a team~\cite{liu2019emergent,openai2019dota,jaderberg2019human}. Others, e.g. Hanabi or Overcooked, are fully-cooperative~\cite{Foerster18BAD,Hu19SAD,lerer2019improving,bard2020hanabi,carroll2019utility}. Most relevantly, some work covers mixed-motive social dilemmas, with both competitive and collaborative elements~\cite{lowe2017multi,leibo2017multi,lerer2017maintaining,crandall2018cooperating,foerster2018learning,serrino2019finding, hughes2020learning}. 

There is little work on large, competitive, many-player settings, known to be harder than their
two-player counterparts~\cite{Chen053Nash,Daskalakis05Games}.
The exception is a remarkable recent success in many-player no-limit Poker that defeated human experts~\cite{brown2019superhuman}.
However, it uses expert abstractions and end-game solving to reduce the game tree size.
Moreover, in Poker players often fold early in the game, until only two players remain and
collusion is strictly prohibited, which reduces the effects of many-player interactions in practice.
In contrast, in Diplomacy 2-player situations are rare and alliances are crucial.

{\bf Diplomacy AI Research:}
Diplomacy is a long-standing AI challenge. Even in the simpler No-Press variant, AIs are far weaker than human players.
Rule-based Diplomacy agents were proposed in the 1980s and 1990s~\cite{kraus1988diplomat,hall1992thoughts,kraus1994negotiation,kraus1995designing}. 
Frameworks such as DAIDE~\cite{daide} DipGame~\cite{fabregues2009testbed} and BANDANA~\cite{deJonge2017} promoted development of stronger rule-based agents~\cite{johansson2005tactical,hal2010diplomacy,ferreira2015dipblue}. One work applied TD-learning with pattern weights~\cite{shapiro2002learning}, but was unable to produce a strong agent.
Negotiation for Computer Diplomacy is part of the Automated Negotiating Agents Competition~\cite{baarslag2012first,de2018challenge}.
We build on DipNet, the recent success in using a graph neural network to imitate human gameplay~\cite{paquette2019no}. DipNet outperformed previous agents, all rule-based systems, by a large margin. 
However, the authors found that A2C~\cite{mnih2016asynchronous} did not significantly improve DipNet. We replicated this result with our improved network architecture (see Appendix~\ref{app:a2c}).

\subsection{No-Press Diplomacy: Summary of Game Rules}
\label{l_sect_diplomacy_intro}

We provide an intentionally brief overview of the core game mechanics. For a longer introduction, see~\cite{paquette2019no}, and the rulebook~\cite{calhamer1959diplomacy}. 
The board is a map of Europe
partitioned into provinces; 34 provinces are {\bf supply centers} (SCs, dots in \PAR, \MUN, \MAR, and \VEN in Figure~\ref{fig:example}). 
Each player controls multiple units of a country. 
Units capture SCs by occupying the province. 
Owning more SCs allows a player to build more units; the game is won by owning a majority of the SCs.
Diplomacy has {\it simultaneous moves}: each turn every player writes down orders for all their units, without knowing what other players will do; players then reveal their moves, which are executed simultaneously.
The next position is fully determined by the moves and game rules, with no chance element (e.g. dice). 

Only one unit can occupy a province, and all units have equal strength. A unit may {\it hold} (guard its province) or {\it move} to an adjacent province. A unit may also {\it support} an adjacent unit to hold or move, to overcome opposition by enemy units. 
Using Figure~\ref{fig:example} as a running example, suppose France orders {\bf move} \PAR $\rightarrow$ \BUR;
if the unit in \MUN~{\bf holds} then the unit in \PAR enters \BUR, but if Germany also ordered \MUN $\rightarrow$ \BUR, both units `bounce' and neither enters \BUR. 
If France wanted to insist on entering to \BUR, they can order \MAR~{\bf support} \PAR $\rightarrow$ \BUR, which gives France 2 units versus Germany's 1,
so France's move order would succeed and Germany's would not. However, \MAR's support can be \textit{cut} by Italy moving \PIE $\rightarrow$ \MAR, leading to an equal-strength bounce as before. 

This example highlights elements that make Diplomacy unique and challenging. Due to simultaneous move resolution, players must anticipate how others will act and reflect these expectations in their own actions. 
Players must also use a stochastic policy (mixed strategy), as otherwise opponents could exploit their determinism. 
Finally, cooperation is essential: Germany would not have been able to prevent France from moving to \BUR without Italy's help. 
Diplomacy is specifically designed so that no player can win on their own without help from other players, so players {\it must} form alliances to achieve their ultimate goal.
In the No-Press variant, this causes pairwise interactions that differ substantially from zero-sum, so difficulties associated with mixed-motive games arise in practice.

\section{Reinforcement Learning Methods}
\label{l_sect_methods}

We adopt a policy iteration (PI) based approach, motivated by successes using PI for perfect information, sequential move, 2-player zero-sum board games \cite{Anthony17ExIT, Silver18AlphaZero}. We maintain a neural network policy $\hat{\pi}$ and a value function $\hat{V}$. Each iteration we create a dataset of games, with actions chosen by an improvement operator which uses a previous policy and value function to find a policy that defeats the previous policy. We then train our policy and value functions to predict the actions chosen by the improvement operator and the game results. The initial policy $\hat{\pi}^0$ and value function $\hat{V}^0$ imitate the human play dataset, similarly to DipNet~\cite{paquette2019no}, providing a stronger starting point for learning. 

Section~\ref{methods:sbr} describes SBR, our best response approximation method, tailored to handle the simultaneous move and combinatorial action space of Diplomacy. Section~\ref{methods:brpi} describes versions of PI that use SBR to approximate iterated best response and fictitious play algorithms. Our neural network training is an improved version of DipNet, described in Section \ref{l_sect_neural_arch} and Appendix~\ref{appendix:network}.

\subsection{Sampled Best Response (SBR)}
\label{methods:sbr}

Our PI methods use best response (BR) calculations as an improvement operator. Given a policy $\pi^b$ defined for all players, the BR for player $i$ is the policy $\pi^*_i$ that maximises the expected return for player $i$ against the opponent policies $\pi^b_{-i}$. A best response may not be a good policy to play as it can be arbitrarily poor against policies other than those it responds to. Nonetheless best responses are a useful tool, and we address convergence to equilibrium with the way we use BRs in PI (Section \ref{methods:brpi}). 

Diplomacy is far too large for exact best response calculation, so we propose a tractable approximation, Sampled Best Response (SBR, Algorithm~\ref{alg:sbr}). SBR makes three approximations: (1) we consider making a single-turn improvement to the policy in each state, rather than a full calculation over multiple turns of the game. (2) We only consider taking a small set of actions, sampled from a candidate policy. (3) We use Monte-Carlo estimates over opponent actions for candidate evaluation.

Consider calculating the value of some action $a_i$ for player $i$ against an opponent policy $\pi^\textrm{b}_{-i}$ (hereafter the \textit{base policy)}. Let $T(s, \mathbf{a})$ be the transition function of the game 
and $V^\pi(s)$ be the state-value function for a policy $\pi$. The 1-turn value to player $i$ of action $a_i$ in state $s$ is given by:
$$Q^{\pi^\textrm{b}}_i(a_i|s) = \mathbb{E}_{a_{-i}\sim\pi^\textrm{b}_{-i}} V^{\pi^\textrm{b}}_i(T(s, (a_i, a_{-i})))$$
We use the value network $\hat{V}$ instead of the exact state-value to get an estimated action-value $\hat{Q}^{\pi^b_i}(a_i|s)$. 

If the action space were small enough, we could exactly calculate $\argmax_{a_i} \hat{Q}^{\pi^b_i}(a_i|s)$, as a 1-turn best response. However, there are far too many actions to consider all of them.
Instead, we sample a set of candidate actions $A_i$ from a \textit{candidate policy} $\pi^\textrm{c}_i(s)$, and only consider these candidates for our approximate best response.
Now the strength of the SBR policy depends on the candidate policy's strength, as we calculate an improvement compared to $\pi^\textrm{c}_i$ in optimizing the 1-turn value estimate.
Note we can use a different policy $\pi^\textrm{c}$ to the policy $\pi^\textrm{b}$ we are responding to.

The number of strategies available to opponents is also too large, so calculating the 1-turn value of any candidate is intractable. We therefore use Monte-Carlo sampling. Values are often affected by the decisions of other players; to reduce variance we use common random numbers when sampling opponent actions: we evaluate all candidates with the same opponent actions (\textit{base profiles}). SBR can be seen as finding a BR to the sampled base profiles, which approximate the opponent policies.

\subsection{Best Response Policy Iteration}
\label{methods:brpi}

We present a family of PI approaches tailored to using (approximate) BRs, such as SBR, in a many-agent game; we refer to them collectively as Best Response Policy Iteration (BRPI) algorithms (Algorithm~\ref{alg:sp_pi_loop}). SBR depends on the $\pi^b, \pi^c, v$ (base policy, candidate policy and value function); we can use historical network checkpoints for these. Different choices give different BRPI algorithms. The simplest version is standard PI with BRs, while others BRPI variants approximate fictitious play.

In the most basic BRPI approach, every iteration $t$ we apply SBR to the {\it latest} policy $\hat{\pi}^{t-1}$ and value $\hat{V}^{t-1}$ to obtain an improved policy $\pi'$ (i.e. SBR($\pi^c=\hat{\pi}^{t-1}$, $\pi^b=\hat{\pi}^{t-1}$, $v=\hat{V}^{t-1})$). We then sample trajectories of self-play with $\pi'$ to create a dataset, to which we fit a new policy $\hat{\pi}^t$ and value $\hat{V}^t$ using the same techniques used to imitate human data (supervised learning with a GNN). We refer to this as Iterated Best Response (IBR). IBR is akin to applying standard single-agent PI methods in self-play, a popular approach for perfect information, 2-player zero-sum games~\cite{sutton2018reinforcement,scherrer2015approximate,Silver18AlphaZero, Anthony17ExIT}. 

However, iteration through exact best responses may behave poorly, failing to converge and leading to cycling among strategies. Further, in a game with simultaneous moves, deterministic play is undesirable, and best responses are typically deterministic. As a potential remedy, we consider PI algorithms based on Fictitious Play (FP)~\cite{brown1951iterative, van2000weakened, leslie2006generalised, heinrich2015fictitious}. In FP, at each iteration all players best respond to the empirical distribution over historical opponent strategies. In 2-player zero-sum, the time average of players' strategies converges to a Nash Equilibrium~\cite{brown1951iterative,robinson1951iterative}. In Appendix~\ref{app:theory}, we review theory on many-agent FP, and prove that continuous-time FP converges to a coarse correlated equilibrium in many-agent games. This motivates approximating FP with a BRPI algorithm. We now provide two versions of Fictitious Play Policy Iteration (FPPI) that do this.

The first method, FPPI-1, is akin to NFSP~\cite{heinrich2016deep}. At iteration $t$, we aim to train our policy and value networks $\hat{\pi}^{t}$, $\hat{V}^{t}$ to approximate the time-average of BRs (rather than the latest BR). 
With such a network, to calculate the BR at time $t$, we need an approximate best response to the latest policy network (which is the time-average policy), so use SBR($\pi^b = \hat{\pi}^{t-1}$, $v=\hat{V}^{t-1})$.
Hence, to train the network to produce the \textit{average} of BRs so far, at the start of each game we uniformly sample an iteration $d \in \{ 0, 1, \ldots, t-1\}$; if we sample $d=t-1$ we use the latest BR, and if $d<t-1$ we play a game with the historical checkpoints to produce the historical BR policy from iteration $d$.~\footnote{A similar effect could be achieved with a DAgger-like procedure~\cite{ross2011reduction}, or reservoir sampling \cite{heinrich2016deep}.}

FPPI-1 has some drawbacks. With multiple opponents, the empirical distribution of opponent strategies does not factorize into the empirical distributions for each player. But a standard policy network only predicts the per-player marginals, rather than the full joint distribution, 
which weakens the connection to FP. Also, our best response operator's strength is affected by the strength of the candidate policy and the value function. But FPPI-1 continues to imitate old and possibly weaker best responses, even after we have trained stronger policies and value functions. 

In our second variant, FPPI-2, we train the policy networks to predict only the latest BR, and explicitly average historical checkpoints to provide the empirical strategy so far.
The empirical opponent strategy up to time $t$ is $\mu^t := \frac{1}{t}\sum_{d<t} \pi_{-i}^{d}$, to draw from this distribution we first sample a historical checkpoint $d<t$, and then sample actions for all players using the same checkpoint. 
Player $i$'s strategy at time $t$ should be an approximate best response to this strategy, and the next policy network $\pi^{t}$ imitates that best response. In SBR, this means we use $\pi^b=\mu^t$ as the base policy.

This remedies the drawbacks of the first approach. The correlations in opponent strategies are preserved because we sample from the same checkpoint for all opponents. More importantly, we no longer reconstruct any historical BRs, so can use our best networks for the candidate policy and value function in SBR, independently of which checkpoints are sampled to produce base profiles. For example, using the latest networks for the candidate policy and value function, while uniformly sampling checkpoints for base profiles, could find stronger best responses while still approximating FP. However, FPPI-2's final time-averaged policy is represented by a mixture over multiple checkpoints.

These variants suggest a design space of algorithms combining SBR with PI. (1) The base policy can either be the latest policy (an IBR method), or from a uniformly sampled previous checkpoint (an FP method). (2) We can also use either the latest or a uniformly sampled previous value function. (3) The candidate policy both acts as a regulariser on the next policy and drives exploration, so we consider several options: using the initial (human imitation) policy, using the latest policy, or using a uniformly sampled checkpoint; we also consider mixed strategies: taking half the candidates from initial and half from latest, or taking half from initial and half from a uniformly sampled checkpoint. 

Appendix \ref{app:blotto} is a case study analysing how SBR and our BRPI methods perform in a many-agent version of the Colonel Blotto game~\cite{Borel21}. Blotto is small enough that exact BRs can be calculated, so we can investigate how exact FP and IBR perform in these games, how using SBR with various parameters affects tabular FP, and how different candidate policy and base policy choices affect a model of BRPI with function approximation. We find that: (1) exact IBR is ineffective in Blotto; (2) stochastic best responses in general, and SBR in particular, improve convergence rates for FP; (3) using SBR dramatically improves the behaviour of IBR methods compared to exact BRs.

\begin{figure}
\begin{minipage}[t]{0.49\textwidth}
\begin{algorithm}[H]
    \centering
    \caption{Sampled Best Response}
    \begin{algorithmic}[1]
    \Require{Policies $\pi^b, \pi^c$, value function $v$}  
    \Function{SBR}{$s$:state, $i$:player}  
        \For{$j \gets 1$ to $B$}
            \State $b_j \sim \pi^b_{-i}(s)$ \Comment{Sample Base Profile}
        \EndFor
        \For {$j \gets 1$ to $C$}
            \State $c_j \sim \pi^c_i(s)$ \Comment{Candidate Action}
            \State $\hat{Q}(c_j) \gets \frac{1}{B} \sum_{k=1}^B v(T(s,(c_j,b_k)))$
        \EndFor
        \State \Return $\arg \max_{c \in \{c_j \}_{j=1}^C} \hat{Q}(c)$
   \EndFunction  
   \end{algorithmic}
\label{alg:sbr}
\end{algorithm}
\end{minipage}
\hfill
\begin{minipage}[t]{0.49\textwidth}
\begin{algorithm}[H]
    \centering
    \caption{Best Response Policy Iteration}
    \begin{algorithmic}[1]
    \Require{Best Response Operator BR}  
    \Function{BRPI}{$\pi_0({\theta}), v_0({\theta})$}  
    \For{$t \gets 1$ to $N$}            
        \State $\pi^\textrm{imp} \gets$ BR($\{ \pi_j \}_{j=0}^{t-1}, \{ v_j \}_{j=0}^{t-1}$)
        \State $D \gets$ Sample-Trajectories($\pi^\textrm{imp}$)
        \State $\pi_i(\theta) \gets$ Learn-Policy($D$)
        \State $v_i(\theta) \gets$ Learn-Value($D$)
    \EndFor
    \State \Return $\pi_N, v_N$

    \EndFunction  

    \end{algorithmic}
\label{alg:sp_pi_loop}
\end{algorithm}
\end{minipage}
\end{figure}

\subsection{Neural Architecture}
\label{l_sect_neural_arch}
Our network is based on the imitation learning of DipNet~\cite{paquette2019no}, which uses an encoder GNN to embed each province, and a LSTM decoder to output unit moves (see DipNet paper for details). We make several improvements, described briefly here, and fully in Appendix~\ref{appendix:network}. (1) We use the board features of DipNet, but replace the original `alliance features' with the board state at the last moves phase, combined with learned embeddings of all actions taken since that phase. (2) In the encoder, we removed the FiLM layer, and added node features to the existing edge features of the GNN. (3) Our decoder uses a GNN relational order decoder rather than an LSTM. These changes increase prediction accuracy by $4-5\%$ on our validation set (data splits and performance comparison in Appendix~\ref{appendix:network}).

\section{Evaluation Methods}
\label{l_sect_eval_methods}

We analyze our agents through multiple lenses: We measure winrates (1) head-to-head between agents from different algorithms and (2) against fixed populations of reference agents. (3) We consider `meta-games' between checkpoints of one training run to test for consistent improvement. (4) We examine the exploitability of agents from different algorithms. Results of each analysis are in the corresponding part of Section~ \ref{l_sect_experiments}. 

{\bf Head-to-head comparison:}
We play 1v6 games between final agents of different BRPI variants and other baselines to directly compare their performance. This comparison also allows us to spot if interactions between pairs of agents give unusual results.
From an evolutionary game theory perspective, 1v6 winrates indicate whether a population of agents can be `invaded' by a different agent, and hence whether they constitute Evolutionary Stable Strategies (ESS)~\cite{taylor1978evolutionary,smith1982evolution}. 
ESS have been important in the study of cooperation, as a conditionally cooperative strategies such as Tit-for-Tat can be less prone to invasion than purely co-operative or mostly non-cooperative strategies~\cite{axelrod1981evolution}.

{\bf Winrate Against a Population:}
We assess how well an agent performs against a reference population. An agent to be evaluated plays against 6 players independently drawn from the reference population, with the country it plays as chosen at random each game. We report the average score of the agent, and refer to this as a ``1v6'' match.~\footnote{The score is 1 for a win, $\frac{1}{n}$ for $n$ players surviving at a timeout of $\sim80$ game-years, and 0 otherwise.}
This mirrors how  people play the game: each player only ever 
represents a single agent, and wants to maximize their score against a population of other people.
We consider two reference populations: (a) only the DipNet agent~\cite{paquette2019no}, the previous state-of-the-art method, 
which imitates and hence is a proxy for human play. (b) a uniform mixture of 15 final RL agents, each from a different BRPI method (see Appendix~\ref{app:results}); BRPI agents are substantially stronger than DipNet, and the mixture promotes opponent diversity.

{\bf Policy Transitivity:}
Policy intransitivity relates to an improvement dynamics that cycles through policy space, rather than yielding a consistent improvement in the quality of the agents~\cite{hofbauer2003evolutionary,balduzzi2018re}, 
which can occur because multiple agents all optimize different objectives.
We assess policy transitivity with \textit{meta-games} between the checkpoints of a training run.
In the meta-game, instead of playing yourself, you elect an `AI champion' to play on your behalf, and achieve the score of your chosen champion.
Each of the seven players may select a champion from among the same set of $N$ pre-trained policies. 
We randomize the country each player plays, so the meta-game is a symmetric, zero-sum, 7-player  game.
If training is transitive, choosing later policies will perform better in the meta-game.

Game theory recommends selecting a champion by sampling one of the $N$ champions according to a Nash equilibrium~\cite{nash1950equilibrium}, with bounded rationality modelled by a Quantal Response Equilibrium (QRE)~\cite{mckelvey1995quantal}. Champions can be ranked according to their probability mass in the equilibrium~\cite{balduzzi2018re}.~\footnote{A similar analysis called a `Nash League' was used to study Starcraft agents~\cite{vinyals2019alphastar}.} We calculate a QRE (see Appendix \ref{app:nashconv_descent}) of the meta-game consisting of $i$ early checkpoints, and see how it changes as later checkpoints are added. In transitive runs we expect the distribution of the equilibrium to be biased towards later checkpoints.

Finding a Nash equilibrium of the meta-game is computationally hard (PPAD-complete)~\cite{papadimitriou1994complexity}, so as an alternative, we consider a simplified 2-player meta-game, where the row player's agent plays for one country, and the other player's agent plays in the other 6, we call this the `\textit{1v6 meta-game}'. We report heatmaps of the payoff table, where the row and column strategies are sorted chronologically. If training is transitive, the row payoff increases as row index increases but decreases as the column index increases, which is visually distinctive~\cite{pmlr-v97-balduzzi19a}.

{\bf Exploitability:} The exploitability of a policy $\pi$ is the margin by which an adversary (i.e. BR) to $\pi$ would defeat a population of agents playing $\pi$; it has been a key metric for Poker agents~\cite{lisy2017eqilibrium}. As SBR approximates a BR to its base policy, it can be used to lower bound the base policy's exploitability, measured by the average score of 1 SBR agent against 6 copies of $\pi$.
The strongest exploit we found mixes the human imitation policy and $\pi$ for candidates, and uses $\pi$'s value function, i.e. SBR($\pi^c=\pi^\textrm{SL}+\pi$, $\pi^b=\pi$, $v=V^\pi$).
People can exploit previous Diplomacy AIs after only a few games, so few-shot exploitability may measure progress towards human-level No-Press Diplomacy agents.
If we use a `neutral' policy and value function, and only a few base profiles, SBR acts as a few-shot measure of exploitability.
To do this we use the human imitation policy $\pi^\textrm{SL}$ for candidates, and - because $V^\textrm{SL}$ is too weak - a value function from an independent BRPI run $V^\textrm{RL}$.

\section{Results}
\label{l_sect_experiments}

We analyse three BRPI algorithms: IBR, FPPI-1 and FPPI-2, defined in Section~\ref{methods:brpi}. In FPPI-2 we use the latest value function. We sample candidates with a mixture taking half the candidates from the initial policy. The other half for IBR and FPPI-2 is from iteration $t-1$; for FPPI-1 it comes from the uniformly sampled iteration. At test time, we run all networks at a softmax temperature $t=0.1$. 
\footnote{We have published a re-implementation of our agents for benchmarking at \url{https://github.com/deepmind/diplomacy}}

{\bf Head-to-head comparison:} we compare our methods to the supervised learning (SL) and RL (A2C) DipNet agents~\cite{paquette2019no} and our SL implementation, with our neural architecture (see Appendix~\ref{app:results} for additional comparisons). 
Table~\ref{tab_1v6_agent_comparisons} shows the average 1v6 score where a single row agent plays against 6 column agents. For the BRPI methods, these are averaged over $5$ seeds of each run. All of our learning methods give a large improvement over both supervised learning baselines, and an even larger winrate over DipNet A2C. Among our learning methods, FPPI-2 achieves the best winrate in 1v6 against each algorithm, and is also the strategy against which all singleton agents do the worst.~\footnote{Note that the diagonal entries of this table involve agents playing against different agents trained using the same algorithm. This means that they are not necessarily equal to $1/7$.}

\begin{table}[ht]
\centering
 \begin{tabular}{l || r  r  r | r r r  } 
\hline
{} & SL \cite{paquette2019no} & A2C \cite{paquette2019no} & SL (ours) & FPPI-1 & IBR & FPPI-2 \\
\hline
\hline
SL \cite{paquette2019no} & 14.2\% & 8.3\% & 16.3\% & 2.3\% & 1.8\% & \textit{0.8\%} \\
A2C \cite{paquette2019no} & 15.1\% & 14.2\% & 15.3\% & 2.3\% & 1.7\% & \textit{0.9\%} \\
SL (ours) & 12.6\% & 7.7\% & 14.1\% & 3.0\% & 1.9\% &\textit{1.1\%} \\
\hline
FPPI-1 & \textbf{26.4\%} & 28.0\% & \textbf{25.9\%} & 14.4\% & 7.4\% & \textit{4.5\%} \\
IBR & 20.7\% & 30.5\% & \textbf{25.8\%} & 20.3\% & 12.9\% & \textit{10.9\%} \\
FPPI-2 & 19.4\% & \textbf{32.5\%} & 20.8\% & \textbf{22.4\%} & \textbf{13.8\%} & \textit{\textbf{12.7\%}} \\

\hline
\end{tabular}
\caption{Average scores for 1 row player vs 6 column players. BRPI methods give an improvement over A2C or supervised learning. All numbers accurate to a $95\%$ confidence interval of $\pm 0.5\%$. Bold numbers are the best value for single agents against a given set of 6 agents, italics are for the best result for a set of 6-agents against each single agent.}
\label{tab_1v6_agent_comparisons}
\end{table}

{\bf Winrate Against a Population:} The left of Figure~\ref{l_fig_perf_sppi_run} shows the performance of our BRPI methods against DipNet through training. Solid lines are the winrate of our agents in 1v6 games (1 BRPI agent vs. 6 DipNet agents), and dashed lines relate to the winrate of DipNet reverse games (1 DipNet agent vs 6 PI agents). The x-axis shows the number of policy iterations. A dashed black line indicates a winrate of 1/7th (expected for identical agents as there are 7 players). The figure shows that all PI methods quickly beat the DipNet baseline before plateauing. Meanwhile, the DipNet winrate drops close to zero in the reverse games. 
The figure on the right is identical, except the baseline is a uniform mixture of our final agents from BRPI methods. 
Against this population, our algorithms do not plateau, instead improving steadily through training.
The figure shows that FPPI-1 tends to under-perform against our other BRPI methods (FPPI-2 and IBR). We averaged 5 different runs with different random seeds, and display $90\%$ confidence intervals with shaded area. 

\begin{figure}[h!tbp]
  \centering
  \begin{minipage}[b]{0.45\textwidth}
    \includegraphics[width=\textwidth]{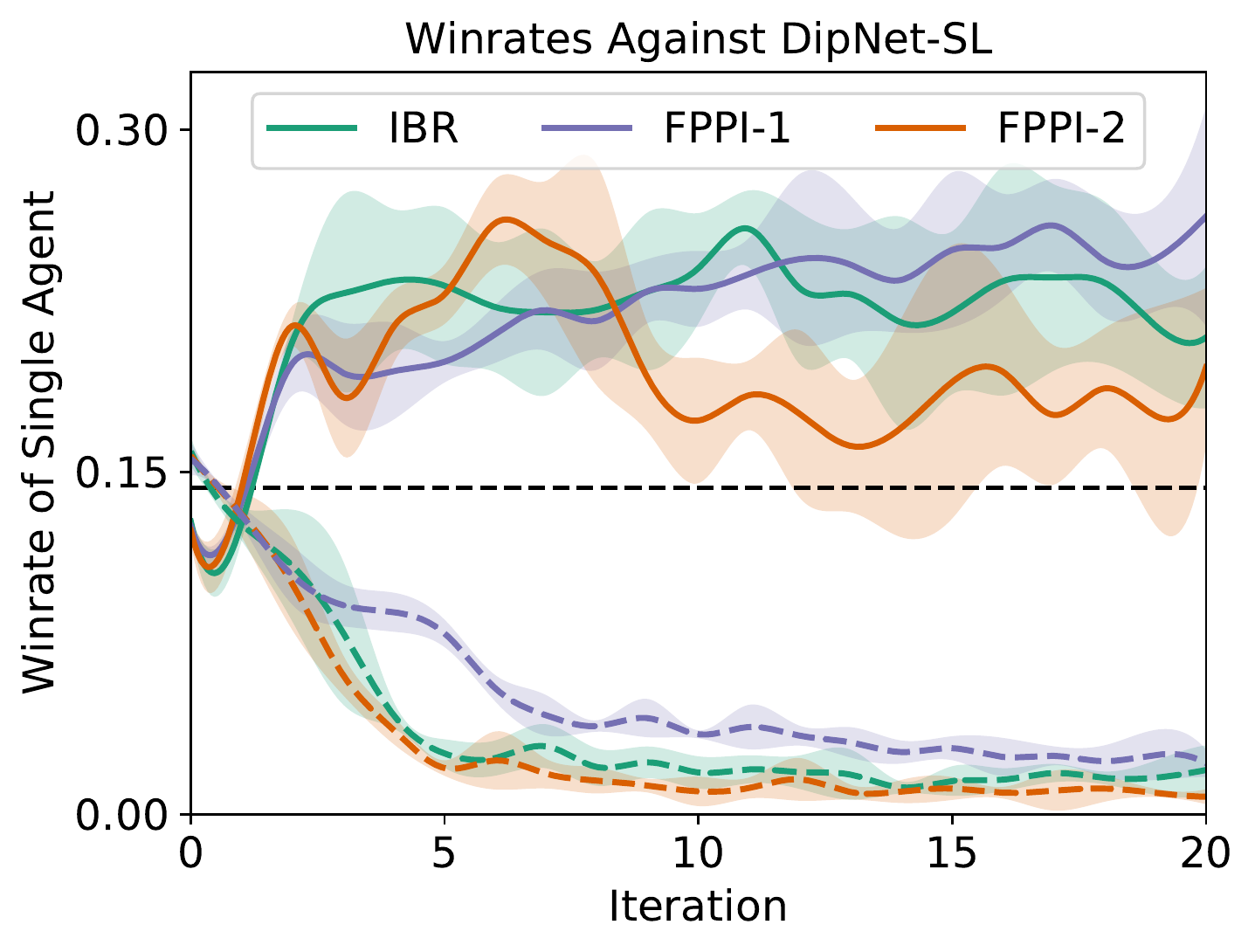}
  \end{minipage}
  \hfill
  \begin{minipage}[b]{0.45\textwidth}
    \includegraphics[width=\textwidth]{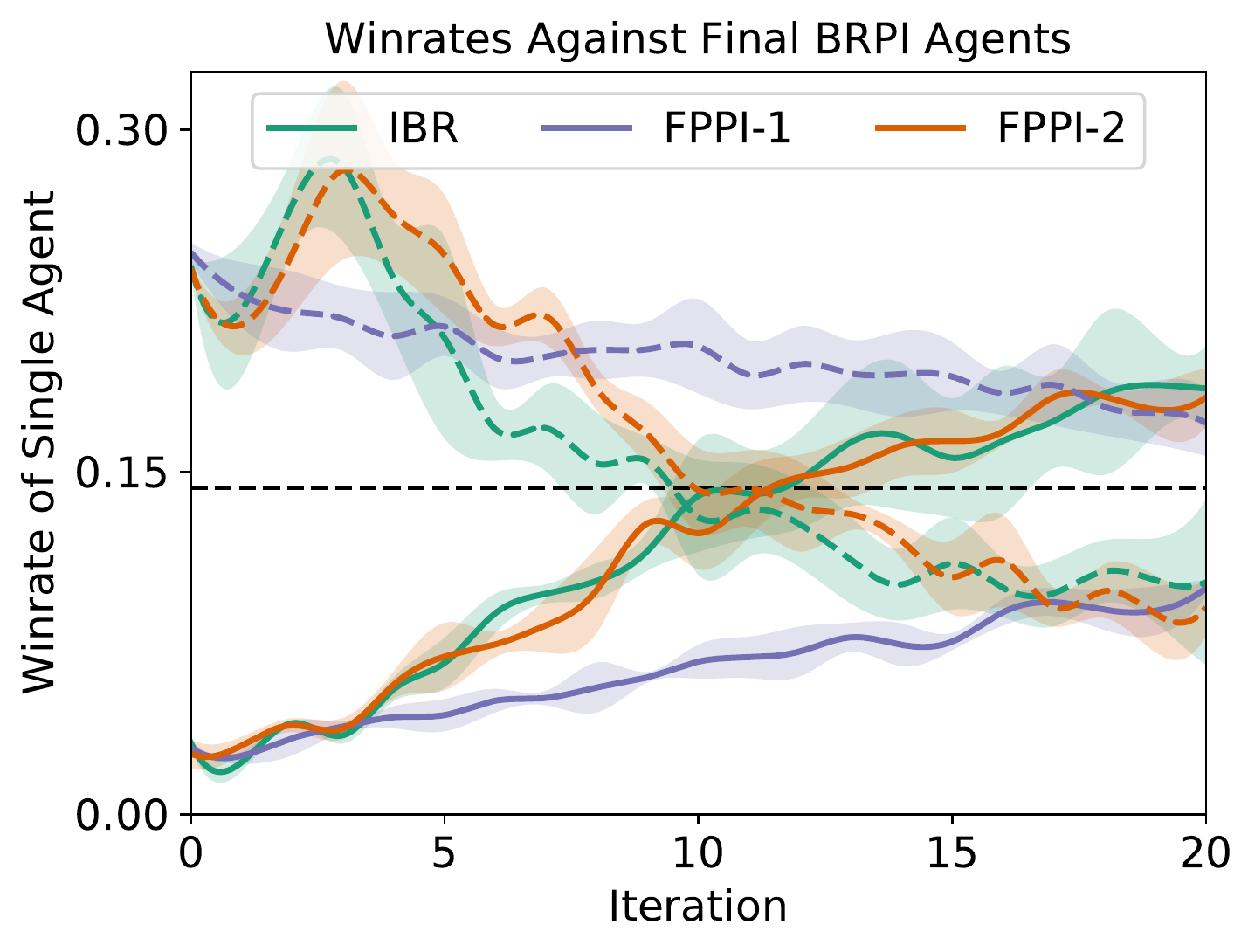}
  \end{minipage}
\vspace{-3mm}\caption{Winrates through training, 1v6 or 6v1 against different reference populations}
\label{l_fig_perf_sppi_run}
\end{figure}

\begin{figure}[h!tbp]
  \centering
  \begin{minipage}[b]{0.25\textwidth}
    \includegraphics[width=\textwidth]{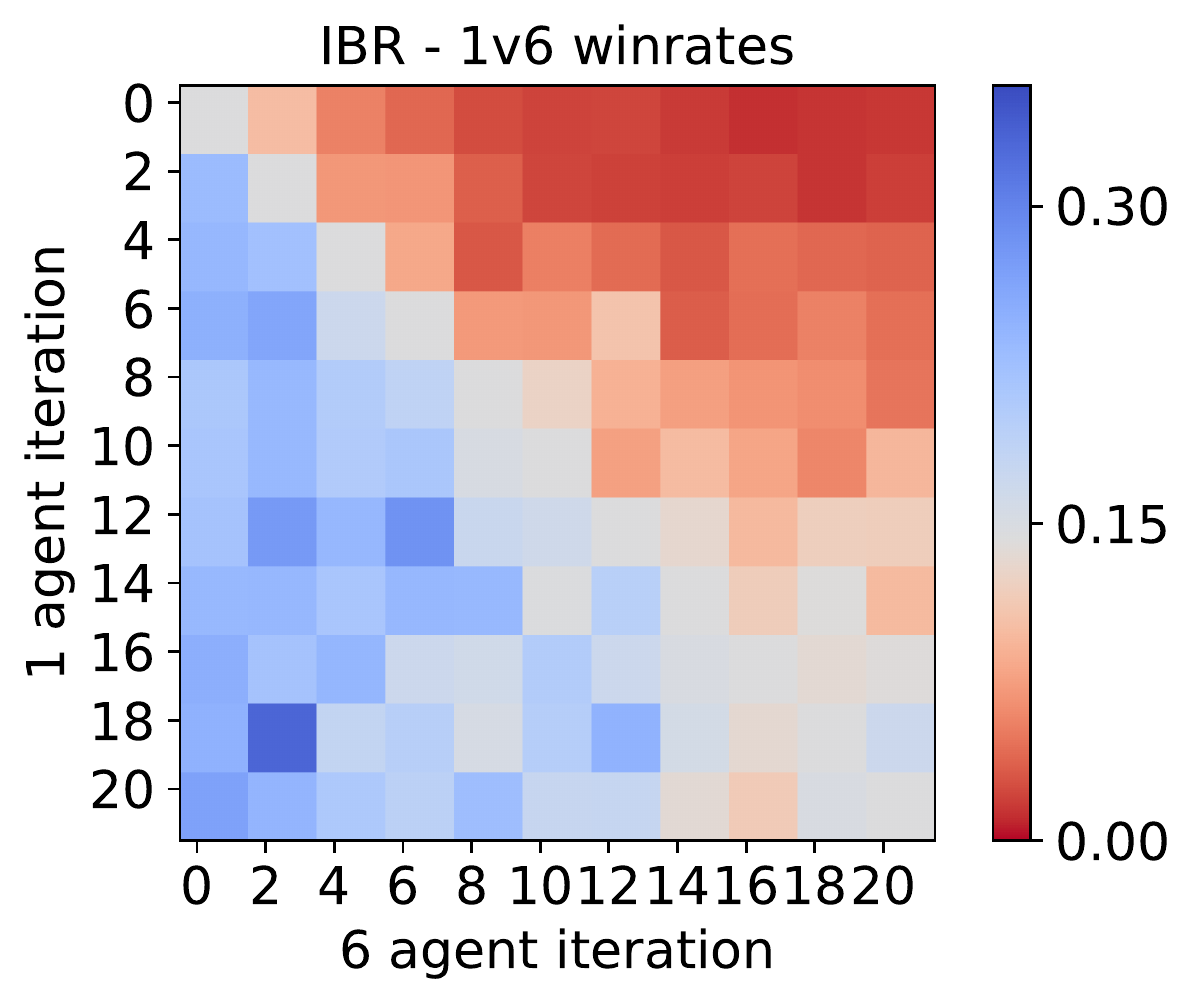}
    \includegraphics[width=\textwidth]{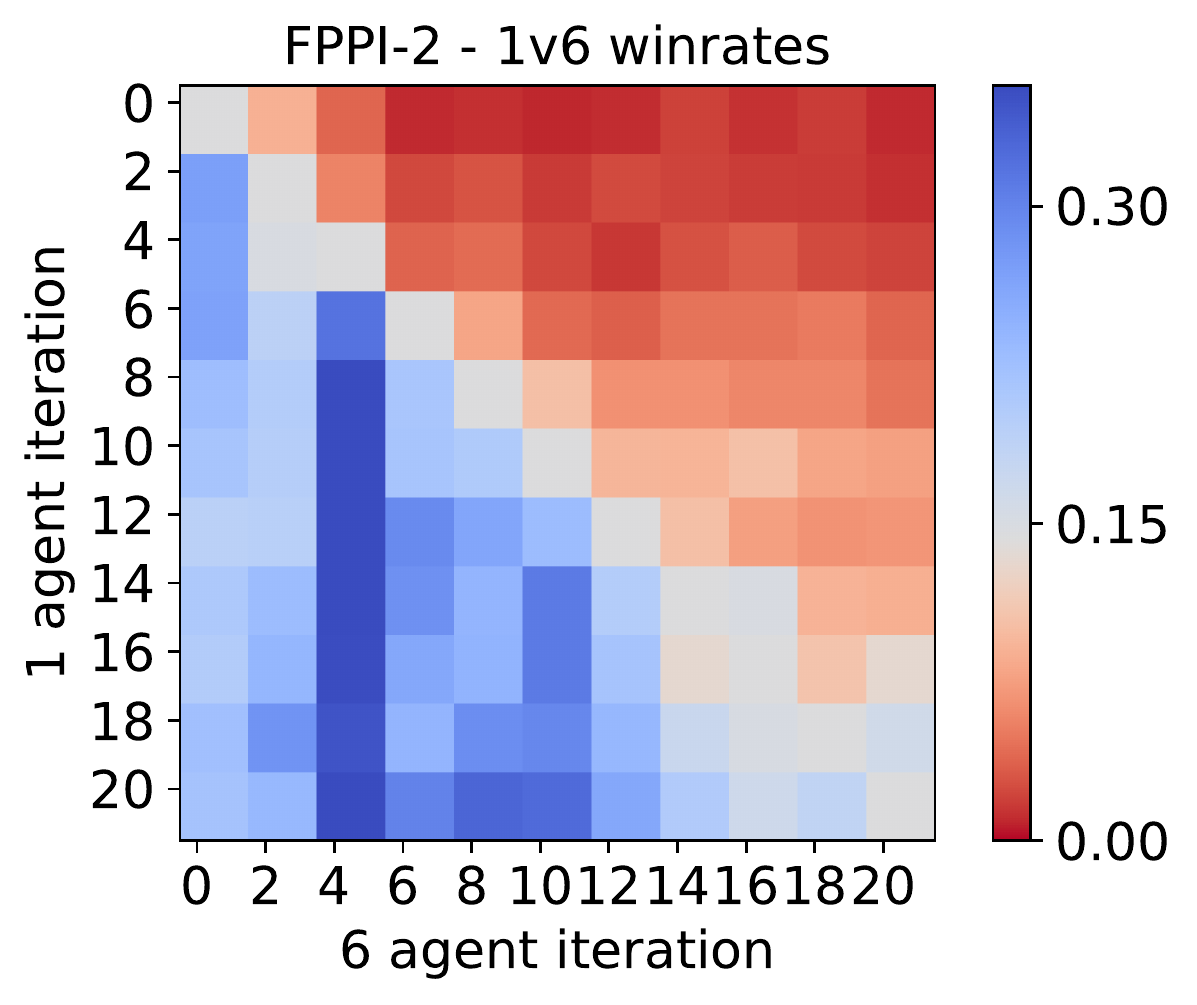}
  \end{minipage}
  \hfill
  \begin{minipage}[b]{0.74\textwidth}
    \includegraphics[width=0.95\textwidth]{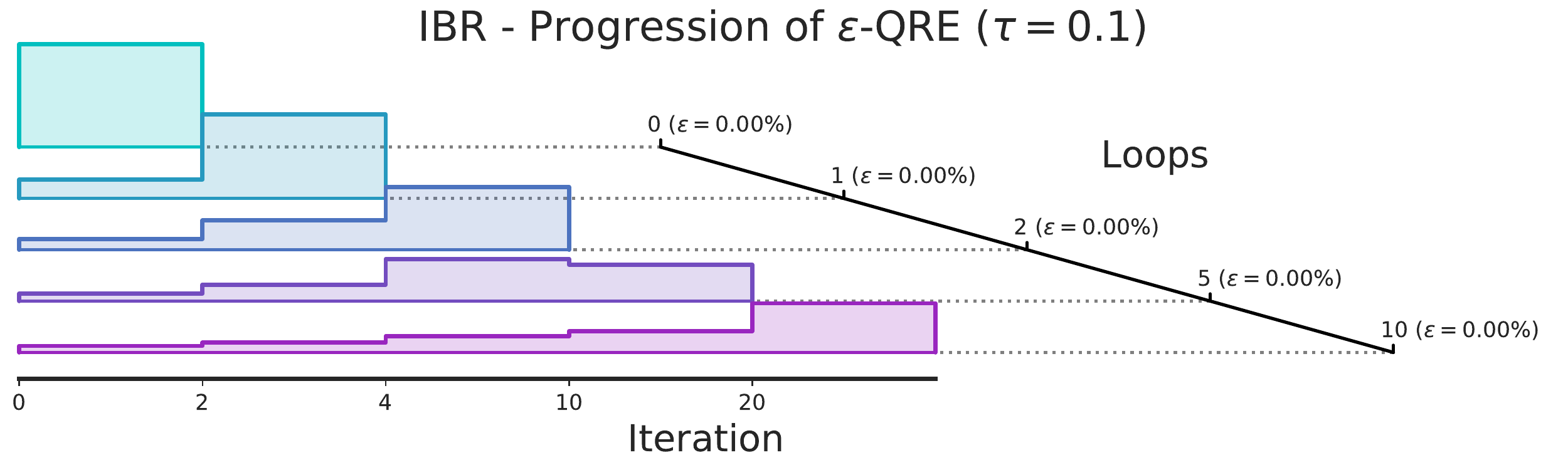}
    \includegraphics[width=0.95\textwidth]{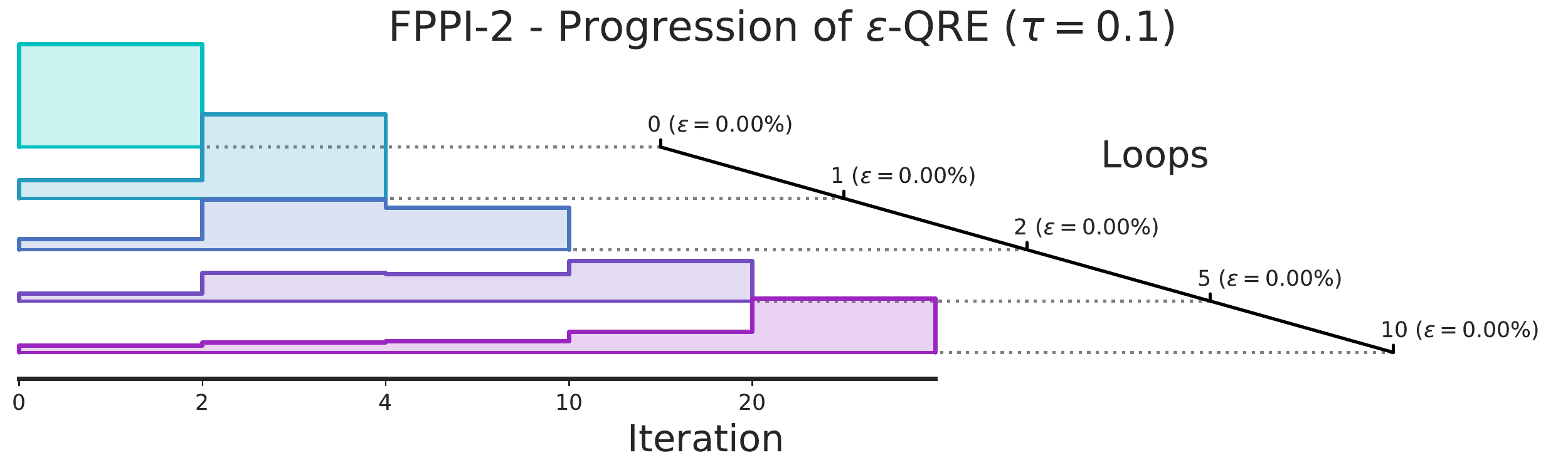}
  \end{minipage}
\caption{Transitivity Meta-Games: Top: IBR, Bottom: FPPI-2. Left: 1v6, Right: QRE in 7-player.}
\label{l_fig_dip_nash_league}
\end{figure}

{\bf Policy Transitivity:} Figure~\ref{l_fig_dip_nash_league} depicts the Diplomacy meta-game between checkpoints produced in a single BRPI run. The heatmaps on the left examine a 1v6 meta game, showing the winrate of one row checkpoint playing against 6 column checkpoints for even numbered checkpoints through the run. The plots show that nearly every checkpoint beats all the previous checkpoints in terms of 1v6 winrate. A checkpoint may beat its predecessors in way that's exploitable by other strategies: checkpoint 4 in the FPPI-2 run beats the previous checkpoints, but is beaten by all subsequent checkpoints by a larger margin than previous checkpoints.

The right side of Figure~\ref{l_fig_dip_nash_league} shows the Nash-league of the full meta-game. The $i^\textrm{th}$ row shows the distribution of a QRE in the meta-game over the first $i$ checkpoints analyzed (every row adds the next checkpoint). We consider checkpoints spaced exponentially through training, as gains to further training diminish with time. The figure shows that the QRE consistently places most of the mass on the recent checkpoint, indicating a consistent improvement in the quality of strategies, rather than cycling between different specialized strategies. This is particularly notable for IBR: every checkpoint is only trained to beat the previous one, yet we still observe transitive improvements during the run (consistent with our positive findings for IBR in Blotto in Appendix~\ref{app:blotto-ibr}).

{\bf Exploitability:} 
We find that all our agents are fairly exploitable at the end of training (e.g. the strongest exploiters achieve a 48\% winrate), and the strongest exploit found does not change much through training, but few-shot exploitability \textit{is} reduced through training.
Agents are more exploitable at low temperatures, suggesting our agents usefully mix strategies. 
Final training targets are less exploitable than final networks, 
including in IBR, unlike what we'd expect if we used an exact BR operator, which would yield a highly exploitable deterministic policy. 
For full details see Appendix~\ref{app:results}.

\section{Conclusion}
\label{l_sect_conclusion}

We proposed a novel approach for training RL agents in Diplomacy with BRPI methods and overcoming the simultaneous moves and large combinatorial action space using our simple yet effective SBR improvement operator. 
We showed that the stochasticity of SBR was beneficial to BRPI algorithms in Blotto.
We set-out a thorough analysis process for Diplomacy agents. 
Our methods improve over the state of the art, yielding a consistent improvement of the agent policy. 
IBR was surprisingly effective, but our strongest method was FPPI-2, a new approach to combining RL with fictitious play.

Using RL to improve game-play in No-press Diplomacy is a prerequisite for investigating the complex mixed motives and many-player aspects of this game. Future work can now focus on questions like:
(1) What is needed to achieve human-level No-Press Diplomacy AI?
(2) How can the exploitability of RL agents be reduced? 
(3) Can we build agents that reason about the incentives of others, for example behaving in a reciprocal manner~\cite{eccles2019imitation}, or by applying opponent shaping \cite{foerster2018learning}?
(4) How can agents learn to use signalling actions to communicate intentions in No-Press Diplomacy?
(5) Finally, how can agents handle negotiation in Press variants of the game, where communication is allowed?

\nocite{balduzzi2020smooth}
\clearpage

\section*{Broader Impact}
\label{app:broader_impact}

We discuss the potential impact of our work, examining possible positive and negative societal impact.

{\bf What is special about Diplomacy?}
Diplomacy~\cite{calhamer1959diplomacy} has simple rules but high emergent complexity. 
It was designed to accentuate dilemmas relating to building alliances, negotiation and teamwork in the face of uncertainty about other agents. The tactical elements of Diplomacy form a difficult environment for AI algorithms: the game is played by seven players, it applies simultaneous moves, 
and has a very large combinatorial action space. 

{\bf What societal impact might it have?}
We distinguish immediate societal impact arising from the availability of the new training algorithm, and indirect societal impact due to the future work on many-agent strategic decision making enabled or inspired by this work.

\textbf{Immediate Impact.} 
Our methods allow training agents in Diplomacy and other temporally extended environments where players take simultaneous actions, and the action of a player can be decomposed into multiple sub-actions, in domains that can can be simulated well, but in which learning has been difficult so far. 
Beyond the direct impact on Diplomacy, possible applications of our method include business, economic, and logistics domains, in as far as the scenario can be simulated sufficiently accurately. Examples include games that require a participant to control multiple units (Starcraft and Dota~\cite{vinyals2019alphastar,openai2019dota} have this structure, but there are many more), controlling fleets of cars or robots~\cite{agmon2008multi,portugal2011survey,vandael2015reinforcement} or sequential resource allocation~\cite{reveliotis1997polynomial,park2001deadlock}. However, applications such as in business or logistics are hard to capture realistically with a simulator, so significant additional work is needed to apply this technology in real-world domains involving multi-agent learning and planning.

While Diplomacy is themed as a game of strategy where players control armies trying to gain control of provinces, it is a very abstract game - not unlike Chess or Checkers. It seems unlikely that real-world scenarios could be successfully reduced to the level of abstraction of a game like Diplomacy. 
In particular, our current algorithms assume a known rule set and perfect information between turns, whereas the real world would require planning algorithms that can manage uncertainty robustly. 

\textbf{Future Impact.} In providing the capability of training a tactical baseline agent for Diplomacy or similar games, this work also paves the way for research into agents that are capable of forming alliances and use more advanced communication abilities, either with other machines or with humans. In Diplomacy and related games this may lead to more interesting AI partners to play with. More generally, this line of work may inspire further work on problems of cooperation. 
We believe that a key skill for a Diplomacy player is to ensure that, wherever possible, their pairwise interactions with other players are positive sum. AIs able to play Diplomacy at human level must be able to achieve this in spite of the incentive to unilaterally exploit trust established with other agents.

More long term, this work may pave the way towards research into agents that play the full version of the game of Diplomacy, which includes communication. In this version, communication is used to broker deals and form alliances, but also to misrepresent situations and intentions. For example, agents may learn to establish trust, but might also exploit that trust to mislead their co-players and gain the upper hand. In this sense, this work may facilitate the development of manipulative agents that use false communication as a means to achieve their goals. To mitigate this risk, we propose using games like Diplomacy to study the emergence and detection of manipulative behaviours in a sandbox --- to make sure that we know how to mitigate such behaviours in real-world applications.

Overall, our work provides an algorithmic building block for finding good strategies in many-agent systems. 
While prior work has shown that the default behaviour of independent reinforcement learning agents can be non-cooperative \cite{leibo2017multi,foerster2018learning,hughes2020learning},
we believe research on Diplomacy could pave the way towards creating artificial agents that can successfully cooperate with others, including handling difficult questions that arise around establishing and maintaining trust and alliances.

\section*{Acknowledgements}
We thank Julia Cohen, Oliver Smith, Dario de Cesare, Victoria Langston, Tyler Liechty, Amy Merrick and Elspeth White for supporting the project. We thank Edward Hughes and David Balduzzi for their advice on the project. We thank Kestas Kuliukas for providing the dataset of human diplomacy games.

\section*{Author Contributions}
T.A., T.E., A.T., J.K., I.G. and Y.B. designed an implemented the RL algorithm. I.G., T.A., T.E., A.T., J.K., R.E., R.W. and Y.B. designed and implemented the evaluation methods. A.T., J.K., T.A., T.E., I.G. and Y.B. designed and implemented the improvements to the network architecture. T.H. and N.P. wrote the Diplomacy adjudicator. M.L. and J.P. performed the case-study on Blotto and theoretical work on FP. T.A., M.L. and Y.B. wrote the paper. S.S. and T.G. advised the project

\bibliographystyle{plain}
\bibliography{references}

\newpage
\appendix
\renewcommand \thepart{}
\renewcommand \partname{}
\part{Appendices} 

\setlength{\cftbeforepartskip}{5em}
\tableofcontents

\addtocontents{toc}{\protect\setcounter{tocdepth}{2}}

\section{Case Study: Many-Player Colonel Blotto}
\label{app:blotto}

Our improvement operator for Diplomacy is SBR, described in Section~\ref{methods:sbr}. Diplomacy is a complex environment, where training requires significant time. To analyze the impact of applying SBR as an improvement operator, we use the much simpler normal-form game called the Colonel Blotto game~\cite{Borel21} as an evaluation  environment. We use the Blotto environment to examine several variants of best response policy iteration. We run these experiments using OpenSpiel~\cite{LanctotEtAl2019OpenSpiel} and will contribute the code to reproduce these results.

Colonel Blotto is a famous game theoretic setting proposed about a century ago. It has a very large action space, and although its rules are short and simple, it has a high emergent complexity~\cite{roberson2006colonel}. The original game has only two players, but we describe an $n$-player variant, Blotto$(n,c,f)$: each player has $c$ coins to be distributed across $f$ fields. The aim is to win the most fields. Each player takes a single action simultaneously and the game ends. The action is an allocation of the player's coins across the fields: the player decides how many of its $c$ coins to put in each of the fields, choosing $c_1,c_2,\ldots,c_f$ where $\sum_{i=1}^f c_i = c$. For example, with $c=10$ coins and $f=3$ fields, one valid action is [7, 2, 1], interpreted as 7 coins on the first field, 2 on the second, and 1 on the third. A field is won by the player contributing the most coins to that field (and drawn if there is a tie for the most coins). The winner receives a reward or +1 (or a +1 is shared among all players with the most fields in the case of a tie), and the losers share a -1, and all player receive 0 in the case of a $n$-way tie.

We based our analysis on Blotto for several reasons. First, like Diplomacy, it has the property that the action space grows combinatorially as the game parameter values increase. Second, Blotto is a game where the simultaneous-move nature of the game matters, so players have to employ mixed strategies, which is a difficult case for standard RL and best response algorithms. Third, it has been used in a number of empirical studies to analyze the behavior of human play~\cite{arad2012multi,Kohli12Blotto}. Finally, Blotto is small enough that distances to various equilibria can be computed exactly, showing the effect each setting has on the empirical convergence rates.

Blotto differs to Diplomacy in several ways. The Nash equilibrium in Blotto tends to cover a substantial proportion of the strategy space, whereas in Diplomacy, many of the available actions are weak, so finding good actions is more of a needle-in-a-haystack problem. In many-agent Blotto, it is difficult for an agent to target any particular opponent with an action, whereas in Diplomacy most attacks are targeted against a specific opponent. Finally, Blotto is a single-turn (i.e. normal form) game, so value functions are not needed. Throughout this section, we use the exact game payoff wherever the value function would be needed in Diplomacy.

\subsection{Definitions and Notation Regarding Equilibria}
\label{l_sect_appendix_defs_notation}

We now provide definitions and notation regarding various forms of equilibria in games, which we use in other appendices also. 

Consider an $N$-player game where players take actions in a set $\{A^i\}_{i\in \{1, \dots, N\}}$. The reward for player $i$ under a profile $\mathbf{a} = \{a_1, \dots, a_N\}$ of actions is $r^i(a_1, \dots, a_N)$.

Each player uses a policy $\{\pi_i\}_{i\in \{1, \dots, N\}}$ which is a distribution over action $A^i$.

We use the following notation:
$$r^i_{\pi} = \mathbb{E}_{\forall i, a^i \sim \pi^i(.)}\left[r^i(a_1, \dots, a_N)\right]$$
$$r^i_{\pi^{-i}} = \mathbb{E}_{\forall j\neq i, a^j \sim \pi^j(.)}\left[r^i(a_1, \dots, a_N)\right]$$

{\bf Nash equilibrium:}
A Nash equilibrium is strategy profile $(x_1, \ldots, x_N)$ such that no player $i$ can improve their win rate by unilaterally deviating from their sampling distribution $x_i$ (with all other player distributions, $x_{-i}$, held fixed). Formally, a Nash equilibrium is a policy $\pi = \{\pi_1, \dots, \pi_N\}$ such that:
\begin{align*}
    \forall i, \max_{\pi'^{i}} \langle \pi'^{i} , r^i_{\pi^{-i}} \rangle = r^i_{\pi}
\end{align*}

{\bf $\epsilon$-Nash:} An $\epsilon$-Nash is a strategy profile such that the most a player can improve their win rate by unilaterally deviating is $\epsilon$:
\begin{align}
    \mathcal{L}_{\textrm{exp}_i}(\boldsymbol{x}) = \max_z \{r^i(z, x_{-i})\} - r^i(x_i, x_{-i}) &\le \epsilon
\end{align}
where $r^i$ is player $i$'s reward given all player strategies. 

The NashConv metric, defined in~\cite{lanctot2017unified}, has a value of 0 at a Nash equilibrium and can be interpreted as a distance from
Nash. It is the sum of over players of how much each player can improve their winrate with a unilateral deviation.

{\bf Coarse Correlated Equilibrium:}
The notion of a Coarse Correlated Equilibrium~\cite{moulin1978strategically} relates to a joint strategy $\pi(a_1, \dots, a_N)$, which might not factorize into independent per-player strategies. A joint strategy $\pi(a_1, \dots, a_N)$ is a Coarse Correlated Equilibrium if for all player $i$:
\begin{align}
    \label{eq:cce}
    \max_{a'_i} E_{a_1, \dots, a_N \sim \pi}\left[r(a'_i, a_{-i})\right] - E_{a_1, \dots, a_N \sim \pi}\left[r(a_1,\dots, a_N)\right] \leq 0
\end{align}

A joint strategy $\pi(a_1, \dots, a_N)$ is a $\epsilon$-Coarse Correlated Equilibrium if for all player $i$:
\begin{align}
    \label{eq:eps-cce}
    \max_{a'_i} E_{a_1, \dots, a_N \sim \pi}\left[r(a'_i, a_{-i})\right] - E_{a_1, \dots, a_N \sim \pi}\left[r(a_1,\dots, a_N)\right] \leq \epsilon
\end{align}

We define a similar metric to NashCov for the empirical distance to a
{\it coarse-correlated equilibrium} (CCE)~\cite{moulin1978strategically}, given its relationship to no-regret learning algorithms.
Note, importantly, when talking about the multiplayer ($n > 2$) variants, $\pi_{-i}^t$ is always the average joint policy
over all of player $i$'s opponents,
rather than the  player-wise marginal average policies composed into a joint policy (these correspond to the same notion of
average policy for the special case of $n = 2$, but not generally).

\begin{definition}[CCEDist]
For an $n$-player normal form game with players $i \in \cN = \{1 \cdots n \}$, reward function $r$,
joint action set $\cA = \cA_1 \times \cdots \times \cA_n$, let $a_{i \rightarrow *}$ be the joint action
where player $i$ replaces their action with $a_i^*$; 
given a correlation device (distribution over joint actions) $\mu$, 
\[
\textsc{CCEDist}(\mu) = \sum_{i \in \cN} \max(0, \max_{a_i^* \in \cA_i} \bE_{a \sim \mu}[r^i(a_{i \rightarrow *}, a_{-i}) - r^i(a)]).
\]
\end{definition}

\subsubsection{Relating $\epsilon$-Coarse Correlated Equilibria, Regret, and CCEDist}
\label{eps_cce_from_regret}

In this subsection, we formally clarify the relationships and terminology between similar but slightly different
concepts\footnote{For further reading on this topic, see~\cite{celli2019learning,farina2019coarse,gibson2013regret,gordon2008no}.}.

Several iterative algorithms playing an $n$-player game produce a sequence of policies $\pi^t = (\pi_1^t, \cdots, \pi_n^t)$ over iterations
$t \in \{ 1, 2, \cdots, T \}$.
Each individual player $i$ is said to have (external) {\bf regret},
\[
R^T_i = \max_{a_i'} \sum_{t=1}^T r^i(a_i', \pi^t_{-i}) - \sum_{t=1}^T r^i(\pi^t),
\]
where $r^i$ is player $i$'s {\it expected} reward (over potentially stochastic policies $\pi^t$). 
Define the {\bf empirical average joint policy} $\mu$ to be the a uniform distribution over $\{ \pi^t~|~1 \le t \le T\}$.
Denote $\delta_i(\mu, a_i')$ to be the incentive for player $i$ to deviate from playing $\mu$ to instead always
choosing $a_i'$ over all their iterations, i.e.
how much the player would have gained (or lost) by applying this deviation.
Now, define player $i$'s incentive to deviate to their best response action, $\epsilon_i$, and notice:
\[
\epsilon_i = \max_{a_i'} \delta_i(\mu, a_i') = \bE_{\pi^t \sim \mu}[R_i^t] = \frac{R_i^T}{T},
\]
where, in an $\epsilon$-CCE, $\epsilon = \max_i \epsilon_i$ (Equation~\ref{eq:eps-cce}). Also, a
CCE is achieved when $\epsilon \le 0$ (Equation~\ref{eq:cce}).
Finally, CCEDist is a different, but very related, metric from $\epsilon$. Instead, it sums over the players and
discards negative values from the sum, so $\textsc{CCEDist}(\mu) = \sum_i \max(0, \epsilon_i)$.

\subsection{A Generalization of Best Response Policy Iteration}

Algorithm~\ref{alg:gbrpi} presents the general family of response-style algorithms that we consider (covering our approaches in Section~\ref{methods:brpi}). 

\begin{algorithm}[!ht]
\caption{A Generalization of BRPI}
\label{alg:gbrpi}
\begin{algorithmic}
    \Require{arbitrary initial policy $\pi^0$, total steps $T$}  
        \For{time step $t \in \{1, 2, \ldots, T \}$}
            \For{player $i \in \{1, \cdots, n\}$}
                \State (response, values) $\gets$ \textsc{ComputeResponse}$(i, \pi^{t-1}_{-i})$   \label{}
                \State $\pi^t_i \gets$ \textsc{UpdatePolicy}$(i$, response, values$)$
            \EndFor
        \EndFor
        \State \Return $\pi^T$
\end{algorithmic}  
\end{algorithm}

Several algorithms fit into this general framework. For example, the simplest is tabular iterated best response (IBR),
where the \textsc{UpdatePolicy} simply overwrites the policy with a best response policy.
Classical fictitious play (FP) is obtained when \textsc{ComputeResponse} returns a best response and
\textsc{UpdatePolicy} updates the policy to be the average of all best responses seen up to time $t$.
Stochastic fictitious play (SFP)~\cite{Fudenberg93SFP} is obtained when the best response policy is defined as
a softmax over action values rather than the argmax, i.e. returning a policy
\[
\pi_i(a) = \textsc{Softmax}_{\lambda}(r^i)(a) = \frac{ \exp(\lambda r^i(a)) }{\sum_a \exp(\lambda r^i(a))},
\]
where $r^i$ is a vector of expected reward for each action.
Exploitability Descent~\cite{lockhart2019computing} defines \textsc{UpdatePolicy} based on gradient descent.
We describe several versions of the algorithm used in Diplomacy below.
A spectrum of game-theoretic algorithms is described in~\cite{lanctot2017unified}.

We run several experiments on different game instances. The number of actions and size of each is listed in
Table~\ref{tab:blotto-games}.

\begin{table}[!ht]
    \centering
    \begin{tabular}{rrr|r|r}
        $n$ & $c$ & $f$ & Number of actions per player $(|A_i|)$ & Size of matrix / tensor $(= |A_i|^n)$\\
        \midrule
        2   & 10   & 3  & 66  & 4356 \\
        2   & 30   & 3  & 496 & 246016 \\
        2   & 15   & 4  & 816 & 665856 \\
        2   & 10   & 5  & 1001 & 1002001 \\
        2   & 10   & 6  & 3003 & 9018009 \\
        \midrule
        3   & 10   & 3  & 66  & 287496 \\
        \midrule
        4   & 8    & 3  & 45  & 4100625 \\
        \midrule
        5   & 6    & 3  & 28  & 17210368 \\
    \end{tabular}
    \caption{Blotto Game Sizes}
    \label{tab:blotto-games}
\end{table}

\subsection{Warm-up: Fictitious Play versus Iterated Best Response (IBR)}
\label{app:blott-warmup}

We first analyze convergence properties of various forms of BRPI, highlighting the difference between IBR and FP approaches. 

\begin{figure}[!ht]
    \centering
    \begin{tabular}{cc}
    \includegraphics[scale=0.5]{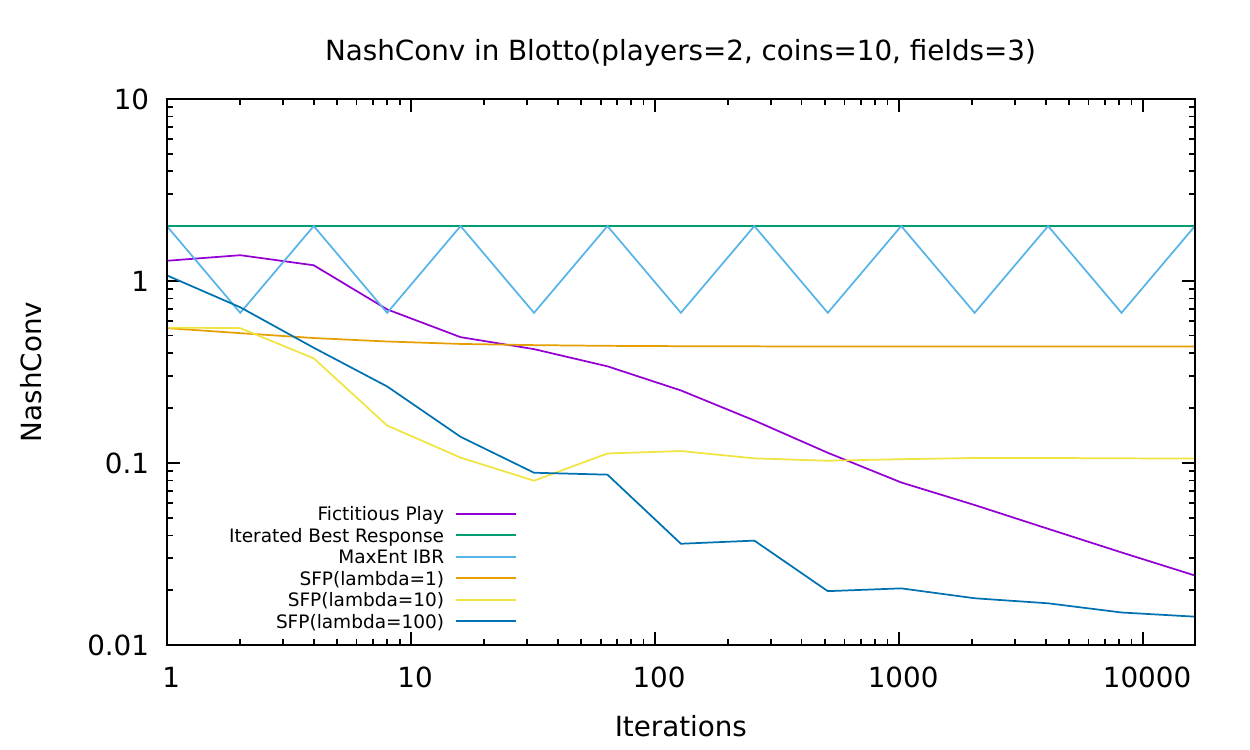} &
    \includegraphics[scale=0.5]{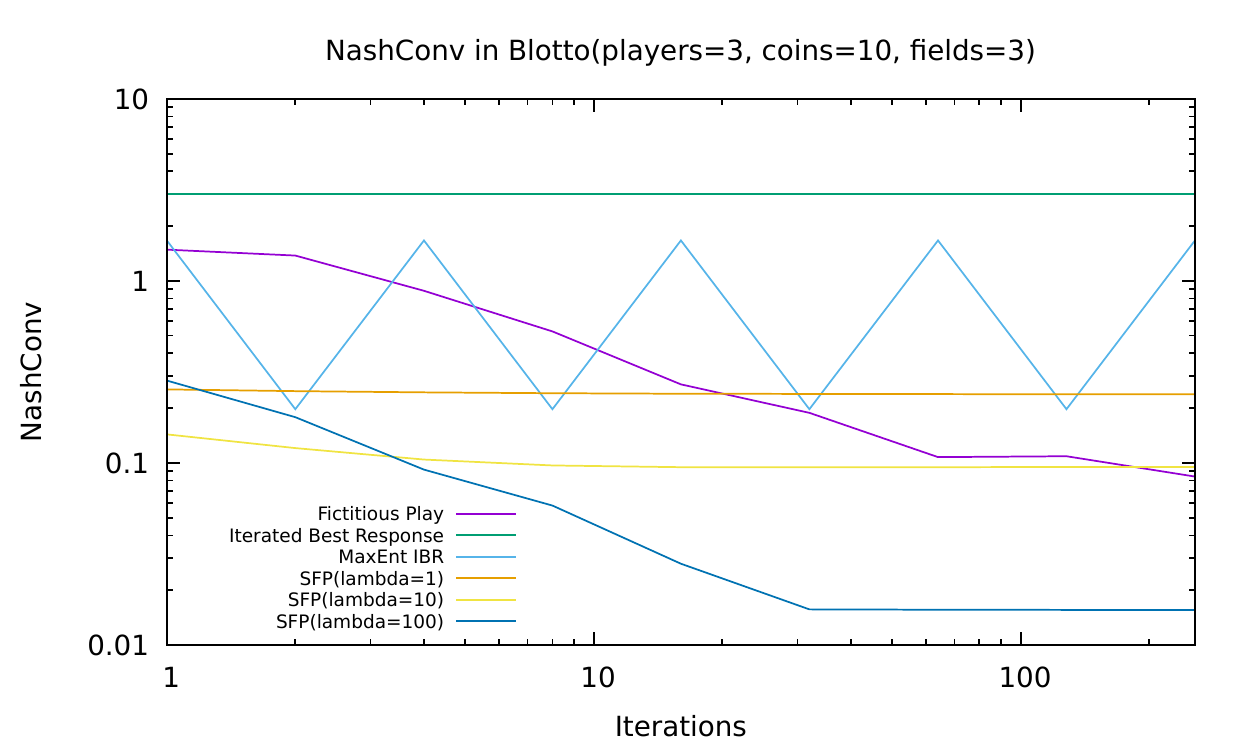}\\
    (a) & (b) \\
    \end{tabular}
    \caption{Convergence of Fictitious Play versus Iterated Best Response variants in (a) Blotto(2, 10, 3),
    and (b) Blotto(3, 10, 3).}
    \label{fig:blotto_fp_ibr}
\end{figure}

Figure~\ref{fig:blotto_fp_ibr}(a) shows the convergence rates to approximate Nash equilibria of fictitious play and iterated best response in Blotto(2, 10, 3). 
We observe that FP is reducing NashConv over time while IBR is not.
In fact, IBR remains fully exploitable on every iteration. It may be cycling around new best responses, but every best response is individually fully exploitable: for every action
$[x, y, z]$ there exists a response of the form $[x + 1, y + 1, z - 2]$ which beats it,
so playing deterministically in Blotto is always exploitable, which
demonstrates the importance of using a stochastic policy.
The convergence graphs for FP and IBR look similar with $n = 3$ players (Figure~\ref{fig:blotto_fp_ibr}(b),
despite both being known to not converge generally, see~\cite{Jordan93} for a counter-example).

One problem with IBR in this case is that it places its entire weight on the newest best response. 
FP mitigates this effectively by averaging, i.e. only moving toward this policy with a step size of
$\alpha_t \approx \frac{1}{t}$. One question, then, is whether having a stochastic operator could improve IBR's convergence rate. We see that indeed MaxEnt IBR-- which returns a mixed best response choosing uniformly over all the tied best response actions-- does find policies that are not fully-exploitable; however,
the NashConv is still non-convergent and not decreasing over time.

Finally, we consider the convergence of Stochastic FP, which seems to reduce NashConv over time as well; however, 
the choice of the temperature $\lambda$ has a strong effect on the convergence curve, 
improving with high $\lambda$ (moving closer to a true best response).
The plateaus occur due to the fact that the introduction of the softmax induces convergence towards Quantal
Response Equilibria (QRE)~\cite{mckelvey1995quantal,hofbauer2002global} rather than Nash equilibria and can be interpreted as
entropy-regularizing the reward~\cite{perolat2020poincar}. Further, unlike FP, SFP is Hannan-consistent~\cite{fudenberg1998theory}
and can be interpreted as a probabilistic best response operator.
This supports the claim that ``stochastic rules perform better than deterministic ones''~\cite{Fudenberg09}.
We elaborate on these points in Appendix~\ref{app:theory}.

\subsection{Fictitious Play using SBR}
\label{app:blott-fp-sbr}

In this subsection, we refer to FP+SBR$(B,C)$ as the instance of algorithm~\ref{alg:gbrpi} that averages the policies like fictitious play but uses a Sampled Best Response operator as described in Algorithm~\ref{alg:sbr}, using $B$ base profiles and $C$ candidates, where the base profile sampling policy is simply $\pi^t$ and the candidate sampling policy is uniform over all actions.

\subsubsection{Trade-offs and Scaling}

\begin{figure}[!ht]
    \centering
    \begin{tabular}{cc}
    \includegraphics[scale=0.52]{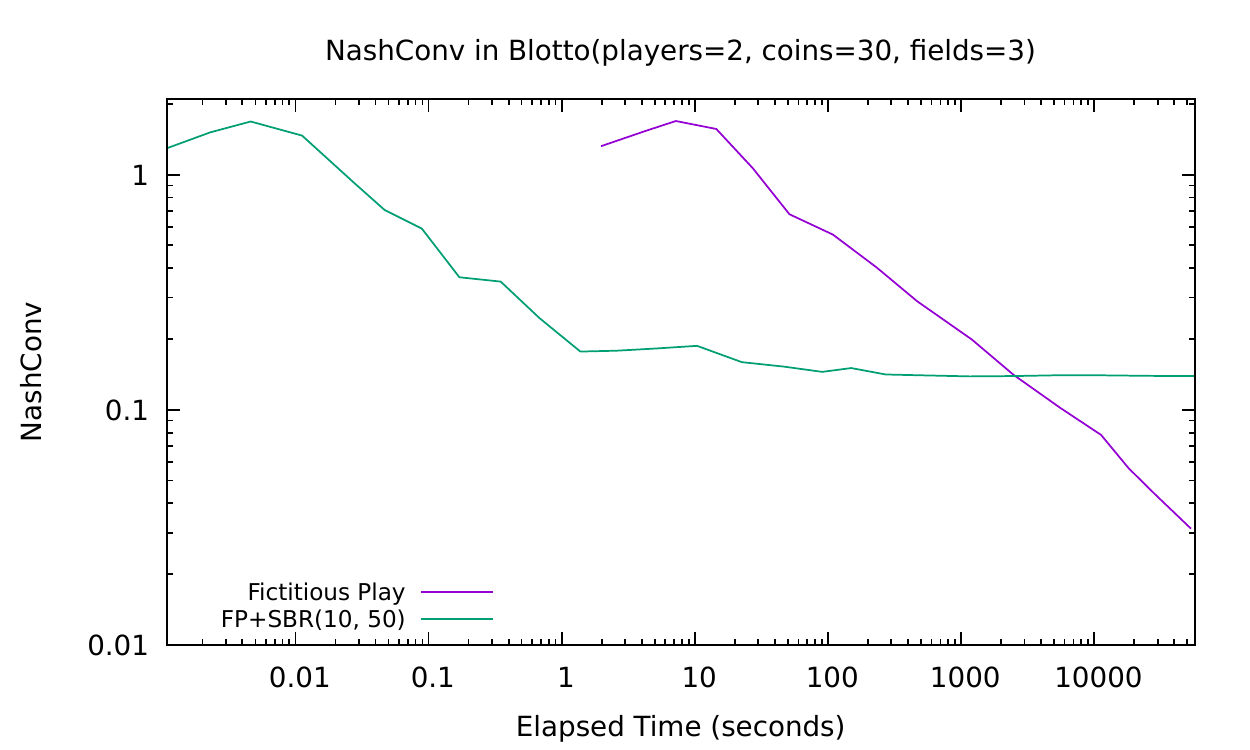} &
    \includegraphics[scale=0.52]{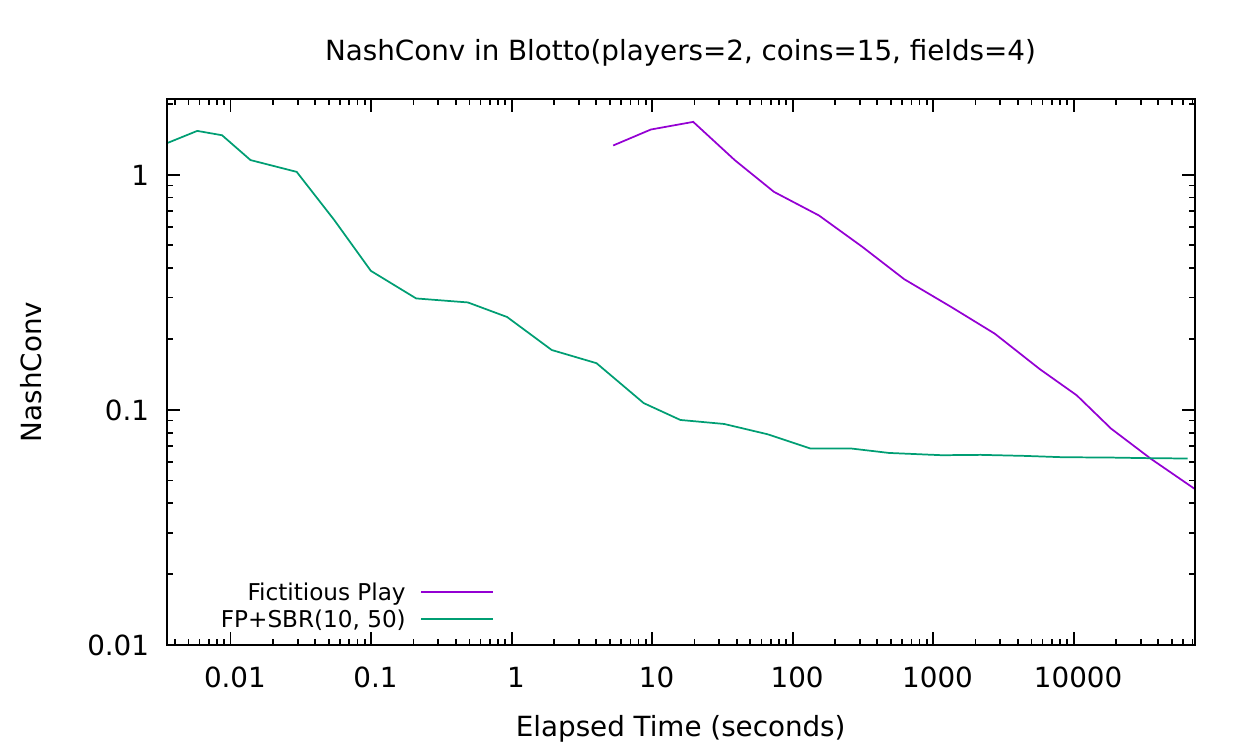}\\
    (a) & (b) \\
    \includegraphics[scale=0.52]{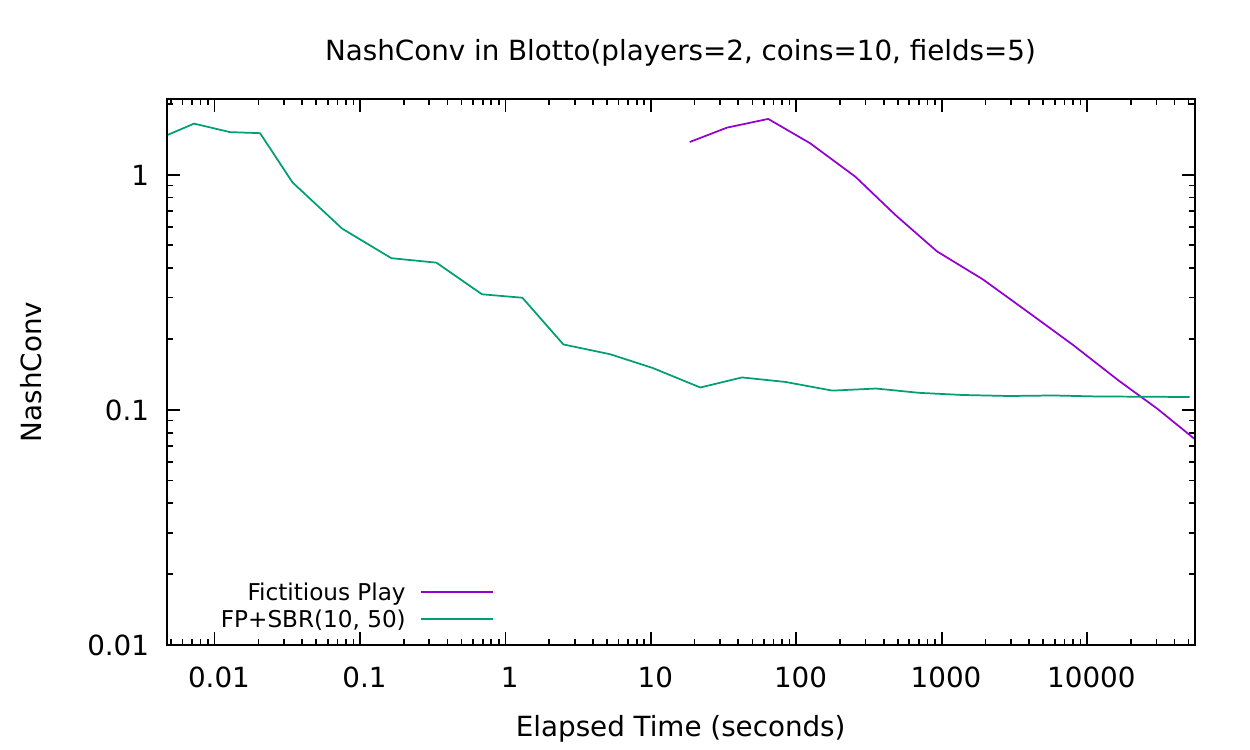} &
    \includegraphics[scale=0.52]{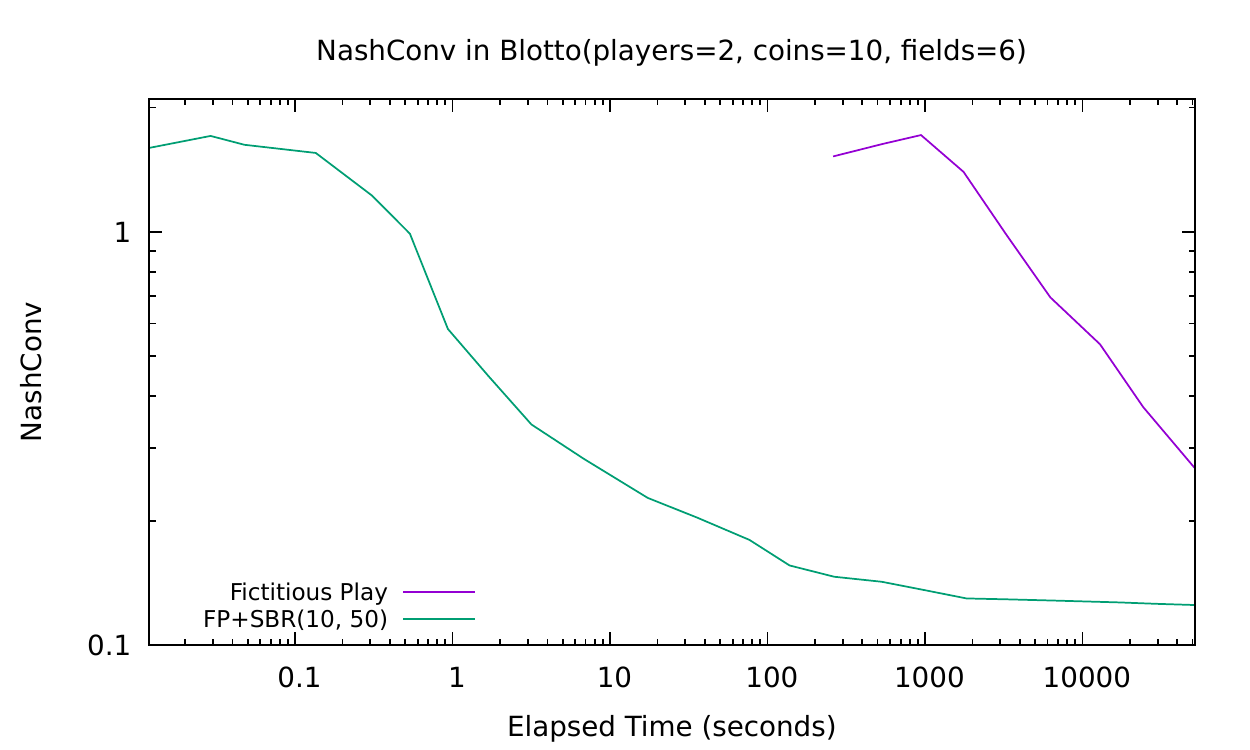}\\
    (c) & (d) \\
    \end{tabular}
    \caption{Convergence of Fictitious Play versus FP+SBR(10,50) by elapsed time in 2-player Blotto with
    increasing action space sizes where (a) $>$ (b) $>$ (c) $>$ (d).}
    \label{fig:blotto2_scaling}
\end{figure}

Figure~\ref{fig:blotto2_scaling} shows several convergence graphs of different 2-player Blotto games with increasing action sizes using elapsed time as the x-axis. The first observation is that, in all cases, FP+SBR computes a policy in a matter of milliseconds, thousands of times earlier than FP's first point. Secondly, it appears that as the action space the game grows, the point by which SBR achieves a specific value and when FP achieves the same value gets further apart: that is, SBR is getting a result with some quality
sooner than FP. For example, in Blotto(2, 10, 6), SBR can achieve the an approximation accuracy in 1 second which takes FP over three hours to achieve. To quantify the trade-offs, we compute a factor of elapsed time to reach the value by FP divided by elapsed time taken to reach NashConv $\approx 0.2$ by FP+SBR(10,50).
These values are 1203, 1421, 2400, and 2834 for the four games by increasing size.
This trade-off is only true up to some threshold accuracy; the benefit of approximation from sampling from SBR leads to plateaus long-term and is eventually surpassed by FP. However, for large enough games even a single full iteration of FP is not feasible.

\begin{figure}[!ht]
    \centering
    \begin{tabular}{ccc}
    \includegraphics[scale=0.35]{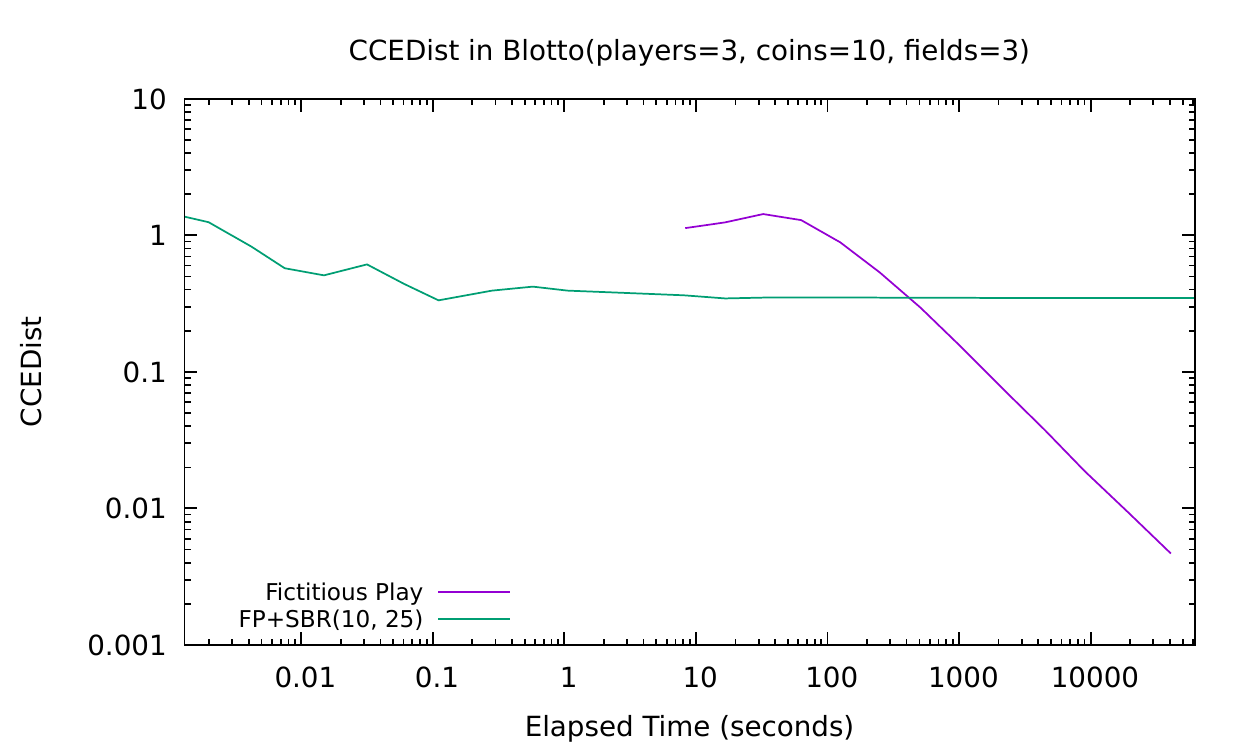} &
    \includegraphics[scale=0.35]{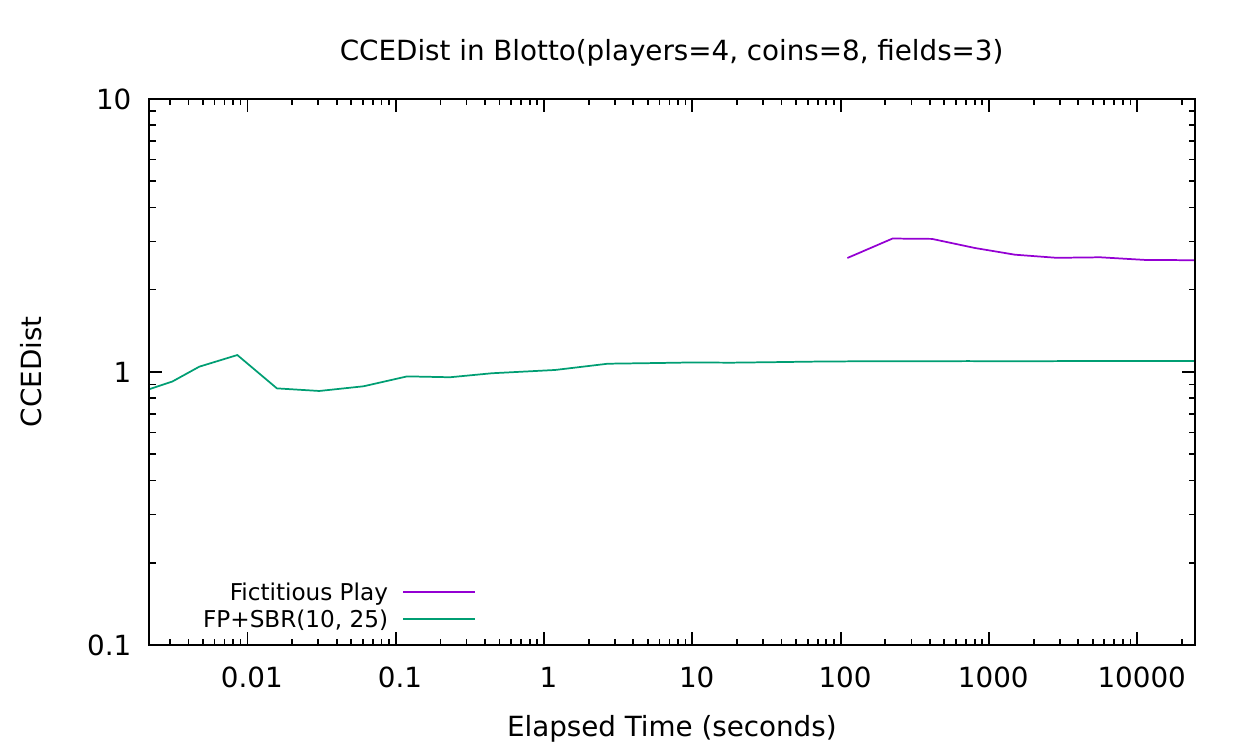} &
    \includegraphics[scale=0.35]{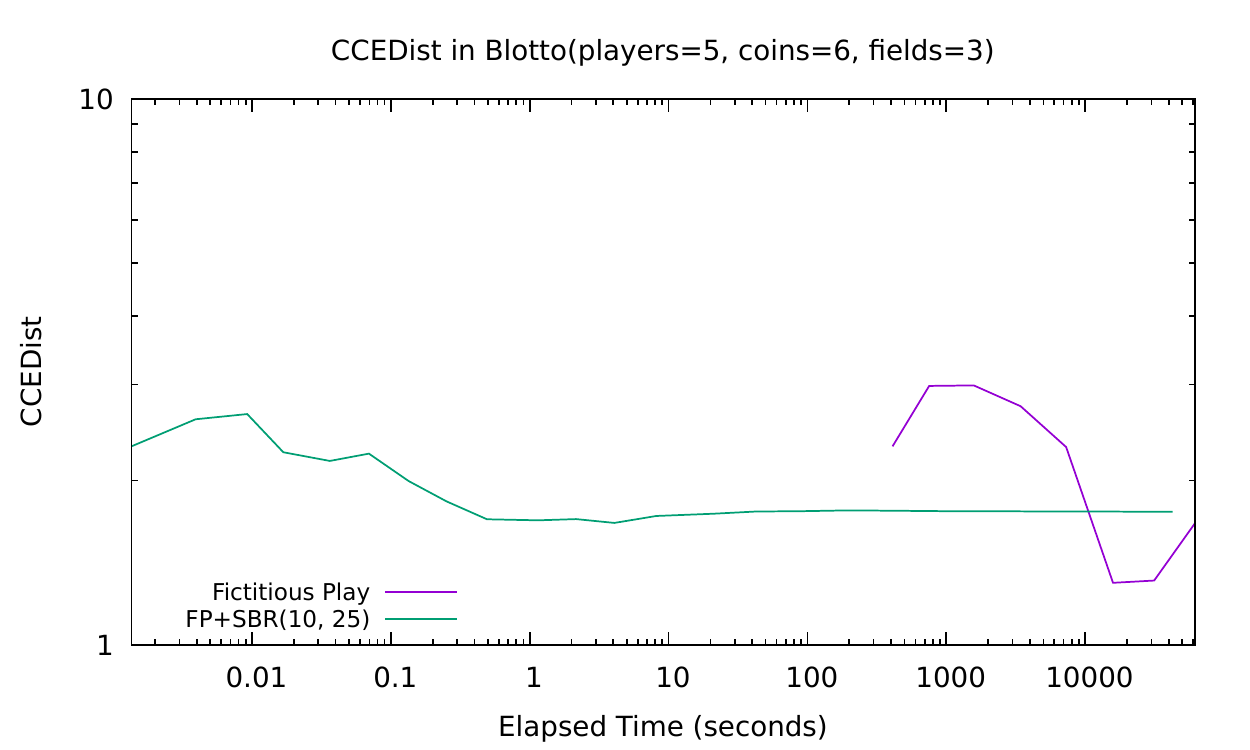} \\
    (a) & (b) & (c) \\
    \end{tabular}
    \caption{Convergence of Fictitious Play versus FP+SBR(10,25) by elapsed time in (a) Blotto(3, 10, 3)
    (b) Blotto(4, 6, 3), and (c) Blotto(5, 6, 3).}
    \label{fig:blotto-mp-scaling}
\end{figure}

The trade-offs are similar in $(n > 2)$-player games. Figure~\ref{fig:blotto-mp-scaling} shows three games of increasing size with reduced action sizes to ensure the game was not too large so joint policies could still fit into memory. There is a similar trade-off in the case of 3-player, where it is clear that FP catches up. In Blotto(4,8,3), even $>25000$ seconds was not enough for the convergence of FP to catch-up, and took $>10000$ seconds for Blotto(5,6,3). In each case, FP+SBR took at most 1 second to reach the CCEDist value at the catch-up point.

\subsubsection{Effects of choices of \texorpdfstring{$B$}{B} and \texorpdfstring{$C$}{C}}

\begin{figure}[!ht]
    \centering
    \begin{tabular}{ccc}
    \includegraphics[scale=0.35]{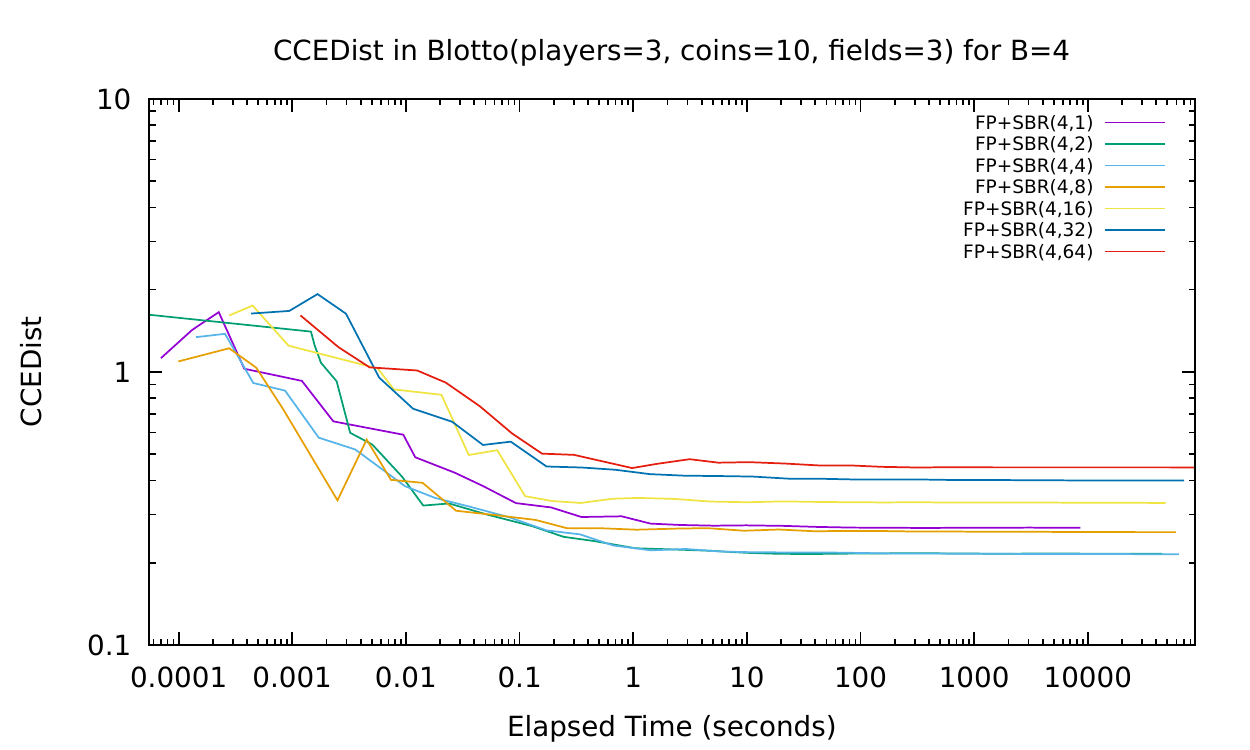} &
    \includegraphics[scale=0.35]{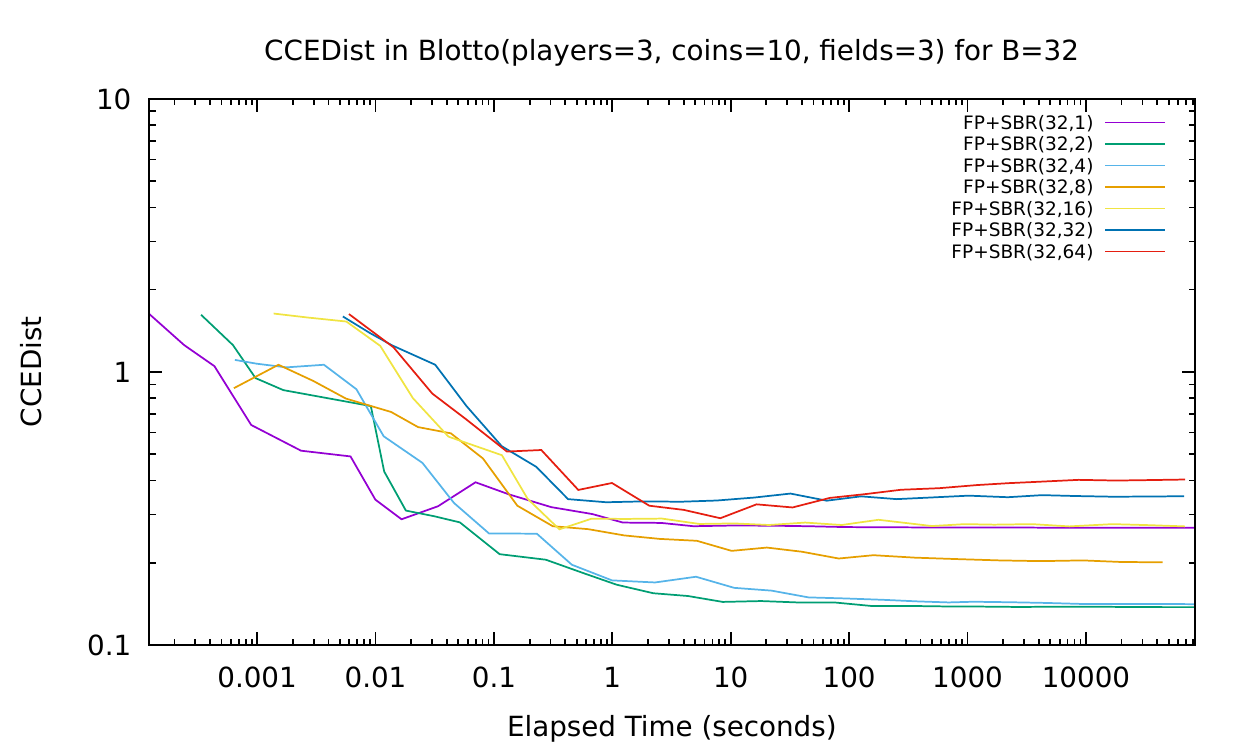} &
    \includegraphics[scale=0.35]{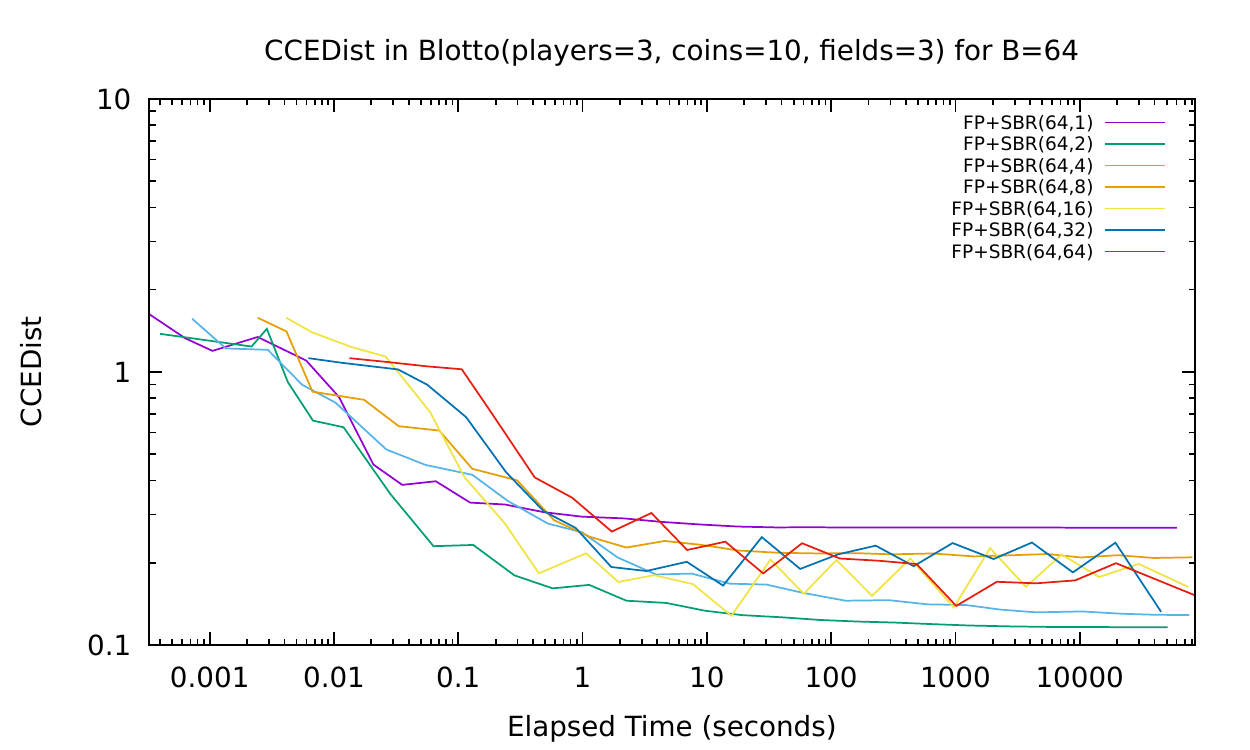} \\
    \includegraphics[scale=0.35]{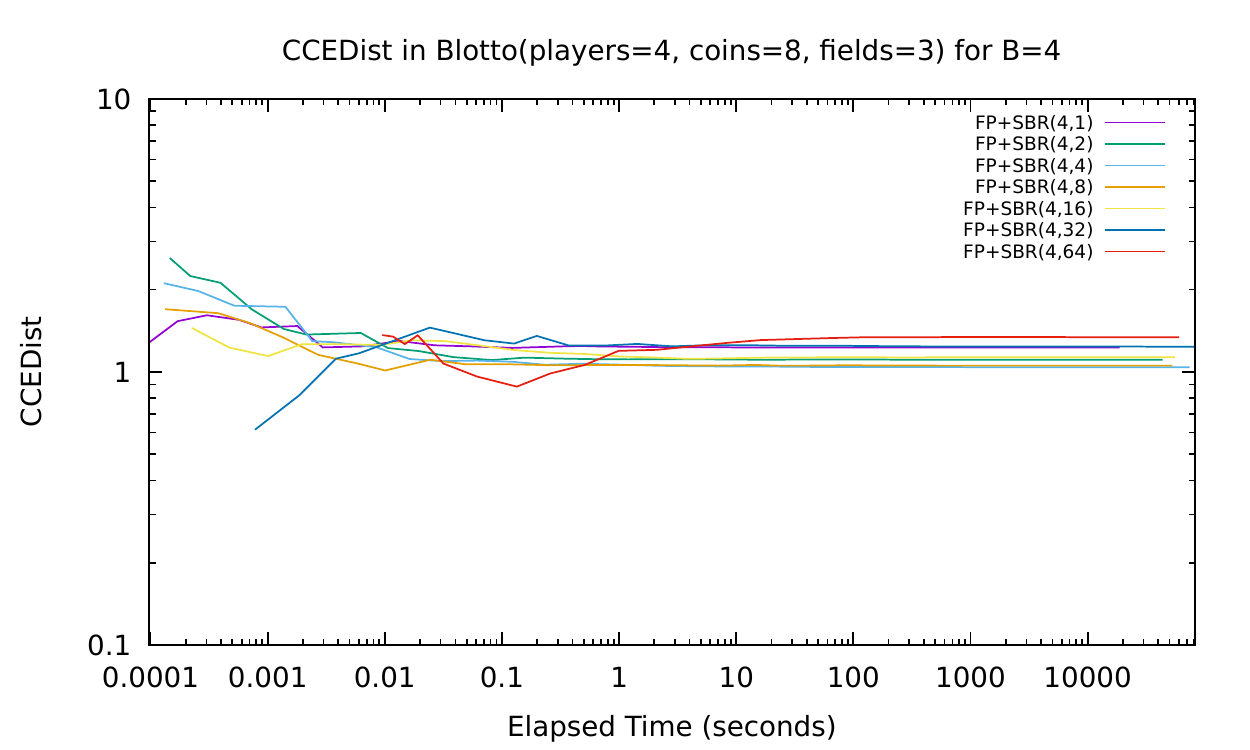} &
    \includegraphics[scale=0.35]{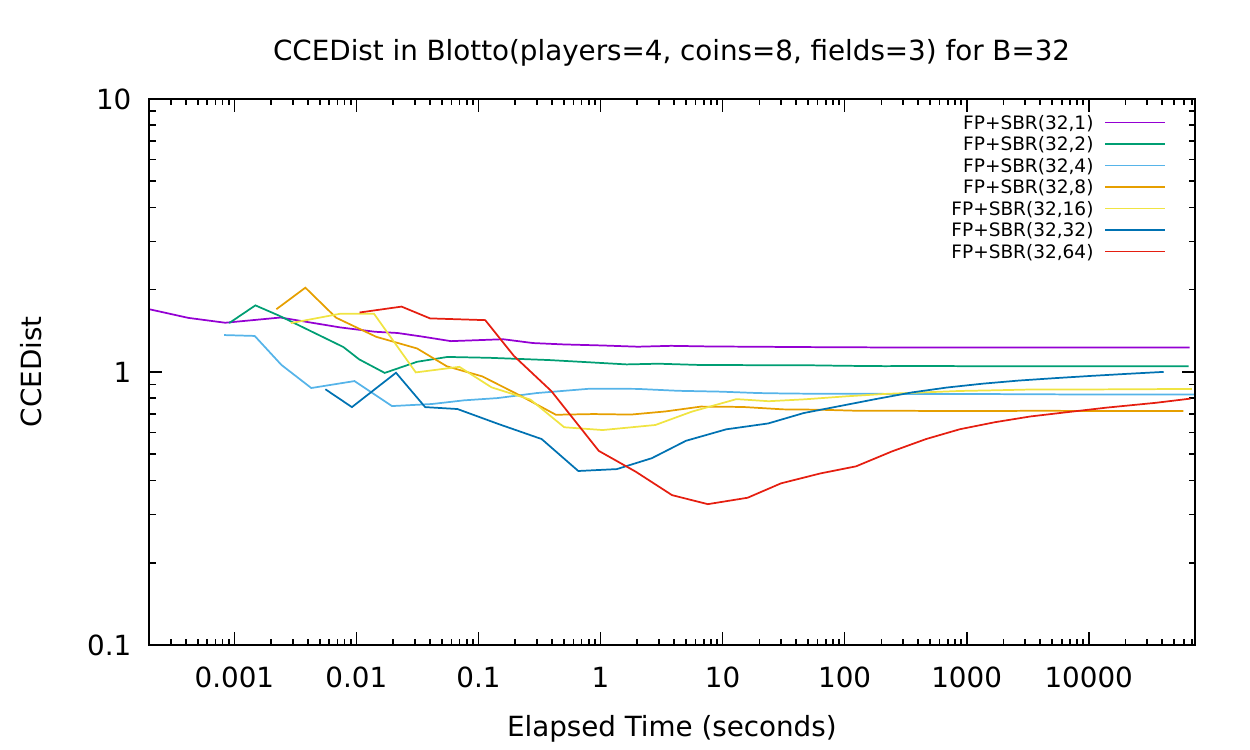} &
    \includegraphics[scale=0.35]{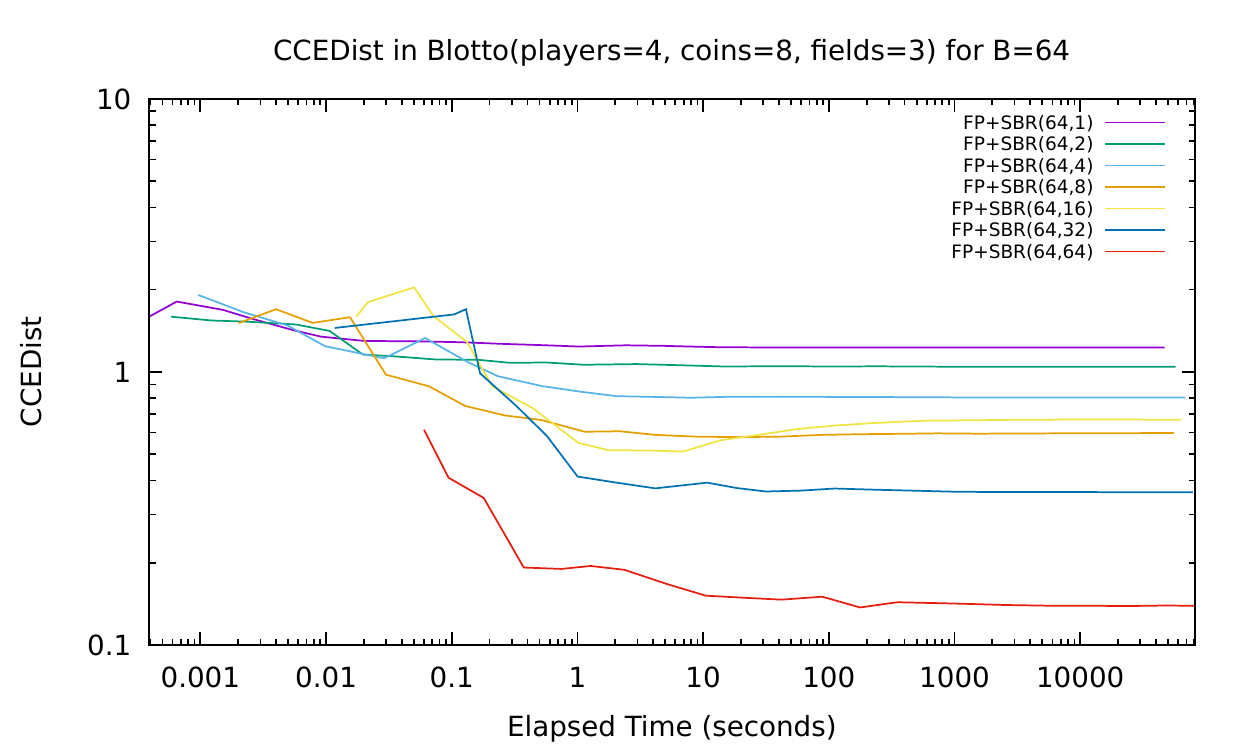} \\
    \includegraphics[scale=0.35]{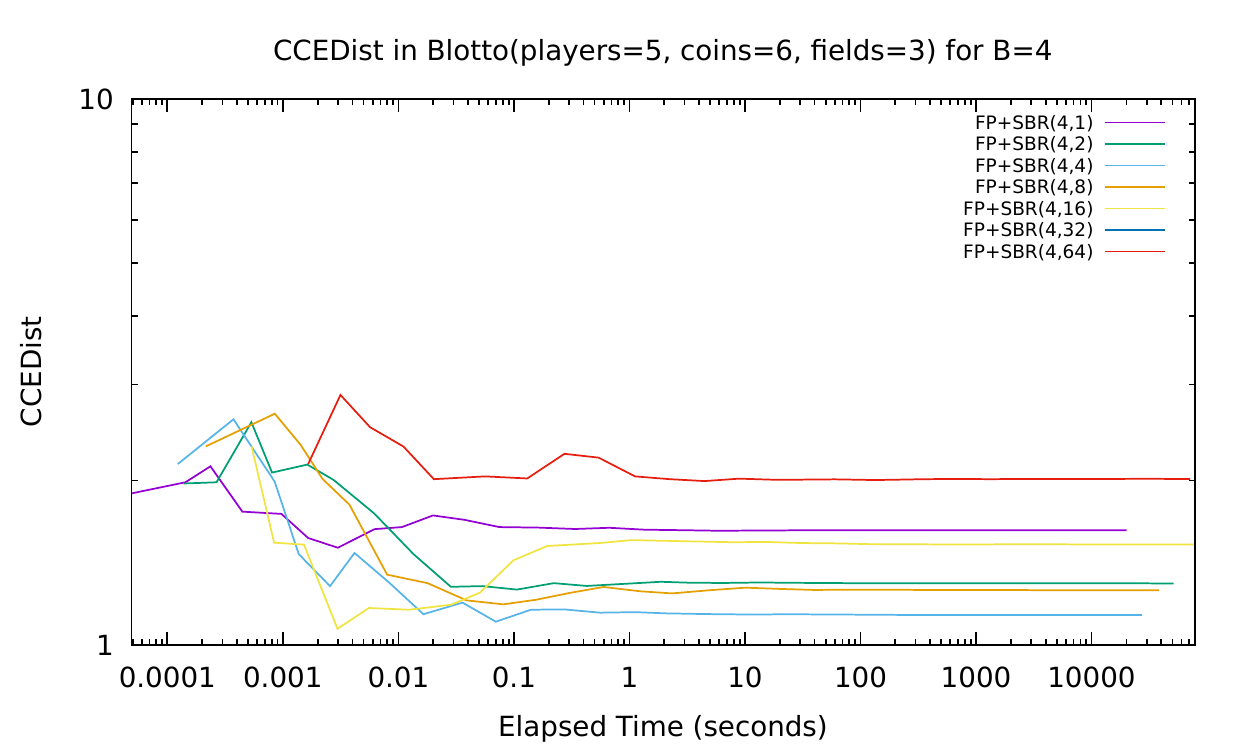} &
    \includegraphics[scale=0.35]{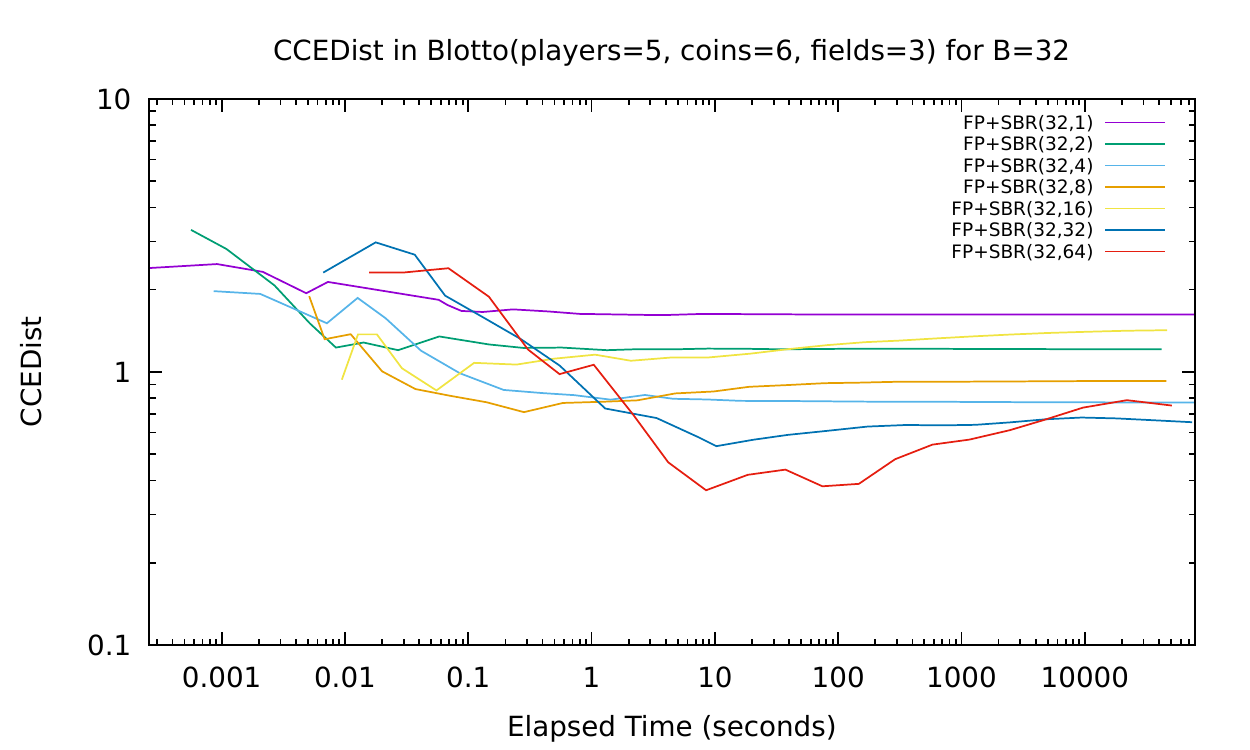} &
    \includegraphics[scale=0.35]{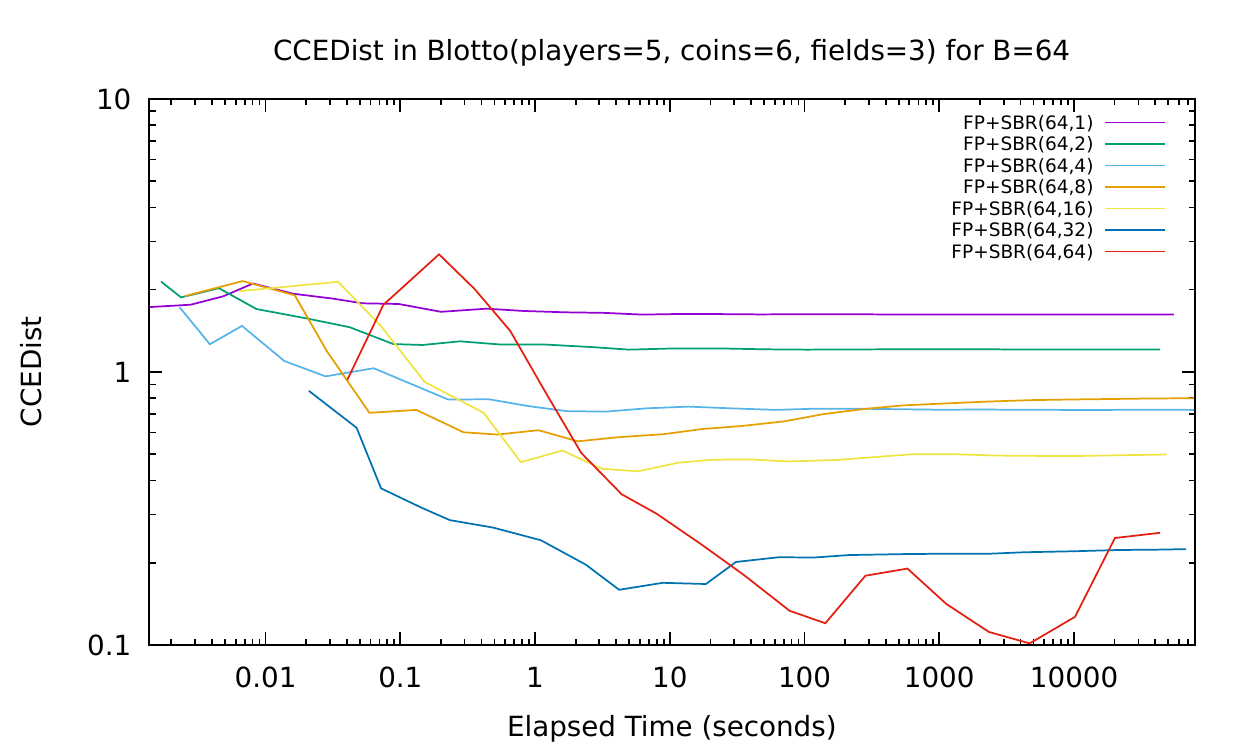} \\
    \end{tabular}
    \caption{Convergence rates of FP+SBR$(B,C)$ for various settings of $B$ and $C$. The first row uses the game of Blotto(3, 10, 3),
    second row Blotto(4, 8, 3), and third row Blotto(5, 6, 3).
    The columns represent $B = 4$, $B = 32$, and $B = 64$, respectively from left to right.}
    \label{fig:fp-sbr-b-c}
\end{figure}

Figure~\ref{fig:fp-sbr-b-c} shows the effect of varying $B$ and $C$ in FP+SBR$(B,C)$. 
Low values of $B$ clearly lead to early plateaus further from equilibrium (graphs for $B < 4$ look similar). At low number of base profiles, it seems that there is a region of low number of candidate samples (2-4) that works best for which plateau is reached, presumably because the estimated maximum over a crude estimate of the expected value has smaller error. As $B$ increases, the distance to equilibria becomes noticeably smaller and sampling more candidates works significantly better than at low $B$.

In the two largest games, FP+SBR(64, 64) was able to reach a CCEDist $\le 0.3$ while fictitious play
was still more than three times further from equilibrium after six hours of running time.

\subsection{BRPI Convergence and Approximation Quality}
\label{app:blotto-sppi}

FP+SBR is an idealized version of the algorithm that explicit performs policy averaging
identical to the outer loop of fictitious play: only the best response step is replaced by a stochastic operator.

We now analyze an algorithm that is closer to emulating BRPI as described in the main paper. Due to stochasticity the policy, the policy trained by BRPI at iteration $t$ for player $i$ can be described as the empirical (joint) distribution:
\[
\pi^t = \frac{1}{N} \sum_{n = 1}^N \bfone(a), \mbox{ where } a \sim \textsc{SBR}(\pi^t_b, \pi^t_c, B, C),
\]
where $\bfone(a)$ is the deterministic joint policy that chooses joint action $a$ with probability 1, 
SBR is the stochastic argmax operator defined in Algorithm~\ref{alg:sbr}, and $\{ \pi^t_b, \pi^t_c \}$ 
are the base profile and candidate sampling policies which are generally functions of
$(\pi^0, \pi^t, \cdots, \pi^{t-1})$. 

The average of the operator over $N$ samples models the best possible fit to dataset of $N$ samples

\subsubsection{Effects of choices of \texorpdfstring{$B$}{B} and \texorpdfstring{$C$}{C}}

To start, in order to compare to the idealized form, we show similar graphs using settings which most closely match FP+SBR: 
$\pi^t_b$ uniformly $t \sim \textsc{Unif}(\{0, \cdots, t-1 \}$ and then samples a base profile from $\pi^t$, and $\pi_c$ samples 
from the initial policy where all players play each action uniformly at random.

\begin{figure}[!ht]
    \centering
    \begin{tabular}{ccc}
    \includegraphics[scale=0.35]{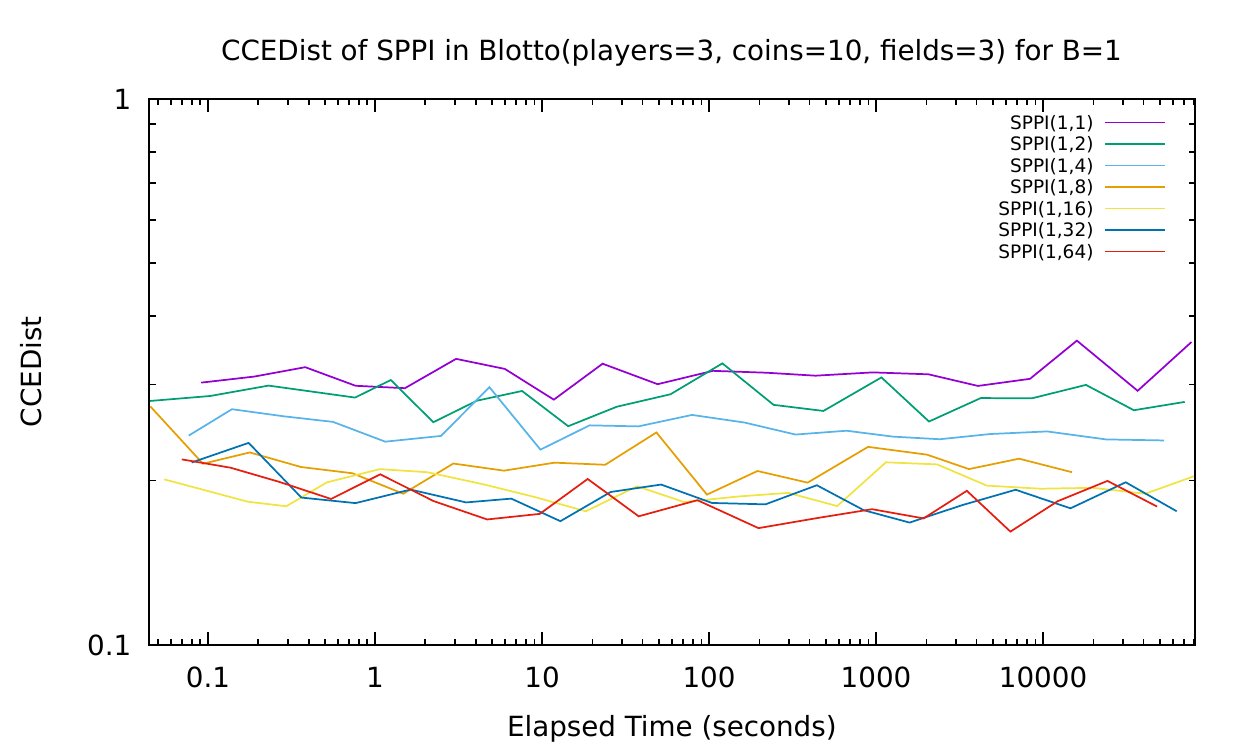} &
    \includegraphics[scale=0.35]{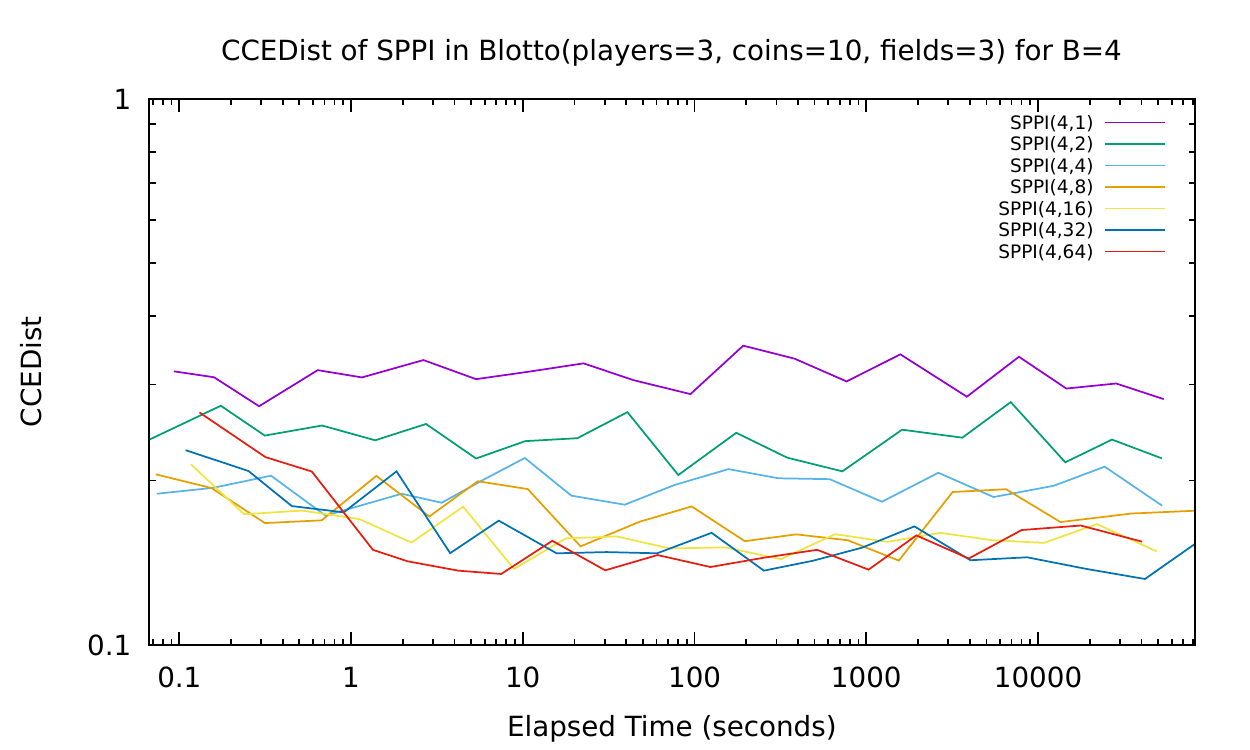} &
    \includegraphics[scale=0.35]{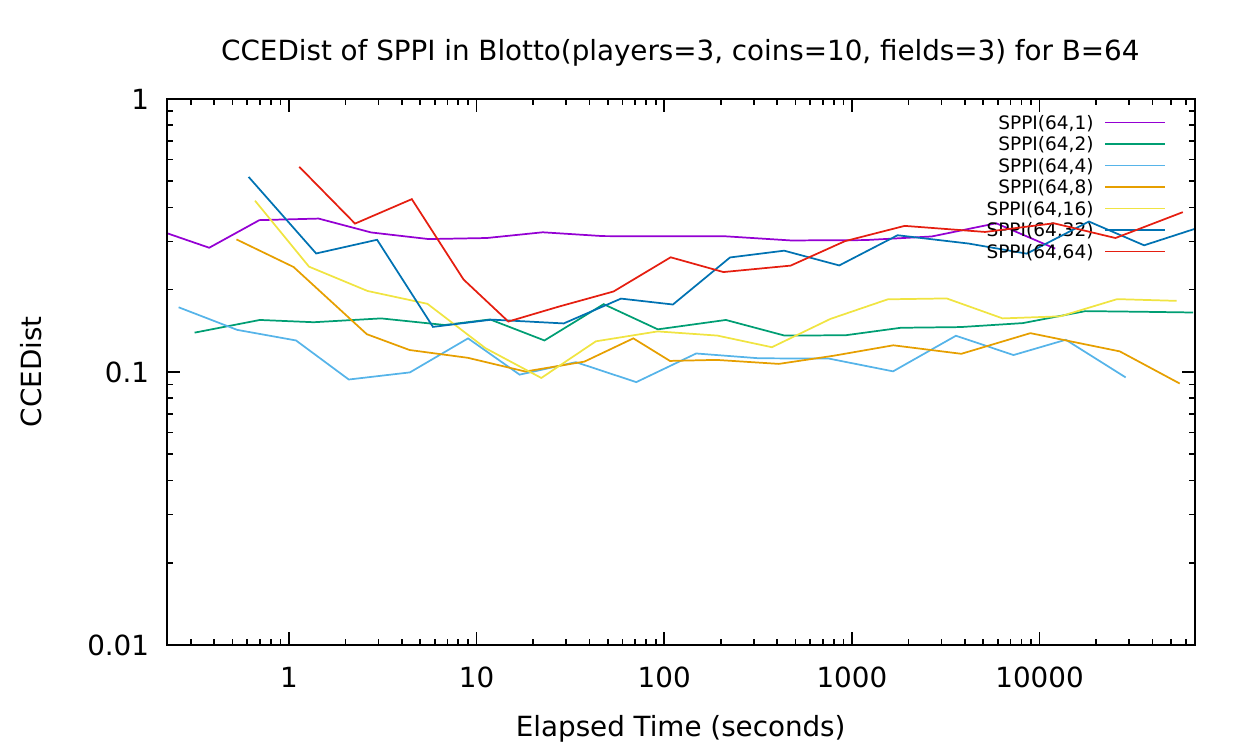} \\
    \includegraphics[scale=0.35]{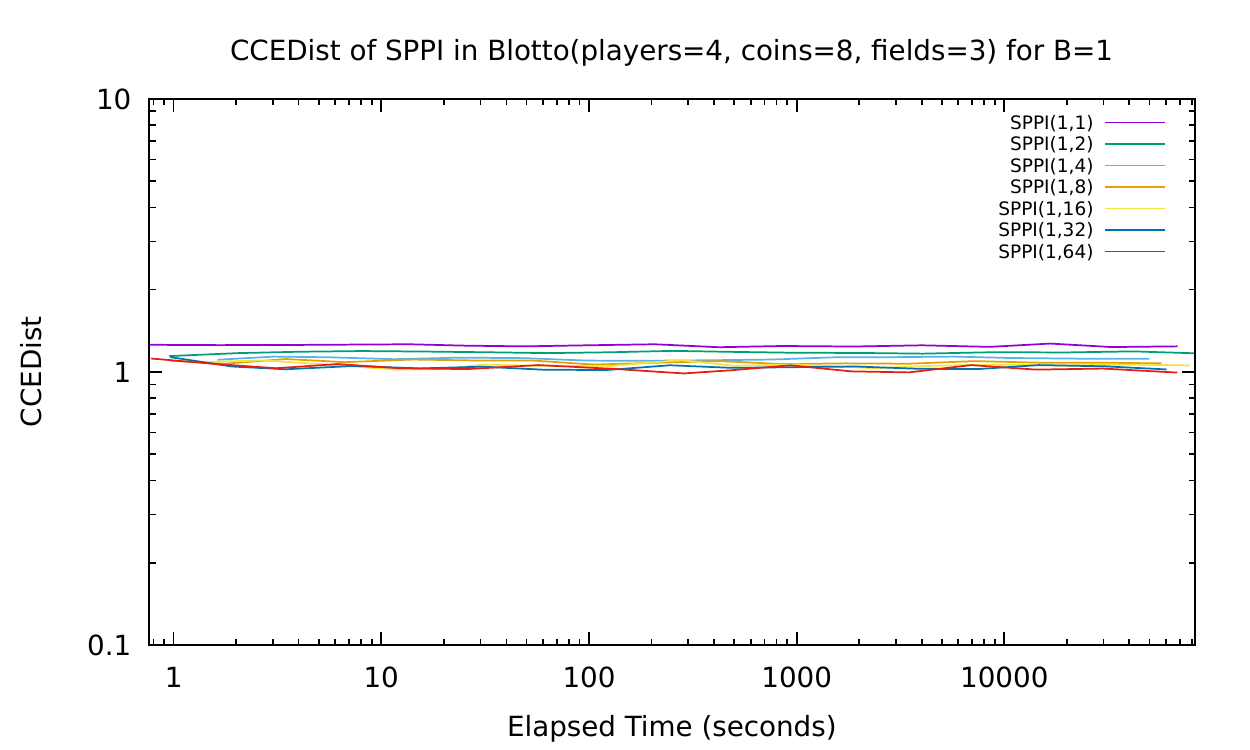} &
    \includegraphics[scale=0.35]{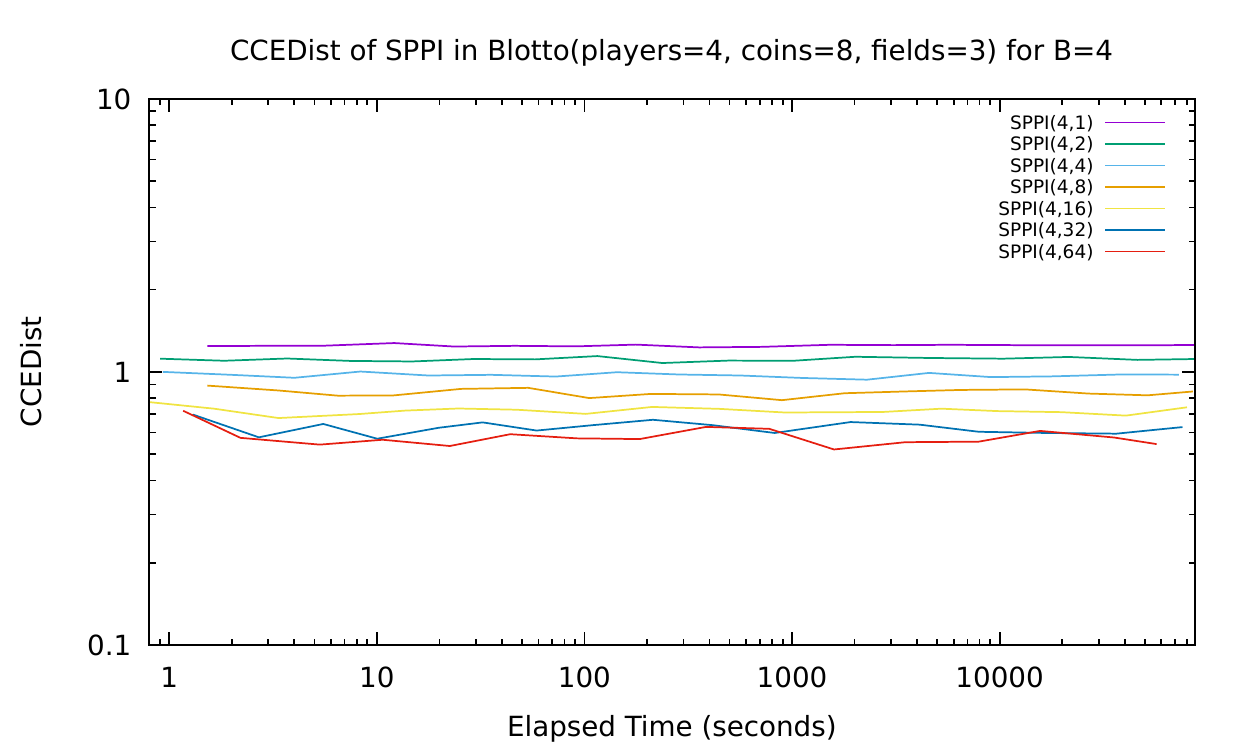} &
    \includegraphics[scale=0.35]{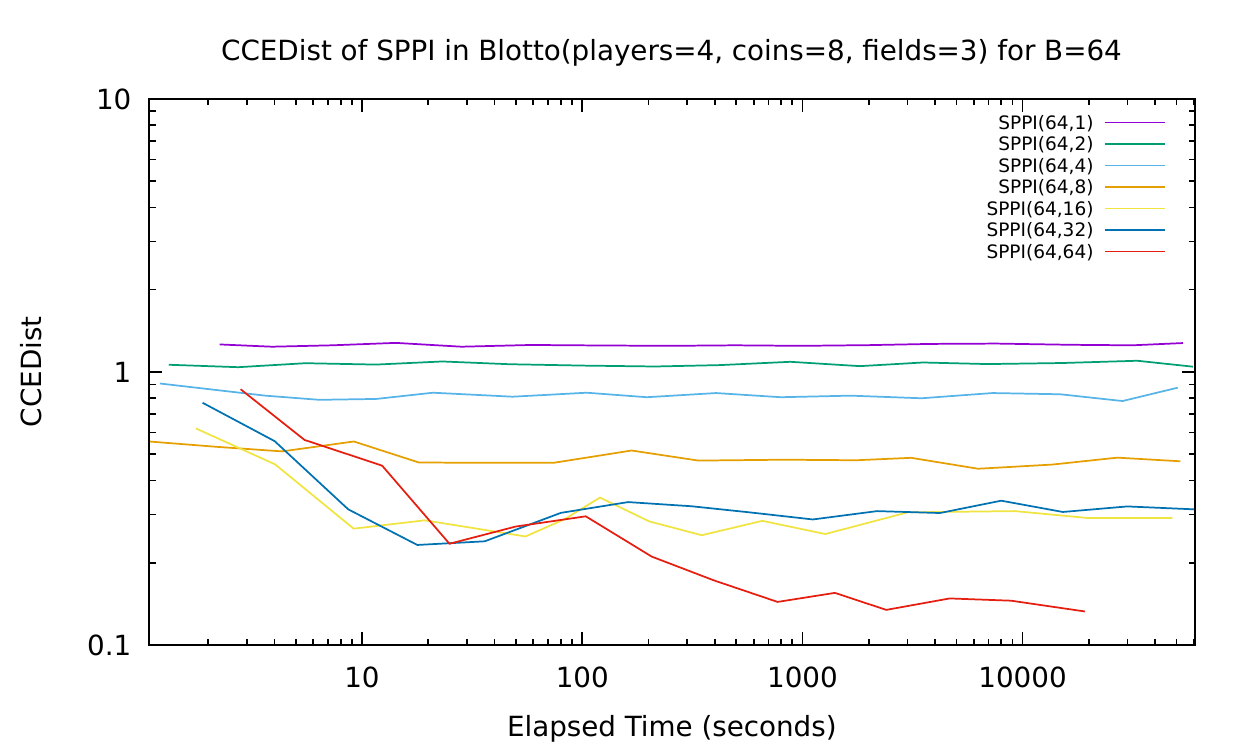} \\
    \includegraphics[scale=0.35]{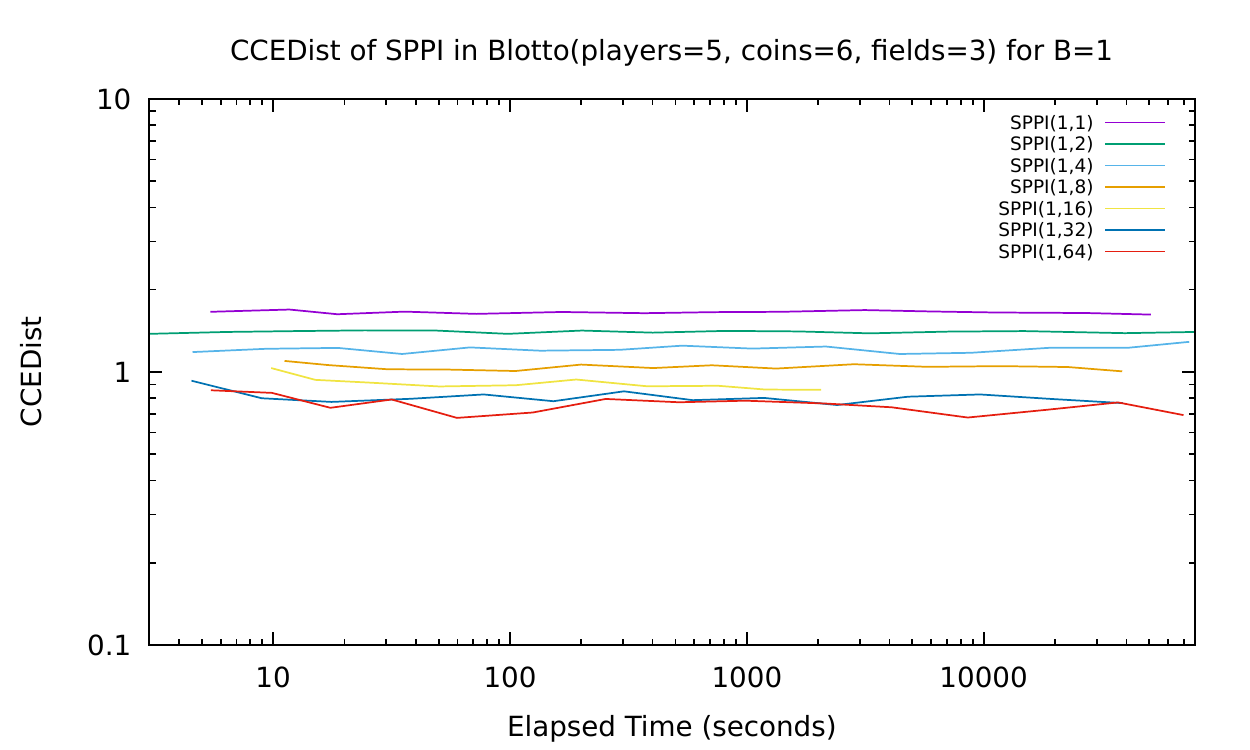} &
    \includegraphics[scale=0.35]{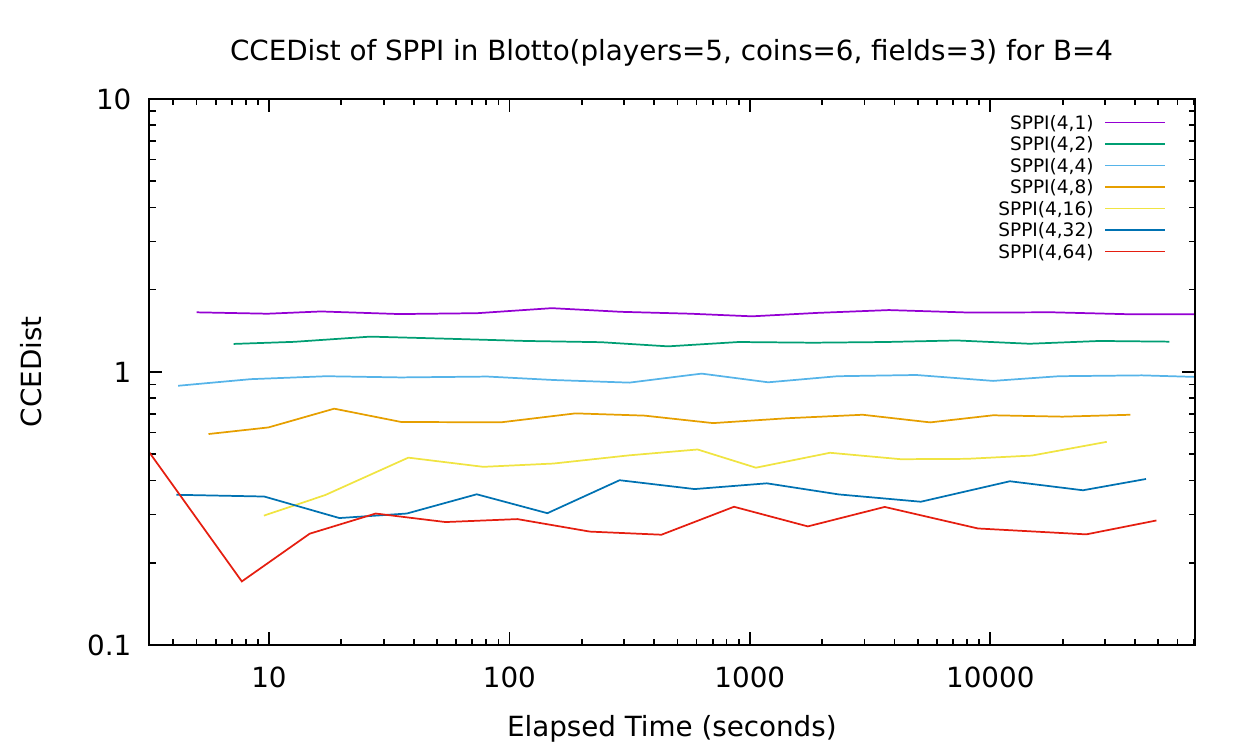} &
    \includegraphics[scale=0.35]{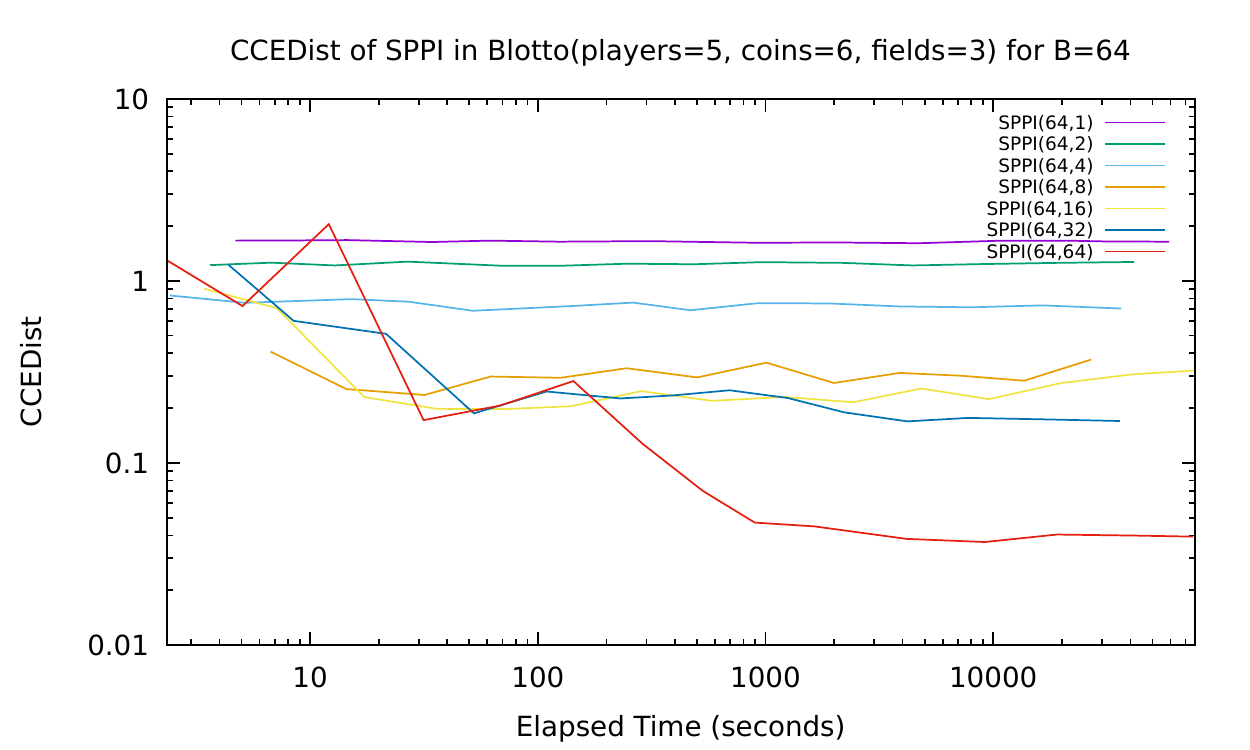} \\
    \end{tabular}
    \caption{Convergence rates of BRPI$(\pi_b, \pi_c, B, C)$ for various settings of $B$ and $C$ using a
    uniform iteration for $\pi_b$ and uniform random action for $\pi_c$, and $N=1000$.
    The first row uses the game of Blotto(3, 10, 3),
    second row Blotto(4, 8, 3), and third row Blotto(5, 6, 3).
    The columns represent $B = 1$, $B = 4$, and $B = 64$, respectively from left to right.}
    \label{fig:sppi-uniform-latest-b-c}
\end{figure}

Figure~\ref{fig:sppi-uniform-latest-b-c} show effects of varying $B$ and $C$ over the games.
Note that convergence to the plateau is much faster than FP+SBR, presumably because of the
$N$ samples per iteration rather than folding one sample into the average.
Like with FP+SBR, the value of $B$ has a strong effect on the plateau that is reached,
and this value is separated by the choice of $C$. 
Unlike FP+SBR the value of $C$ has a different effect: higher $C$ is generally better at lower values of $B$.
This could be due to the fact that, in BRPI, the only way the algorithm can represent a stochastic policy
is via the $N$ samples, whereas FP+SBR computes the average policy exactly;
the error in the max over a crude expectation may be less critical than having a granularity of a
fixed limit of $N$ samples.

This is an encouraging result as it shows that a mixed policy can be trained through multiple samples
from a stochastic best response operator {\it on each iteration}, rather than computing the
average policy explicity. However, this comes at the extra cost of remembering all the past policies;
in large games, this can be done by saving checkpoints of the network periodically and querying them
as necessary.

\subsubsection{Varying the Candidate Sampling Policy}

Most of the runs look similar to the previous subsection (convergence plateau is mostly reached within 60 -- 100 seconds), so
to demonstrate the effect of the various candidate sampling policies, we instead show the CCEDist reached after
running for a long time ($>50000$ seconds). This roughly captures the asymptotic value of each method, rather than
its convergence rate.

\begin{figure}[h!]
\begin{tikzpicture}
  \centering
  \begin{axis}[
        ybar, axis on top,
        title={CCEDist reached by BRPI with $\pi_b$ sampling from a uniform past policy and $B=2$, $C=16$},
        height=6cm, width=12.5cm,
        bar width=0.3cm,
        ymajorgrids, tick align=inside,
        major grid style={draw=white},
        enlarge y limits={value=.1,upper},
        ymin=0, ymax=0.9,
        axis x line*=bottom,
        axis y line*=right,
        y axis line style={opacity=0},
        tickwidth=0pt,
        enlarge x limits=true,
        legend style={
            at={(0.5,-0.2)},
            anchor=north,
            legend columns=-1,
            /tikz/every even column/.append style={column sep=0.3cm}
        },
        ylabel={CCEDist},
        symbolic x coords={
           Blotto3, Blotto4, Blotto5
        },
       xtick=data,
    ]
    
    \addplot [draw=none, fill=blue!30] coordinates {
        (Blotto3, 0.15)
        (Blotto4, 0.89)
        (Blotto5, 0.58)
     };
   \addplot [draw=none,fill=red!30] coordinates {
        (Blotto3, 0.18)
        (Blotto4, 0.27)
        (Blotto5, 0.19)
     };
   \addplot [draw=none, fill=green!30] coordinates {
        (Blotto3, 0.32)
        (Blotto4, 0.4)
        (Blotto5, 0.58)
     };
    \addplot [draw=none, fill=cyan!30] coordinates {
        (Blotto3, 0.15)
        (Blotto4, 0.51)
        (Blotto5, 0.24)
     };
   \addplot [draw=none,fill=magenta!30] coordinates {
        (Blotto3, 0.18)
        (Blotto4, 0.57)
        (Blotto5, 0.19)
    };

    \legend{Initial, Uniform, Latest, Initial + Uniform, Initial + Latest }
  \end{axis}
\end{tikzpicture}
\caption{Long-term CCEDist reached by BRPI$(2,16)$ with $\pi_b$ choosing a uniform past policy in Blotto(3, 10, 3) (left), Blotto(4, 8, 3) (middle), and Blotto(5, 6, 3) (right).}
\label{fig:blotto-candidate-sampling-uniform-pib-2-16}
\end{figure}
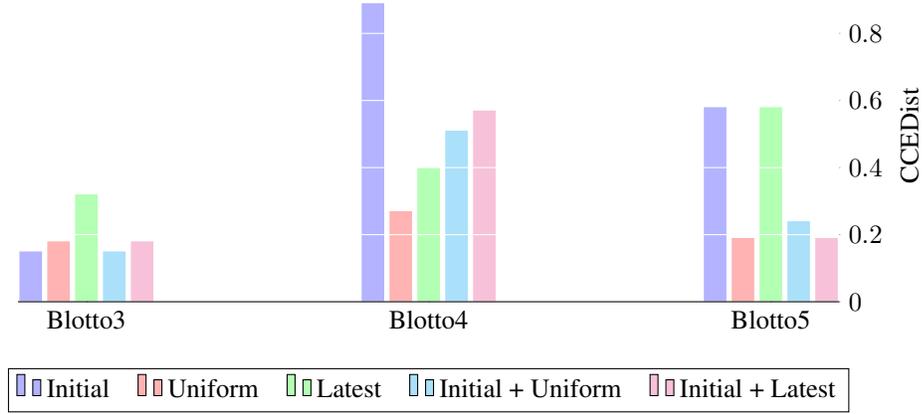

Figure~\ref{fig:blotto-candidate-sampling-uniform-pib-2-16} shows the long-term CCEDist reached by BRPI at $B=2$ and $C=16$ for various
choices of the candidate sampling schemes. There is a clearly best choice of using uniformly sampled past policy to 
choose candidates followed by the mixtures: initial + uniform, and initial + latest.

\subsubsection{Iterated Sampled Best Response}
\label{app:blotto-ibr}

We now consider the case where the base sampling policy is the last policy: $\pi^t_b = \pi^{t-1}$. 
Figure~\ref{fig:blotto-candidate-sampling-latest-pib-2-16} shows the long-term CCEDist reached by BRPI at $B=2$ and $C=16$ for various choices of the candidate sampling schemes.

\begin{figure}[h!]
\begin{tikzpicture}
  \centering
  \begin{axis}[
        ybar, axis on top,
        title={CCEDist reached by BRPI with $\pi^t_b = \pi^{t-1}$ and $B=2$, $C=16$},
        height=6cm, width=12.5cm,
        bar width=0.3cm,
        ymajorgrids, tick align=inside,
        major grid style={draw=white},
        enlarge y limits={value=.1,upper},
        ymin=0, ymax=2.5,
        axis x line*=bottom,
        axis y line*=right,
        y axis line style={opacity=0},
        tickwidth=0pt,
        enlarge x limits=true,
        legend style={
            at={(0.5,-0.2)},
            anchor=north,
            legend columns=-1,
            /tikz/every even column/.append style={column sep=0.3cm}
        },
        ylabel={CCEDist},
        symbolic x coords={
           Blotto3, Blotto4, Blotto5
        },
       xtick=data,
    ]
    
    \addplot [draw=none, fill=blue!30] coordinates {
       (Blotto3, 0.18)
       (Blotto4, 0.91)
       (Blotto5, 0.60)
     };
   \addplot [draw=none,fill=red!30] coordinates {
       (Blotto3, 0.22)
       (Blotto4, 0.32)
       (Blotto5, 2.34)
     };
   \addplot [draw=none, fill=green!30] coordinates {
       (Blotto3, 0.5)
       (Blotto4, 0.37)
       (Blotto5, 1.64)
     };
    \addplot [draw=none, fill=cyan!30] coordinates {
       (Blotto3, 0.17)
       (Blotto4, 0.64)
       (Blotto5, 1.11)
     };
   \addplot [draw=none,fill=magenta!30] coordinates {
       (Blotto3, 0.15)
       (Blotto4, 0.52)
       (Blotto5, 0.19)
    };

    \legend{Initial, Uniform, Latest, Initial + Uniform, Initial + Latest }
  \end{axis}
\end{tikzpicture}
\caption{Long-term CCEDist reached by BRPI$(2,16)$ with with $\pi^t_b = \pi^{t-1}$ in Blotto(3, 10, 3) (left), Blotto(4, 8, 3) (middle), and Blotto(5, 6, 3) (right).}
\label{fig:blotto-candidate-sampling-latest-pib-2-16}
\end{figure}
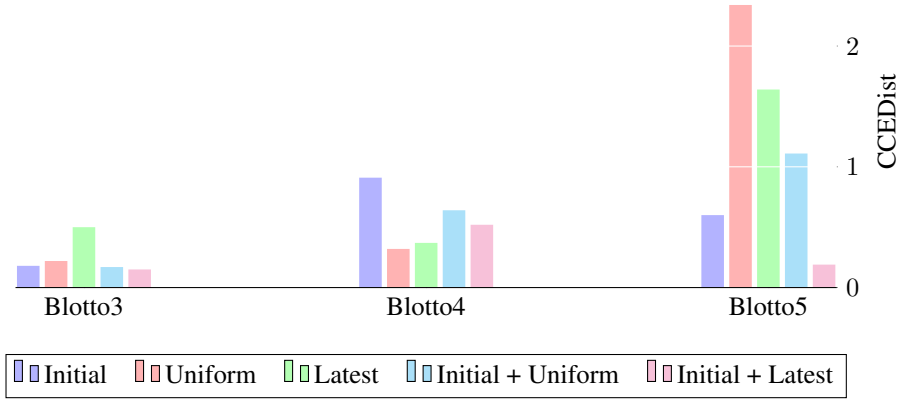

In this case, there is no clear winner in all cases, but initial + latest seems to be a safe choice among these five sampling schemes. Despite the plateau values being generally higher than $\pi_b$ sampling from a uniform past policy across the candidate sampling schemes (note the y-axis scale differs between the two plots),
the values under the initial + latest sampling policy are matched in two of the three games.

Though using the last policy for $\pi_b$ yields generally higher final plateaus in CCEDist, the fact that these are comparable to those achieved by FP+SBR stands in stark contrast to when we use an exact best response operator in \ref{app:blott-warmup}; in that case, Iterated Best Response makes no progress. This shows that using a Sampled Best Response in the place of an exact one can dramatically improve the behaviour of Iterated Best Response.

\section{Additional Results}
\label{app:results}
\subsection{Behavioural analysis}

We present some descriptive statistics of the behaviour exhibited by the different networks in self-play games, and by human players in the datasets. We examine the move-phase actions of agents, investigating the tendency of agents to support another power's unit to move or hold, which we refer to as ``cross power support''. We also examine the success rates, which are defined by whether the other power made a corresponding move for that unit (respectively, either the target move or a non-moving order). Figure~\ref{fig:cross_supports} compares the proportion of actions that are cross power support across different agents, and their success (for both holding and moving). The results indicate the BRPI agents have a substantially reduced rate of cross power hold support, and the BRPI agents have a substantially increased rate of cross-power move support. The A2C agent attempts both types of support less often but succeeds a higher proportion of the time.

This analysis is related but different to the cross-support analysis in \cite{paquette2019no}, which considers cross-power supports as a proportion of supports, and rather than looking at ``success'' as we've defined it for support orders, they measure ``effectiveness'', defined in terms of whether the support made a difference for the success of the move order or defence being supported.

We also examine the propensity of agents to {\it intrude on} other agents, defined as one of the following:
\begin{itemize}
    \item a move order (including via convoy) into a territory on which another agent has a unit
    \item a move order (including via convoy) into a supply center owned by another agent (or remaining/holding in another agent's supply center during fall turns)
    \item successfully supporting/convoying a move falling in the categories above
\end{itemize}

We define two powers to be \emph{at conflict} in a moves phase\footnote{We exclude other phases from this analysis.} if either one intrudes upon the other, and to be \emph{distant} if neither one has the option of doing that. Then we define the \emph{peace proportion} of a network to be the proportion, among instances in which powers are non-distant, that they're not at conflict; and the \emph{peace correlation} to be the correlation, among those same instances, between conflict in the current moves phase and conflict in the next moves phase. In Figure~\ref{fig:peace} we compare these statistics across our different networks. These results indicate that BRPI reduces the peace proportion while maintaining the peace correlation, while A2C brings down both the peace proportion and the peace correlation significantly.

\begin{figure}[!tb]
    \centering
    \subfloat[][Comparison of the cross-power support behaviours of different networks.]{
    \centering
    \includegraphics[width=.45\textwidth]{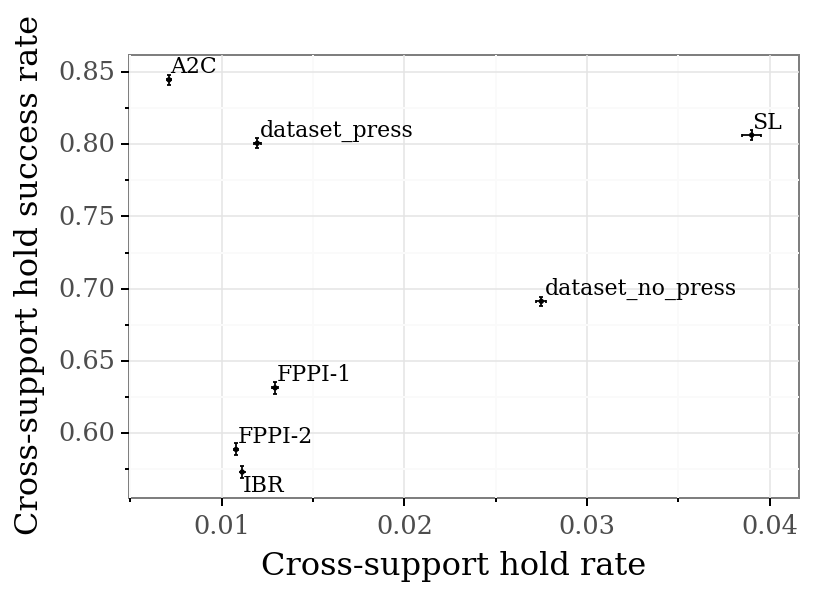}\includegraphics[width=.45\textwidth]{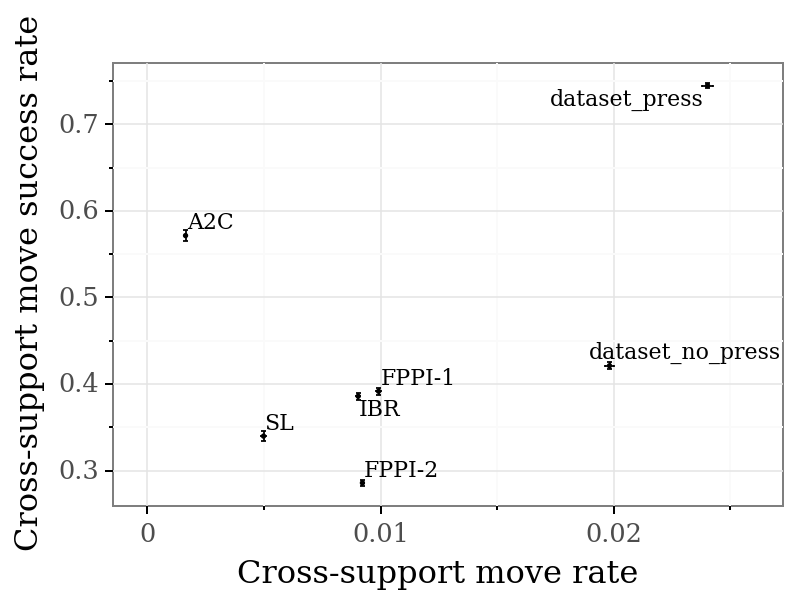}
    \label{fig:cross_supports}
    }
    
    \subfloat[][Comparison of peace correlations between different networks.]{
    \centering
    \includegraphics[width=.45\textwidth]{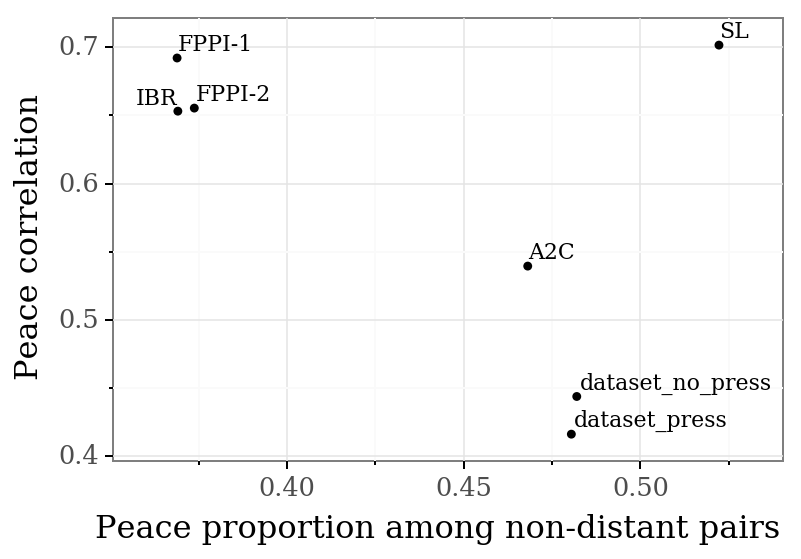}
    \label{fig:peace}
    }
    \caption{Descriptive behavioural statistics of the different networks, as well as human play in the datasets.}
\end{figure}

We used sampling temperature $0.1$ for these agents, and considered only the first 10 years of each game in order to make the comparisons more like-for-like, particularly to human games. We used 1000 self-play games for each network; results for IBR, FPPI-1, and FPPI-2 were combined from results for the final networks from 5 training runs. In addition we included the evaluation sets from both our press and no-press datasets, excluding games that end within 5 years and games where no player achieved 7 supply centers.

\subsection{Exploitability}
\label{app:exploitability}
In this section we give more details on our experiments on the exploitability of our networks. We use for two different exploiters for each agent we exploit, both of which are based on the Sampled Best Response operator, using a small number of samples from the policy as the base profiles to respond to.

Firstly, we use a few shot exploiter. Apart from using base profiles from the policy being exploited, but otherwise is independent from it -- the value function for SBR is taken from an independent BRPI run (the same for all networks exploited), and the candidates from the human imitation policy; SBR($\pi^c=\pi^\textrm{SL}$, $\pi^b=\pi$, $v=V^{\textrm{RL}}$). This has the advantage of being the most comparable between different policies; the exploits found are not influenced by the strength of the network's value function, or by the candidates they provide to SBR. This measure should still be used with care; it is possible for an agent to achieve low few-shot exploitability without being strong at the game, for example by all agents playing pre-specified moves, and uniting to defeat any deviator.

The other exploiter shown is the best found for each policy. For policies from the end of BRPI training, this is SBR($\pi^c=\pi^\textrm{SL} + \pi$, $\pi^b=\pi$, $v=V^{\pi}$), which uses a mixture of candidates from the exploitee and supervised learning, and the exploitee's value function. For $\pi^{SL}$, we instead use SBR($\pi^c=\pi^\textrm{SL} + \pi^\textrm{RL}$, $\pi^b=\pi$, $v=V^{\textrm{RL}}$), where $\pi^\textrm{RL}$ and $v=V^{\textrm{RL}}$ are from a BRPI run. This is because the value function learned from human data is not correct for $\pi^{SL}$, and leadds to weak exploits.

Figure \ref{fig:exploitability} shows the winrates achieved by each of these exploiters playing with $6$ copies of a network, for our supervised learning agent and the final agent from one run of each BRPI setting. All these networks are least exploitable at $t=0.5$; this appears to balance the better strategies typically seen at lower temperatures with the mixing needed to be relatively unexploitable. In the few-shot regime, IBR and FPPI-2 produce agents which are less exploitable than the supervised learning agent; for the best exploiters found, the picture is less clear. This may be because we do not have as good a value function for games with $\pi^{SL}$ as we do for the other policies.
 
\begin{figure}[!ht]
    \centering
    \includegraphics[scale=.5]{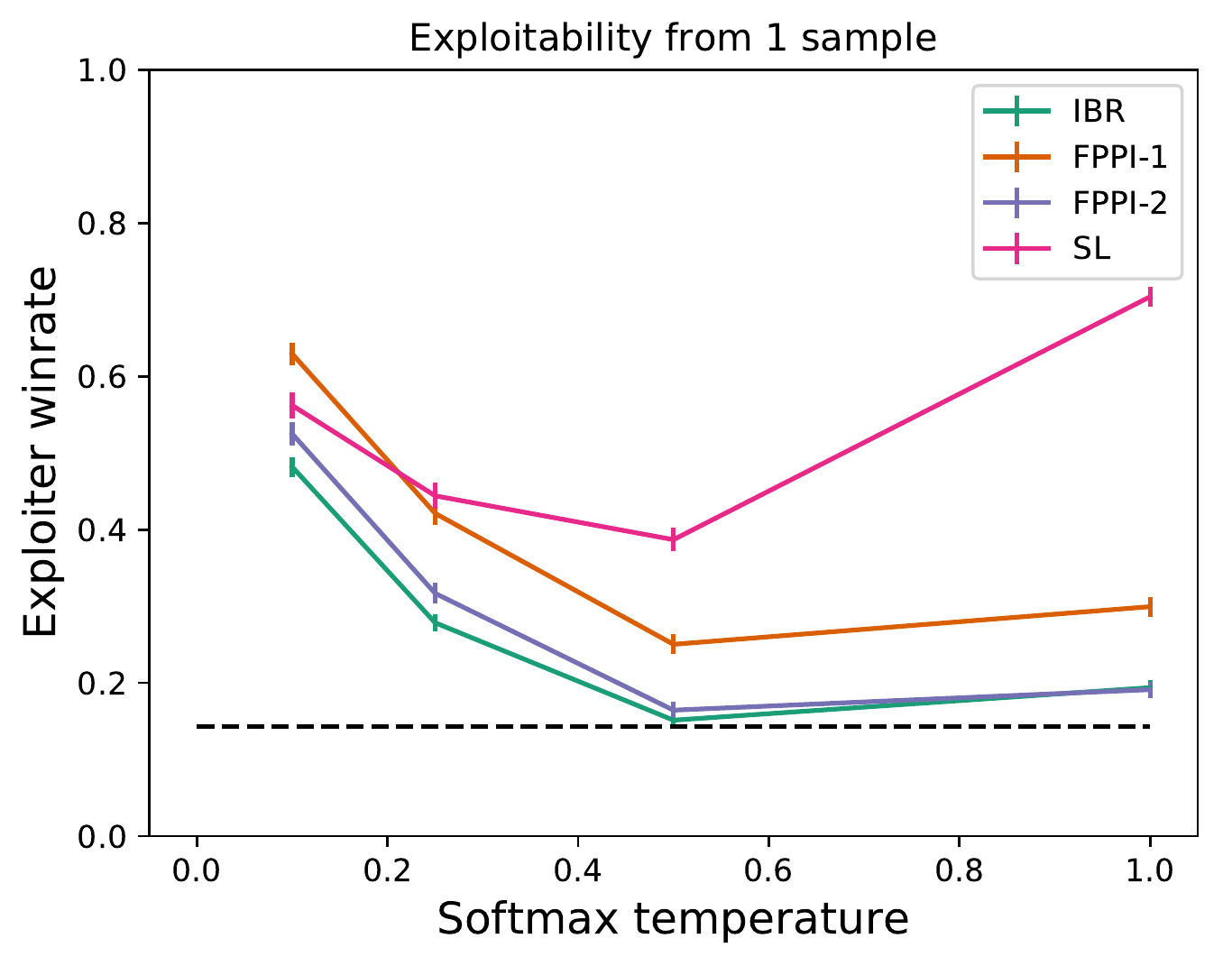}
    \includegraphics[scale=.5]{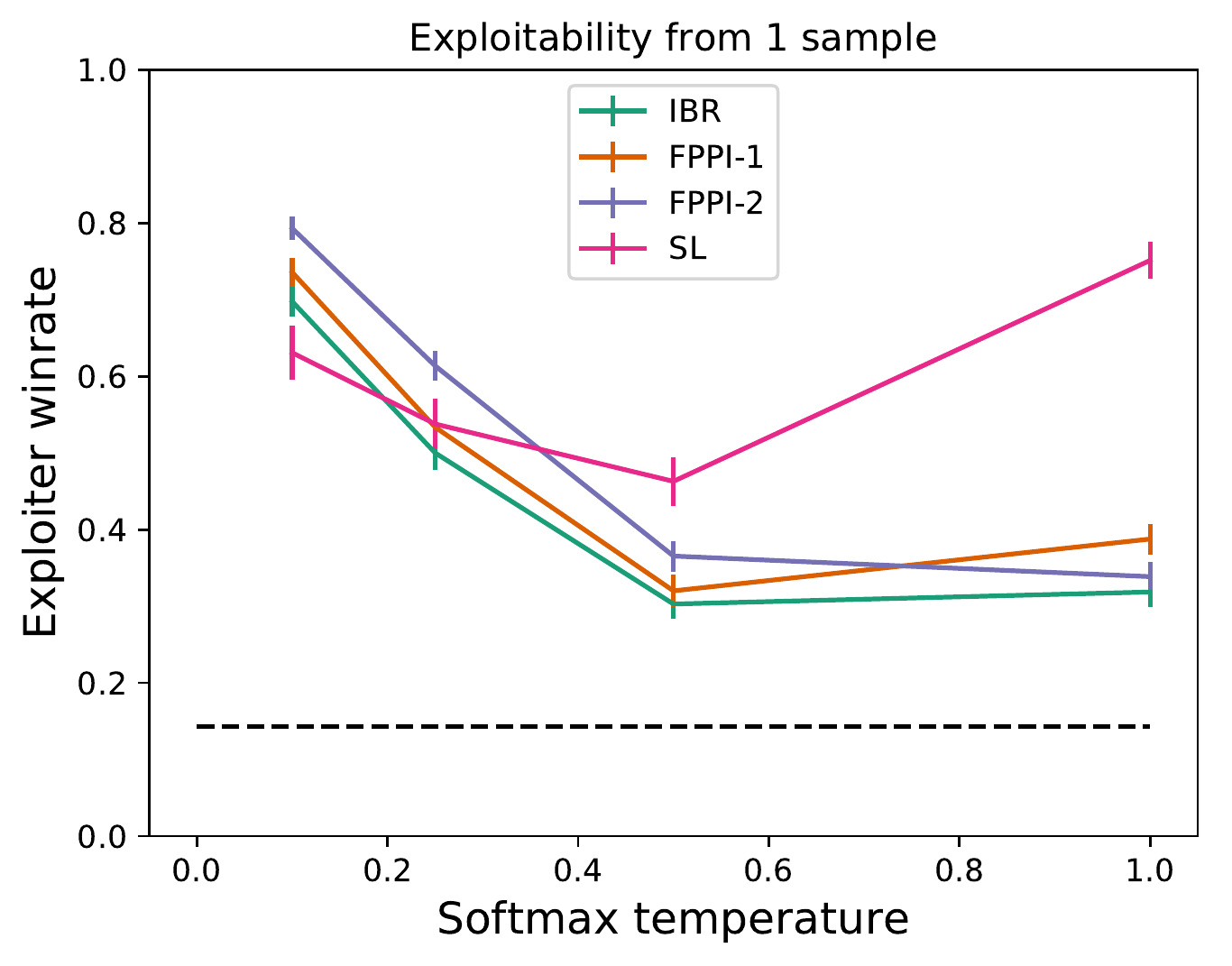}
    \includegraphics[scale=.5]{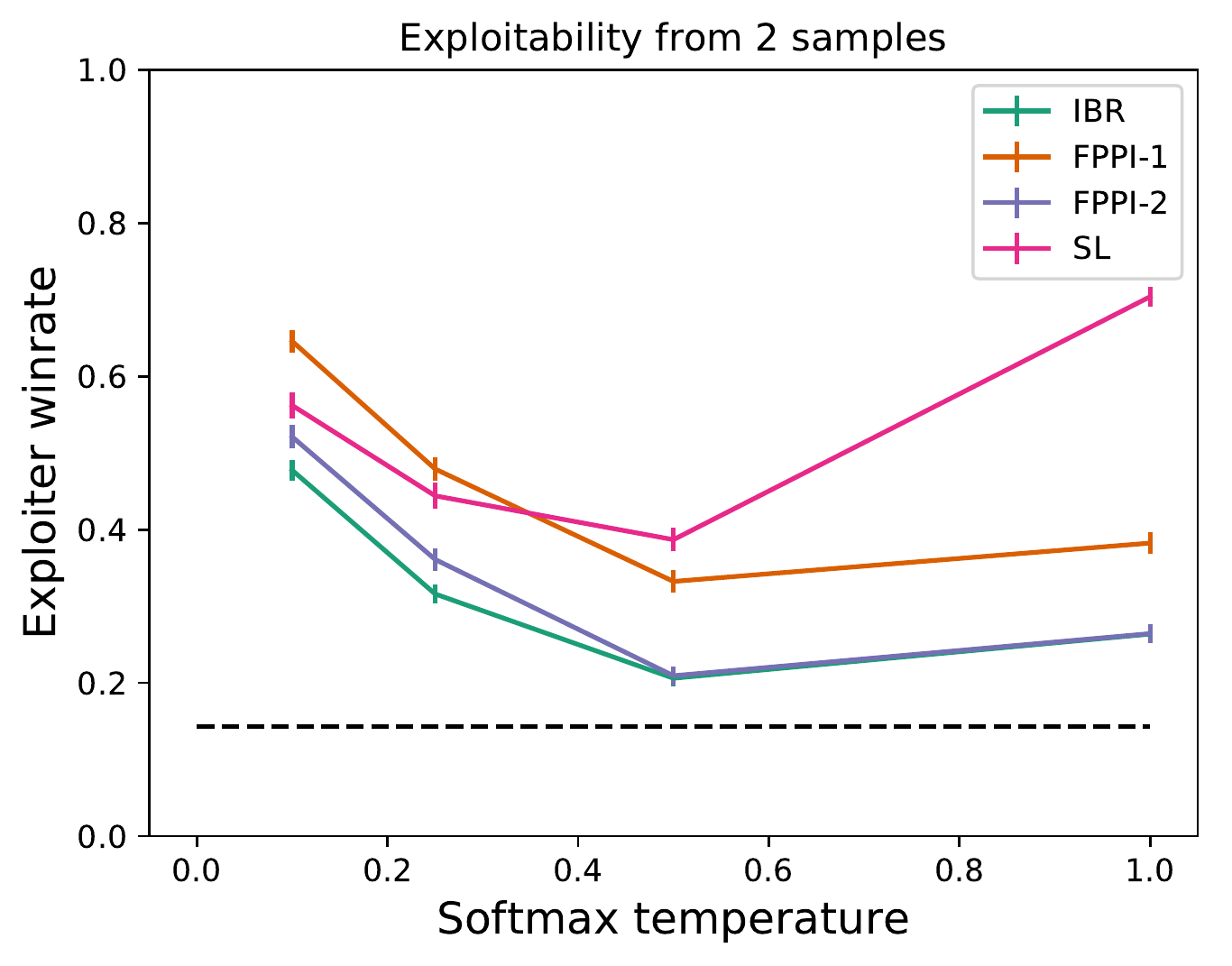}
    \includegraphics[scale=.5]{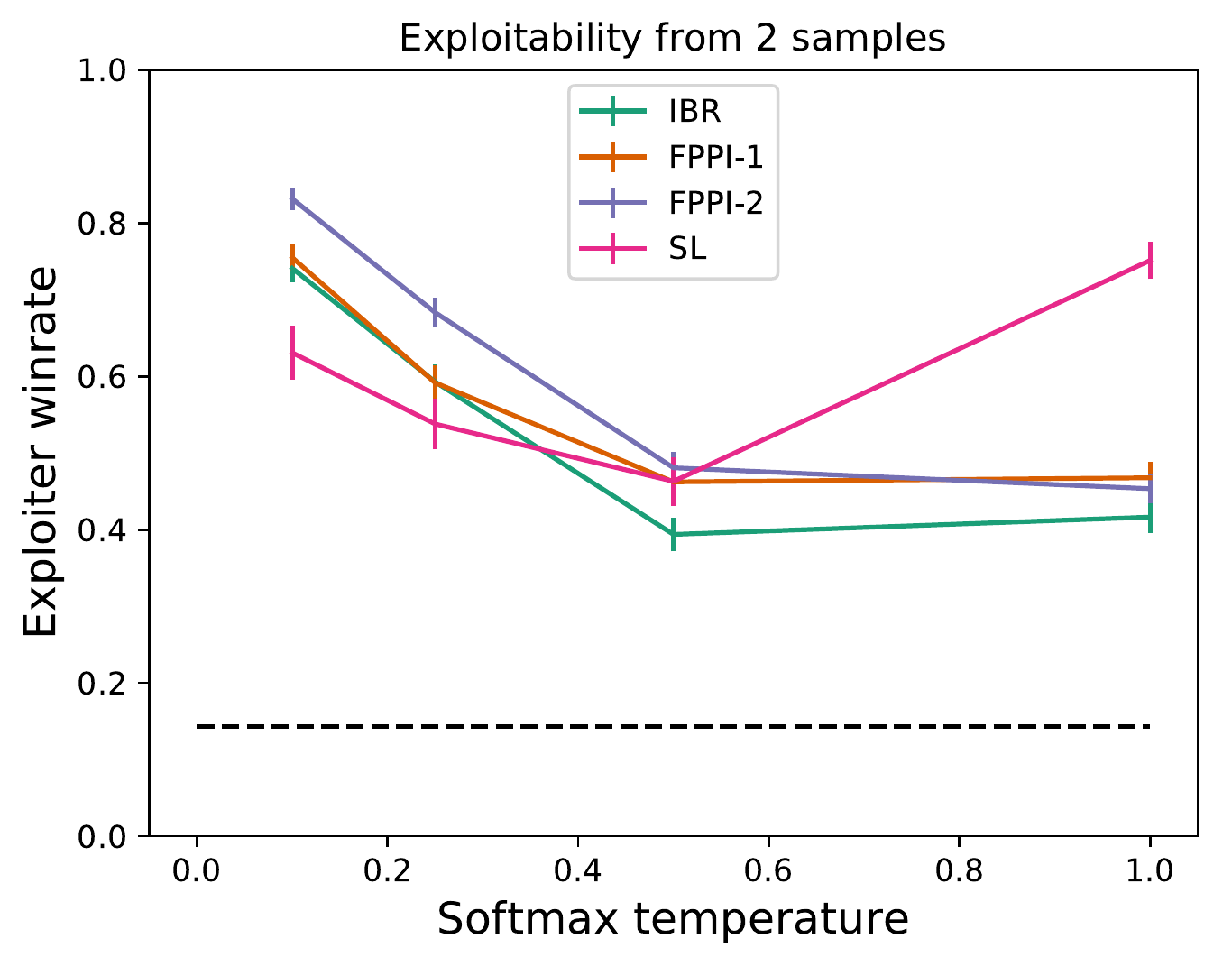}
    \includegraphics[scale=.5]{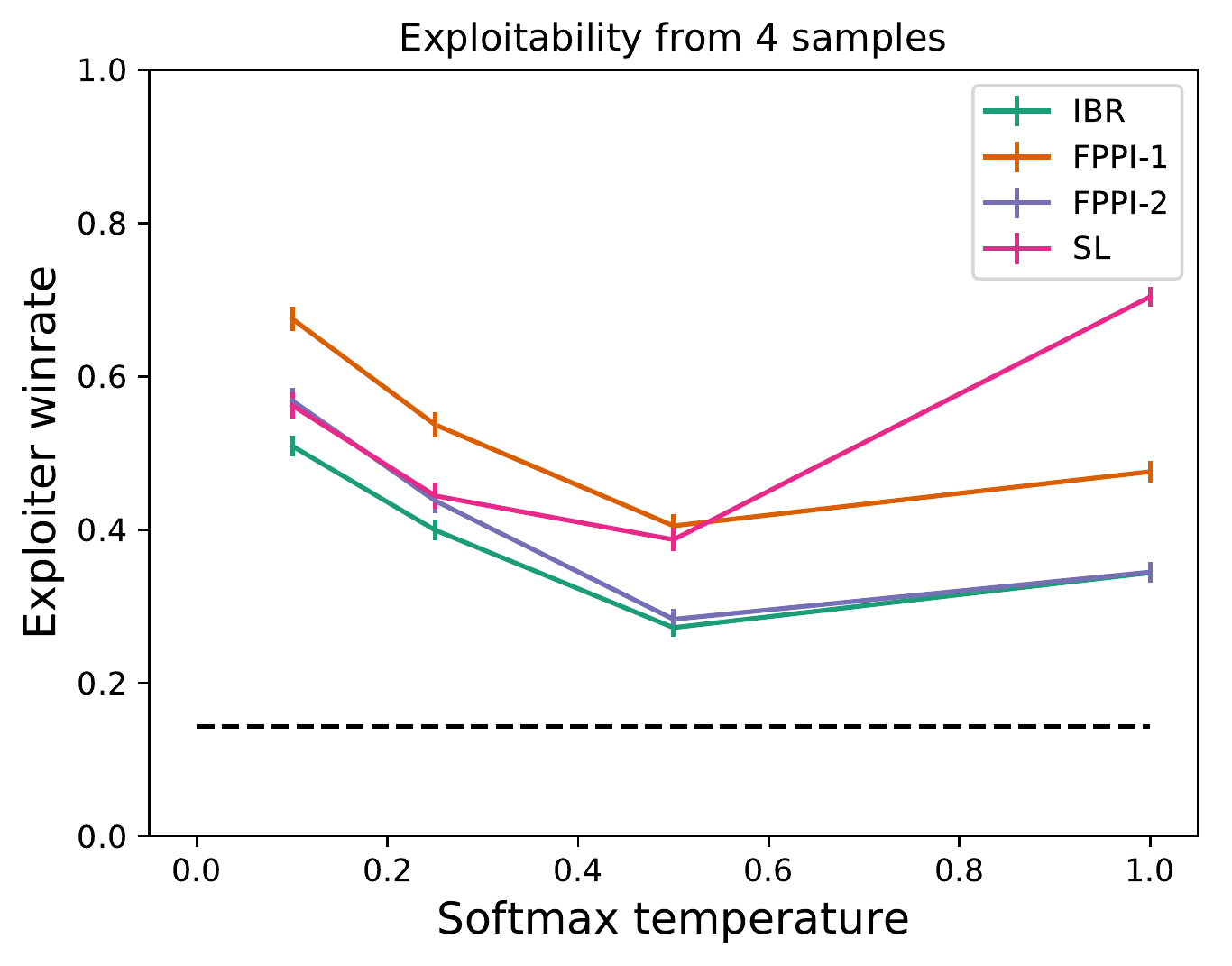}
    \includegraphics[scale=.5]{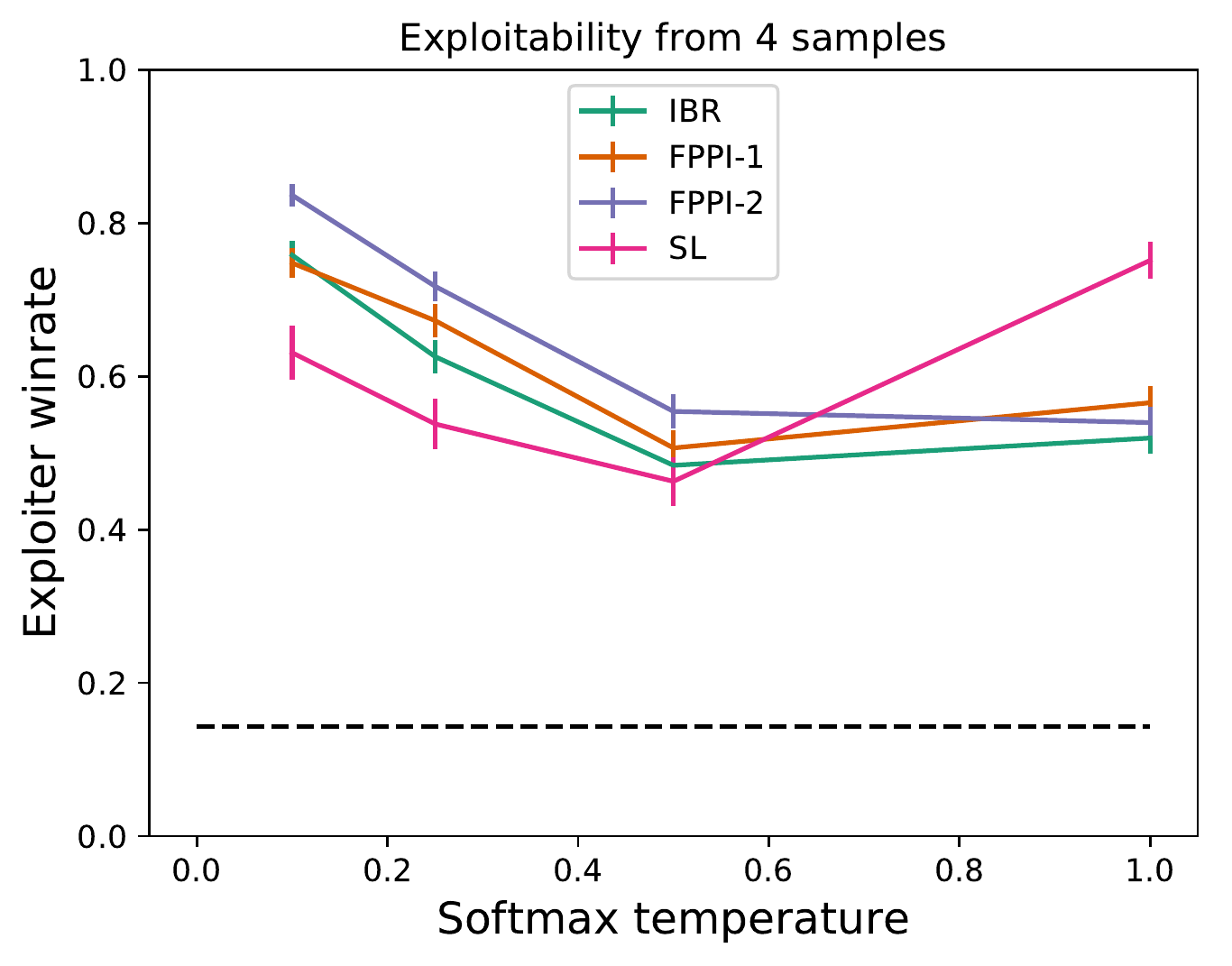}
    \caption{Exploiter winrates of imitation and final BRPI networks. Left column shows the exploits achieved by few shot exploiters, right column the best exploiters found.}
    \label{fig:exploitability}
\end{figure}

Tables \ref{tab:few-target-exploits} and \ref{tab:best-target-exploits} shows lower bounds on the exploitability of the final training targets from the three RL runs, again with few-shot and best exploiters. This gives a lower bound for exploitability which is less than for the networks these targets are improving on. This is particularly interesting for the training target for IBR -- which consists of a single iteration of Sampled Best Response. This is in contrast with what we would see with an exact best response, which would be a highly exploitable deterministic policy.

\begin{table}[ht]
\centering
 \begin{tabular}{| c | c | c | c | c | c | c| } 
 \toprule
  & 1 profile & 2 profiles & 4 profiles \\
\midrule
IBR & $0.008 \pm 0.009$ & $0.063 \pm 0.011$ & $0.129 \pm 0.012$ \\
FPPI-1 & $0.107 \pm 0.013$ & $0.189 \pm 0.015$ & $0.262 \pm 0.016$ \\
FPPI-2 & $0.021 \pm 0.011$ & $0.066 \pm 0.012$ & $0.140 \pm 0.014$ \\
Target(IBR) & $-0.027 \pm 0.008$ & $0.002 \pm 0.009$ & $0.063 \pm 0.012$ \\
Target(FPPI-1) & $-0.012 \pm 0.015$ & $0.066 \pm 0.019$ & $0.156 \pm 0.024$ \\
Target(FPPI-2) & $-0.032 \pm 0.014$ & $0.024 \pm 0.018$ & $0.127 \pm 0.024$ \\
\bottomrule
\end{tabular}
\caption{Few-shot exploitability of final networks training targets with different numbers of base profiles. Networks are shown at the temperature with the highest lower bound on exploitability.}
\label{tab:few-target-exploits}
\end{table}

\begin{table}[ht]
\centering
 \begin{tabular}{| c | c | c | c | c | c | c| } 
 \toprule
  & 1 profile & 2 profiles & 4 profiles \\
\midrule
IBR & $0.160 \pm 0.019$ & $0.251 \pm 0.022$ & $0.341 \pm 0.022$ \\
FPPI-1 & $0.177 \pm 0.021$ & $0.319 \pm 0.023$ & $0.364 \pm 0.023$ \\
FPPI-2 & $0.223 \pm 0.020$ & $0.338 \pm 0.021$ & $0.411 \pm 0.023$ \\
Target(IBR) & $0.049 \pm 0.011$ & $0.109 \pm 0.013$ & $0.187 \pm 0.015$ \\
Target(FPPI-1) & $0.149 \pm 0.019$ & $0.257 \pm 0.022$ & $0.373 \pm 0.024$ \\
Target(FPPI-2) & $0.051 \pm 0.016$ & $0.157 \pm 0.019$ & $0.271 \pm 0.022$ \\
\bottomrule
\end{tabular}
\caption{Best lower bound on exploitability of final networks and training targets with different numbers of base profiles. Networks are shown at the temperature with the highest lower bound on exploitability.}
\label{tab:best-target-exploits}
\end{table}
All the exploiting agents here use $64$ candidates for SBR at each temperature in $(0.1, 0.25, 0.5, 1.0)$.

\subsection{Head to head comparisons}
Here, we add to Table \ref{tab_1v6_agent_comparisons} additional comparisons where we run a test-time improvement step on our initial and final networks. These steps use a more expensive version of SBR than we use in training; we sample $64$ candidates from the network at each of four temperatures $(0.1, 0.25, 0.5, 1.0)$, and also sample the base profiles at temperature $0.1$. We only compare these training targets to the networks and training targets for other runs -- comparing a training target to the network from the same run would be similar to the exploits trained in Appendix~\ref{app:exploitability}.

For the final networks, these improvement steps perform well against the final networks (from their own algorithm and other algorithms). For the initial network, the resulting policy loses to all agents except the imitation network it is improving on; we hypothesise that this is a result of the value function, which is trained on the human dataset and so may be inaccurate for network games.

We also add the agent produced by our implementation of A2C, as described in \ref{app:a2c}.

\begin{table}[ht]
\small
\centering
\begin{tabular}{l | p{7mm}  p{7mm}  p{7mm}  p{7mm} | p{9mm} p{7mm}  p{9mm} | p{7mm}  p{8mm} p{9mm}  } 
\toprule
{} & SL \cite{paquette2019no} & A2C \cite{paquette2019no} & SL (ours) & A2C (ours) & FPPI-1 net & IBR net & FPPI-2 net & SBR-SL & SBR-IBR & SBR-FPPI-2 \\
\midrule
SL \cite{paquette2019no} & 14.2\% & 8.3\% & 16.3\% & 7.7\% & 2.3\% & 1.8\% & 0.8\% & 38.6\% & 2.1\% & 2.3\% \\
A2C \cite{paquette2019no} & 15.1\% & 14.2\% & 15.3\% & 17.0\% & 2.3\% & 1.7\% & 0.9\% & 54.9\% & 2.3\% & 2.4\% \\
SL (ours) & 12.6\% & 7.7\% & 14.1\% & 10.6\% & 3.0\% & 1.9\% & 1.1\% & 28.1\% & 1.6\% & 1.4\% \\
A2C & 14.1\% & 3.5\% & 18.7\% & 14.1\% & 2.8\% & 2.3\% & 1.2\% & 36.6\% & 2.2\% & 1.7\% \\
\midrule
FPPI-1 net & 26.4\% & 28.0\% & 25.9\% & 30.9\% & 14.4\% & 7.4\% & 4.5\% & 67.0\% & 4.4\% & 3.0\% \\
IBR net & 20.7\% & 30.5\% & 25.8\% & 29.4\% & 20.3\% & 12.9\% & 10.9\% & 70.2\% & 5.5\% & 8.0\% \\
FPPI-2 net & 19.4\% & 32.5\% & 20.8\% & 32.1\% & 22.4\% & 13.8\% & 12.7\% & 73.6\% & 6.6\% & 8.1\% \\
\midrule
SBR-SL & 12.5\% & 1.8\% & 15.8\% & 9.2\% & 0.9\% & 0.5\% & 0.1\% & 14.1\% & 0.2\% & 0.2\% \\
SBR-IBR & 24.5\% & 23.3\% & 28.4\% & 30.3\% & 37.7\% & 37.8\% & 31.4\% & 57.7\% & 14.4\% & 16.0\% \\
SBR-FP-2 & 20.0\% & 23.1\% & 32.2\% & 29.0\% & 44.3\% & 28.2\% & 35.0\% & 62.2\% & 14.9\% & 14.9\% \\
\bottomrule
\end{tabular}
\caption{Matches between different algorithms. Winrates for 1 row player vs 6 column players}
\label{tab_1v6_agent_comparisons_too_big}
\end{table}

In Table \ref{tab_1v6_brpi_comparisons}, we give comparisons for the final networks of agents trained using the design space of BRPI specified in \ref{methods:brpi}. The notation for candidate policies and base profiles is $\pi_0$ for the imitation network we start training from, $\pi_{t-1}$ is the policy from the previous iteration, and $\mu_{t-1}$ is the average of the policies from previous iterations. $V_{t-1}$ is the value function from the previous iteration, and $V^{\mu}_{t-1}$ is the average value function from the previous iterations. When $\mu_{t-1}$ and $V^{\mu}_{t-1}$ are used, the sampling is coupled; the value function is from the same network used for the candidates and/or base profiles.

The fourth column of Table \ref{tab_1v6_brpi_comparisons} records which kind of BRPI method each is. Methods that use the latest policy only for base profiles are IBR methods, uniformly sampled base profiles are FP methods. FP methods that use the latest networks for the value function or for candidate sampling do not recreate the historical best responses, so are FPPI-2 methods. The remaining methods are FPPI-1 methods. The asterisks mark the examples of IBR, FPPI-1 and FPPI-2 selected for deeper analysis in section \ref{l_sect_experiments}; these were chosen based on results of experiments during development that indicated that including candidates from the imitation policy was helpful. In particular, they were not selected based on the outcome of the training runs presented in this work.

We find that against the population of final networks from all BRPI runs ($1$ vs $6$ BRPI), IBR and FPPI-2 do better than FPPI-1. For candidate selection, using candidates from the latest and initial networks performs best. For beating the DipNet baseline, we find that using candidates from the imitation policy improves performance substantially. This may be because our policies are regularised towards the style of play of these agents, and so remain more able to play in populations consisting of these agents.

\begin{table}[ht]
\centering
 \begin{tabular}{c | c | c | p{12mm} || p{15mm} | p{15mm} | p{15mm} | p{15mm}  } 
 \toprule
$\pi_c$ & $\pi_b$ & $V$ & BRPI type & 1 vs \newline6 BRPI & 1 BRPI \newline vs 6 & 1 vs\newline 6 DipNet & 1 DipNet \newline vs 6 \\
 \midrule
 \midrule
$\pi_{0}$ & $\pi_{t-1}$ & $V_{t-1}$ & IBR & 9.6\% & 16.1\% & 24.7\% & 4.0\% \\
$\pi_{0}$ & $\mu_{t-1}$ & $V_{t-1}$ & FPPI-2 & 9.8\% & 17.7\% & 24.4\% & 3.7\% \\
$\pi_{0}$ & $\mu_{t-1}$ & $V^\mu_{t-1}$ & FPPI-1 & 8.5\% & 17.1\% & 27.3\% & 3.3\% \\
$\mu_{t-1}$ & $\pi_{t-1}$ & $V_{t-1}$ & IBR & 15.4\% & 10.9\% & 25.9\% & 1.3\% \\
$\mu_{t-1}$ & $\mu_{t-1}$ & $V_{t-1}$ & FPPI-2 & 14.0\% & 12.6\% & 22.3\% & 1.7\% \\
$\mu_{t-1}$ & $\mu_{t-1}$ & $V^\mu_{t-1}$ & FPPI-1 & 9.3\% & 15.6\% & 24.4\% & 2.3\% \\
$\mu_{t-1} + \pi_0$ & $\pi_{t-1}$ & $V_{t-1}$ & IBR & 14.3\% & 12.8\% & 25.5\% & 2.0\% \\
$\mu_{t-1} + \pi_0$ & $\mu_{t-1}$ & $V_{t-1}$ & FPPI-2 & 13.8\% & 12.6\% & 26.1\% & 1.7\% \\
$\mu_{t-1} + \pi_0$ & $\mu_{t-1}$ & $V^\mu_{t-1}$ & FPPI-1* & 9.3\% & 17.2\% & 26.4\% & 2.3\% \\
$\pi_{t-1}$ & $\pi_{t-1}$ & $V_{t-1}$ & IBR & 16.0\% & 7.5\% & 14.3\% & 0.5\% \\
$\pi_{t-1}$ & $\mu_{t-1}$ & $V_{t-1}$ & FPPI-2 & 17.4\% & 6.9\% & 13.2\% & 0.7\% \\
$\pi_{t-1}$ & $\mu_{t-1}$ & $V^\mu_{t-1}$ & FPPI-2 & 9.7\% & 15.5\% & 13.6\% & 3.5\% \\
$\pi_{t-1} + \pi_0$ & $\pi_{t-1}$ & $V_{t-1}$ & IBR* & 16.4\% & 9.8\% & 20.7\% & 1.8\% \\
$\pi_{t-1} + \pi_0$ & $\mu_{t-1}$ & $V_{t-1}$ & FPPI-2* & 17.6\% & 8.2\% & 19.4\% & 0.8\% \\
$\pi_{t-1} + \pi_0$ & $\mu_{t-1}$ & $V^\mu_{t-1}$ & FPPI-2 & 12.5\% & 13.2\% & 23.9\% & 1.8\% \\
\midrule
$\pi_{0}$ & mean & mean & - & 9.3\% & 17.0\% & 25.5\% & 3.7\% \\
$\mu_{t-1}$ & mean & mean & - & 12.9\% & 13.0\% & 24.2\% & 1.8\% \\
$\mu_{t-1} + \pi_0$ & mean & mean & - & 12.5\% & 14.2\% & 26.0\% & 2.0\% \\
$\pi_{t-1}$ & mean & mean & - & 14.4\% & 10.0\% & 13.7\% & 1.6\% \\
$\pi_{t-1} + \pi_0$ & mean & mean & - & 15.5\% & 10.4\% & 21.4\% & 1.5\% \\
\midrule
mean & $\pi_{t-1}$ & $V_{t-1}$ & - & 14.3\% & 11.4\% & 22.2\% & 1.9\% \\
mean & $\mu_{t-1}$ & $V^\mu_{t-1}$ & - & 9.9\% & 15.7\% & 23.1\% & 2.6\% \\
mean & $\mu_{t-1}$ & $V_{t-1}$ & - & 14.5\% & 11.6\% & 21.1\% & 1.7\% \\
\bottomrule
\end{tabular}
\caption{Performance of different BRPI variants against BRPI and DipNet agents. The scores are all for the $1$ agent. All results are accurate to $0.5\%$ within a confidence interval of $95\%$}
\label{tab_1v6_brpi_comparisons}
\end{table}

\section{Imitation Learning and Neural Network Architecture}
\label{appendix:network}

We fully describe the architecture of the neural network we use for approximating policy and value functions,  including hyperparameter settings. The architecture is illustrated in Figure~\ref{fig:arch}.

\begin{figure}[ht]
    \centering
    \includegraphics[width=\linewidth]{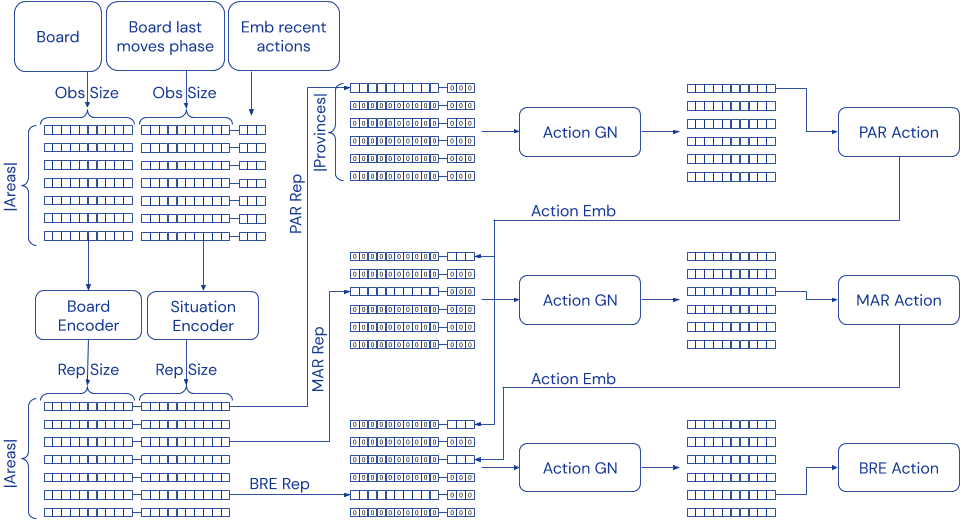}
    \caption{The neural network architecture for producing actions.}
    \label{fig:arch}
\end{figure}

Our network outputs policy logits for each unit on the board controlled by the current player $p$, as well as a value estimate for $p$. It takes as inputs:
\begin{itemize}
    \item $x_b$, a representation of the current state of the board, encoded with the same $35$ board features per area~\footnote{An \textit{area} is any location that can contain a unit. This is the 75 provinces (e.g. Portugal, Spain), plus the 6 coastal areas of the three bicoastal provinces (e.g. Spain (south coast), Spain (north coast)).} used in DipNet
    \item $x_m$, the state of the board during the last moves phase, encoded the same way
    \item $x_o$, the orders issued since that phase 
    \item $s$, the current {\it Season}
    \item $p$, the current power
    \item $x_d$, the current build numbers (i.e. the difference between the number of supply centres and units for each power)
\end{itemize}

Similar to DipNet, we first produce a representation of previous gameplay. This incorporates learned embeddings $e_o(x_o)$, $e_s(s)$, and $e_p(p)$ applied to the recent orders, the season, and the power. The recent orders embeddings are summed in each area, producing $\tilde e_o(x_o)$. We begin by concatenating $x_m$ and $\tilde e_o(x_o)$ to produce $\tilde x_m$. Then, we concatenate $x_d$ and $e_s(s)$ to each of $x_b$ and $\tilde x_m$ to produce $\bar{x}_b = [x_b, x_d, e_s(s)]$ and $\bar{x}_m = [x_m, e_o(x_o), x_d, e_s(s)]$. (DipNet uses hardcoded ``alliance features'' in place of $\tilde x_m$, and leaves out $x_d$.)

Next, we process each of $\bar{x}_b$ and $\bar{x}_m$ with identical, but independent stacks of 12 Graph Neural Networks (GNNs) \cite{battaglia2018relational} linked (apart from the first layer) by residual connections. In particular, each residual GNN computes $\bar{x}^{\ell+1} = \bar{x}^{\ell} + \text{ReLU}(\text{BatchNorm}([\hat{\bar{x}}^\ell, A\cdot\hat{\bar{x}}^\ell]))$, where $\hat{\bar{x}}^\ell_{n,j} = \sum_i \bar{x}^\ell_{n, i} w^\ell_{n, i, j}$, with $\ell$ indexing layers in the stack, $n$ the areas, $A$ the normalized adjacency matrix of the Diplomacy board, and $\bar{x}$ being a stand-in for either $\bar{x}_b$ or $\bar{x}_m$. After concatenating $e_p(p)$ to the resulting embeddings, we use 3 additional GNN layers with residual connections to construct $x^{s,p}_b$ and $x^{s,p}_m$, which we concatenate to form our final state encoding $x^{s,p} = [x^{s,p}_b, x^{s,p}_m]$. Note that, although DipNet describes their encoders as Graph Convolutional Networks, their GNNs are not convolutional over nodes, and weights are not tied between GNNs -- this produced better results, and we followed suit.

The value outputs are computed from the state encodings $x^{s,p}$ by averaging them across $p$ and $s$, and then applying a ReLU MLP with a single hidden layer to get value logits.

Like DipNet, we construct order lists from our state encoding $x^{s,p}$ considering one by one the provinces $n(1), \dots, n(k)$ requiring an order from $p$ according to a fixed order. Unlike DipNet, we use a Relational Order Decoder (ROD). Our ROD module is a stack of 4 GNNs with residual connections that, when considering the $k$-th province, takes as input the concatenation of $x^{s,p}_{n(k)}$ and $z_{n(1), \ldots, n(k-1)}$. Precisely, $x^{s,p}_{n(k)}$ is a masked version of $x^{s,p}$ where all province representations except $n(k)$ are zeroed out, and $z_{n(1), \ldots, n(k-1)}$ contains embeddings of the orders already on the list scattered to the slots corresponding to the provinces they referred to: $n(1),\ldots,n(k-1)$. The output of the ROD corresponding to province $n(k)$ is then mapped to action logits through a linear layer with no bias $w$. Similarly to DipNet, after sampling, the row of $w$ corresponding to the order selected is used to fill in the $n(k)$-th row of $z_{n(1), \ldots, n(k-1),n(k)}$.

Table~\ref{tab:imitation} compares the imitation accuracy and winrates improvements when switching from the DipNet neural architecture to the one we use (indicating a slight improvement in performance). For our RL experiments we chose the architecture with the highest imitation accuracy despite its decline in winrates; this is because the RL improvement loop relies on imitation, so imitation performance is the chief desideratum.

Furthermore, the winrates are affected by a confounding factor that we uncovered while inspecting the imperfect winrate of the final imitation network against a random policy. What we found was that in games where the network didn't beat the random policy, the network was playing many consecutive all-holds turns, and hitting a 50-year game length limit in our environment. This reflects the dataset: human players sometimes abandon their games, which shows up as playing all-holds turns for the rest of the game. We hypothesize that the encoder that observes the preceding moves-phase board, and especially the actions since then, is better able to represent and reproduce this behaviour. This is to its detriment when playing games, but is easily addressed by the improvement operator.

\begin{table}[!ht]
    \centering
\begin{tabular}{p{\widthof{$+$Relational decoder}}rrrrrrrr}
\toprule
& \multicolumn{4}{c}{Imitation accuracy (\%)} &  \multicolumn{4}{c}{Winrates (\%)} \\
\cmidrule(lr){2-5}\cmidrule(lr){6-9}
&\multicolumn{2}{c}{Teacher forcing}&\multicolumn{2}{c}{Whole-turn}&\multicolumn{2}{c}{vs Random}&\multicolumn{2}{c}{vs DipNet SL}\\
\cmidrule(lr){2-3}\cmidrule(lr){4-5}\cmidrule(lr){6-7}\cmidrule(lr){8-9}
Architecture&Press&No-press&Press&No-press&1v6&6v1&1v6&6v1\\
\midrule
DipNet replication & 56.35 & 58.03 & 25.67 & 26.86 & 100 & 16.67 & 15.99 & 14.42 \\
$+$Encoder changes  & 59.08 & 60.32 & 30.79 & 30.28 & 100 & 16.66 & 16.75 & 14.51 \\
$+$Relational decoder & 60.08 & 61.97 & 30.26 & 30.73 & 100 & 16.67 & 17.71 & 14.33 \\
$-$Alliance features & 60.68 & 62.46 & 30.96 & 31.36 & 99.16 & 16.66 & 13.30 & 14.25 \\
\bottomrule
\end{tabular}
    \caption{Imitation learning improvements resulting from our changes to DipNet~\cite{paquette2019no}.}
    \label{tab:imitation}
\end{table}

\subsection{Hyperparameters}
\label{app:im-hypers}

For the inputs, we use embedding sizes $10$, $16$, and $16$ for the recent orders $x_o$, the power $p$, and the season $s$, in addition to the same $35$ board features per area used in DipNet. The value decoder has a single hidden layer of size $256$. The relational order decoder uses an embedding of size $256$ for each possible action.

The imitation networks were trained on a P100 GPU, using an Adam optimizer with learning rate $0.003$ and batch size $256$, on a 50-50 mixture of the No-Press and Press datasets used for DipNet. The datasets were shuffled once on-disk, and were sampled during training via a buffer of $10^4$ timesteps into which whole games were loaded. During training we filtered out powers that don't end with at least $7$ SCs as policy targets, and filtered out games where no power attains $7$ SCs as value targets. $2000$ randomly selected games from each dataset were held out as a validation set.

Power and season embeddings were initialized with random uniform entries in $[-1,1]$. Previous-order embeddings were initialized with random standard normal entries. The decoder's action embeddings were initialized with truncated normal entries with mean 0 and standard deviation $1/\sqrt{18584}$ ($18584$ is the number of possible actions). The GNN weights were initialized with truncated normal entries with mean 0 and standard deviation $2/\sqrt{(\text{input width})\cdot(\text{number of areas})}$. The value network was initialized with truncated standard normal entries with mean 0 and standard deviation $1/\sqrt{\text{input width}}$ for both its hidden layer and its output layer.

\subsection{Data cleaning}
\label{app:data-cleaning}
We use the dataset introduced in \cite{paquette2019no} for our pre-training on human imitation data. Before doing so, we run the following processing steps, to deal with the fact that the data is from a variety of sources and contains some errors:
\begin{itemize}
    \item We exclude any games labelled in the dataset as being on non-standard maps, with fewer than the full $7$ players, or with non-standard rules.
    \item We attempt to apply all actions in the game to our environment. If we do so successfully, we include the game in the dataset. We use observations generated by our environment, not the observations from the dataset, because some of the dataset's observations were inconsistent with the recorded actions.
    \item As final rewards for training our value function, we take the reward from the dataset. For drawn games, we award $1/n$ to each of the $n$ surviving players.
    \item We deal with the following variations in the rules (because not every game was played under the same ruleset):
   \begin{itemize}
     \item Some data sources apparently played forced moves automatically. So if there is only one legal move, and no move for the turn in the dataset, we play the legal move.
     \item Some variants infer whether a move is a convoy or land movement from the other orders of the turn, rather than making this explicit. To parse these correctly, we retry failed games, changing move orders to convoys if the country ordering the move also ordered a fleet to perform a convoy of the same route. For example, if the orders are A BEL - HOL, F HOL - BEL and F NTH C BEL - HOL, we update the first order to be A BEL - HOL (via convoy). If this rewritten game can be parsed correctly, we use this for our imitation data.
   \end{itemize}
   \item Any game which can not be parsed with these attempted corrections is excluded from the dataset. One large class of games that cannot be parsed is those played under a ruleset where the coast of supports matters -- that is, a move to Spain (sc) can fail because a support is ordered to Spain (nc) instead. Other than these, the errors were varied; the first $20$ games manually checked appear to have units ending up in places which are inconsistent with their orders, or similar errors. It is possible that these involved manual changes to the game state external to the adjudicator, or were run with adjudicators which either had bugs or used non-standard rulesets.
\end{itemize}

This process resulted in number of exclusions and final data-sets sizes reported in Tab.~\ref{tab:d-set-stats}.
\begin{table}[!ht]
    \centering
\begin{tabular}{l| r r r r}
\toprule
{} & \multicolumn{2}{c}{Training set} & \multicolumn{2}{c}{Validation set}\\
\midrule
{} & \multicolumn{1}{c}{Press} & \multicolumn{1}{c}{No-press} & \multicolumn{1}{c}{Press} & \multicolumn{1}{c}{No-press}\\
\midrule
\midrule
Available games & 104456 & 31279 & 2000 & 2000\\
\midrule
Excluded: non-standard map & 833 & 10659 & 17 & 705\\
Excluded: non-standard rules & 863 & 1 & 19 & 0\\
Excluded: min SCs not met & 2954 & 517 & 48 & 31\\
Excluded: unable to parse & 3554 & 79 & 78 & 4\\
\midrule
Included & 96252 & 20023 & 1838 & 1260\\
\bottomrule
\end{tabular}
\caption{Number of games available, included and excluded in our data-sets.}\label{tab:d-set-stats}
\end{table}

Diplomacy has been published with multiple rulebooks, which slightly differ in some ways. Some rulebooks do not completely specify the rules, or introduce paradoxes. The Diplomacy Adjudicator Test Cases describe a variety of rulesets, interpretations and paradox resolution methods~\cite{datc}. We make the same interpretations as the webDiplomacy website~\cite{webdip_datc}. Our adjudicator uses Kruijswijk's algorithm~\cite{kruijswijk_adjudication}.

\section{BRPI settings}
\label{app:rl-settings}
For all the reinforcement learning runs reported, we use the settings as in section \ref{app:im-hypers} for the learning settings, with the exception of the learning rate for Adam, which is $10^{-4}$. We update the policy iteration target by adding a new checkpoint every $7.5$ million steps of experience. We use an experience replay buffer of size $50000$; to save on the amount of experience needed, each datapoint is sampled $4$ times before it is evicted from the buffer.

For sampled best response, we use $2$ base profiles. We sample $16$ candidate moves for each player; in the case where we use two sources of candidate (such as $\pi^\textrm{SL}$ and latest checkpoint), we sample $8$ from each. The base profiles are the same for each of the $7$ powers, and if the policy being responded to is the same as a policy producing candidates, we reuse the base profiles as candidate moves.

We force draws in our games after random lengths of time. The minimum game length is $2$ years; after that we force draws each year with probability $0.05$. Since our agents do not know how to agree draws when the game is stalemated, this gives the game length a human-like distribution and incentivises agents to survive even when they are unlikely to win. When games are drawn in this way, we award final rewards dependent on supply centres, giving each power a reward of the proportion of all controlled SCs which they control at the end of the game. Otherwise, no rewards are given except for a reward of $1$ for winning.

\section{A2C with V-Trace off policy correction}
\label{app:a2c}
\begin{figure}[ht]
    \centering
    \includegraphics[width=0.4\linewidth]{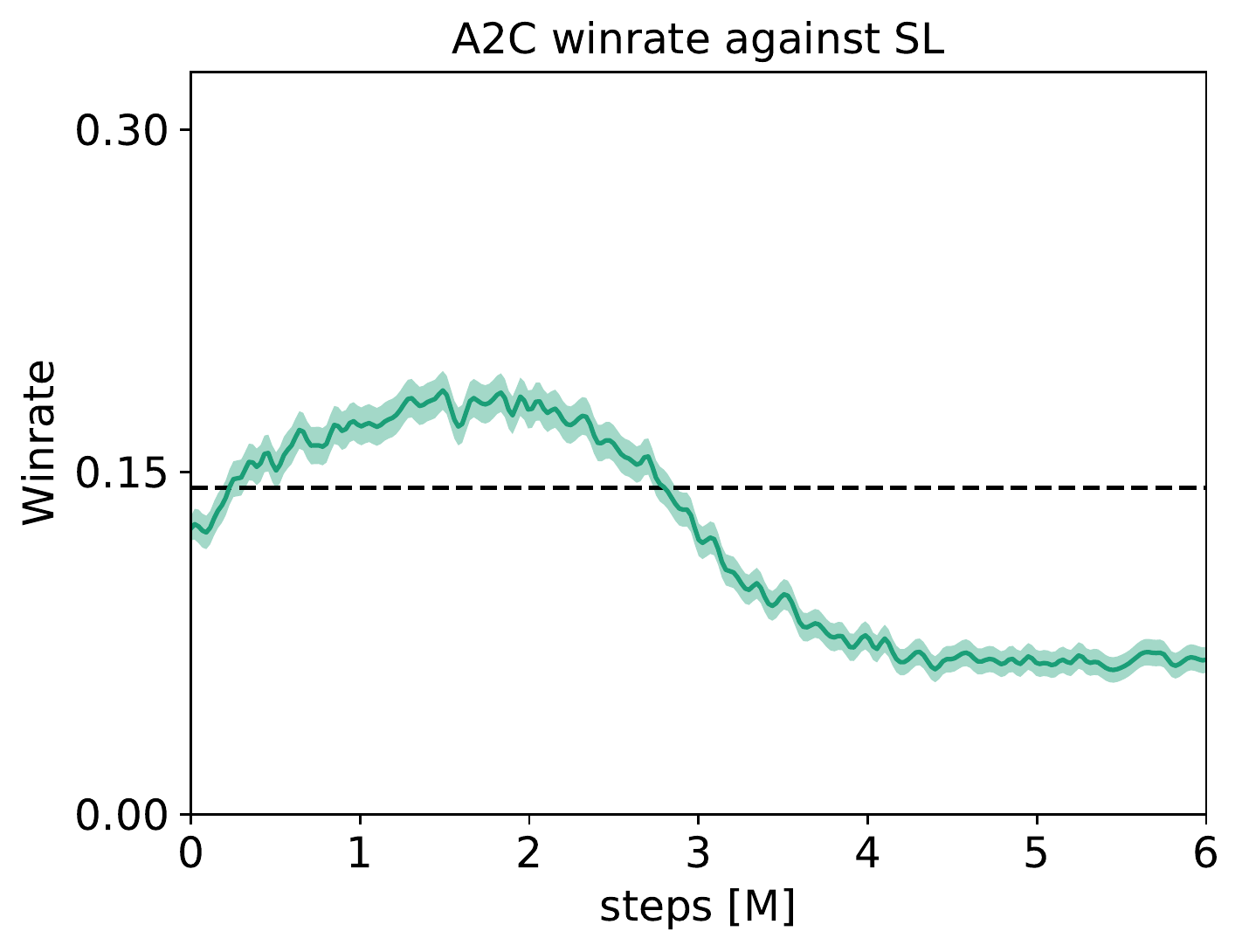}
    \caption{Winrate of 1 A2C v. 6 Imitation Learning (SL). Shaded areas are error-bars over 7000 games, uniformly distributed over the power controlled by A2C. The step counter on the horizontal axis shows the number of state, action, rewards triplets used by the central learner. Note that this is different than other RL plots in the text, where policy iteration loops are shown.}
    \label{fig:a2c}
\end{figure}
In this section we describe our implementation of the batched advantage actor-critic (A2C) with V-Trace off policy correction algorithm we used as our policy gradient baseline \cite{espeholt2018impala}, and very briefly comment on its performance. As in our BRPI experiment, we let Reinforcement Learning training start from our Imitation Learning baseline. We use the same network architecture as for BRPI.

Our implementation of A2C uses a standard actor-learner architecture where actors periodically receive network parameters from a central learner, and produce experience using self-play. The single central learner, in turn, retrieves the experience generated by our actors, and updates its policy and value network parameters using batched A2C with importance weighting.

Two points to note in our implementation:
\begin{enumerate}
    \item Diplomacy's action space is complex: our network outputs an order for each of the provinces controlled by player $p$ at each turn. Orders for each province are computed one by one, and our ROD module ensures inter-province consistency. We therefore expand the off policy correction coefficients for acting policy $\mu$ and target policy $\pi$ as follows: $\rho = \frac{\pi(a_t | s_t)}{\mu(a_t | s_t)} = \frac{p_{\pi}(a_{t_1} | s_t)p_{\pi}(a_{t_2} | a_{t_1}, s_t)\ldots p_{\pi}(a_{t_k} | a_{t_1},\ldots,a_{t_{k-1}}, s_t)}{p_{\mu}(a_{t_1} | s_t)p_{\mu}(a_{t_2} | a_{t_1}, s_t)\ldots p_{\mu}(a_{t_k} | a_{t_1},\ldots,a_{t_{k-1}}, s_t)}$, where $p_{\pi}(a_{t_l} | a_{t_1},\ldots,a_{t_{l-1}})$ is the probability of selecting order $a_{t_l}$ for unit $l$, given the state of the board at time $t$, $s_t$, and all orders for previous units $a_{t_1},\ldots,a_{t_{l-1}}$, under policy $\pi$.
    
    \item We do not train the value target using TD. Instead, just as in the imitation step of BRPI algorithms, we augment trajectories with returns and use supervised learning directly on this target.
    \item As in \ref{app:rl-settings}, we force draws after random lengths of time, and otherwise only have rewards when games are won. This differs to the A2C agent trained in \cite{paquette2019no}, where a dense reward was given for capturing supply centres.
\end{enumerate}

Fig.~\ref{fig:a2c} shows A2C's performance in 1v6 tournaments against the Imitation Learning starting point (SL). We observe that A2C's win-rate in this setting steadily increases for about 3M steps, and then gradually declines.
In comparisons to other algorithms, we report the performance of the A2C agent with the best win-rate against its Imitation Learning starting point.

\section{Calculation of Confidence Intervals}
\label{app:stats}

The confidence intervals in figure \ref{l_fig_perf_sppi_run} are for the mean scores of 5 different experiments. The confidence interval is for the variation in the means across the random seeds. This is calculated based on a normal distribution assumption with unknown variance (i.e. using a t-distribution with 4 degrees of freedom)~\cite{krishnamoorthy2016handbook}.

The confidence intervals quoted for 1v6 winrate tables are the measurement confidence interval, they only reflect randomness in the outcomes of games between our actual agents, and do not represent differences due to randomness during training the agents. To calculate the confidence interval, we first calculate the confidence interval for the winrate for each combination of singleton agent, 6-country agent, and country the singleton agent plays as. For this confidence interval, we use the Wilson confidence interval (with continuity correction)~\cite{wilson1927probable}. We then combine the confidence intervals using the first method from Waller et al.~\cite{waller1994confidence}. Unlike using a single Wilson confidence interval on all data, this corrects for any imbalance in the distribution over agents or countries played in the data generated, for example due to timeouts or machine failures during evaluation.

\section{Gradient Descent for Finding an \texorpdfstring{$\epsilon$}{epsilon}-Nash Equilibrium in the Meta-Game}
\label{app:nashconv_descent}

Given our definition for an $\epsilon$-Nash, we measure distance from a Nash equilibrium with $\mathcal{L}_{\textrm{exp}}(\boldsymbol{x}) = \sum_i \mathcal{L}_{\textrm{exp}_i}(\boldsymbol{x})$, known as \emph{exploitability}
or \texttt{Nash-conv}, and attempt to compute an approximate Nash equilibrium via gradient descent. By redefining $r^i \leftarrow r^i + \tau \texttt{entropy}(x_i)$, we approach a Quantal Response Equilibrium~\cite{mckelvey1995quantal} (QRE) instead; QRE models players with bounded rationality ($\tau \in [0,\infty)$ recovers rational (Nash) and non-rational at its extremes). Further, QRE can be viewed as a method that performs entropy-regularizing of the reward, which helps convergence in various settings~\cite{perolat2020poincar}. For further discussion of QRE and its relation to convergence of FP methods, see Appendix~\ref{app:theory}. 

Computing a Nash equilibrium is PPAD-complete in general~\cite{papadimitriou1994complexity}, however, computing the relaxed solution of an $\epsilon$-Nash equilibrium proved to be tractable in this setting when the number of strategies is sufficiently small. 

Note that a gradient descent method over $\mathcal{L}_\textrm{exp}(\boldsymbol{x})$ is similar to, but not the same as Exploitability Descent~\cite{lockhart2019computing}. Whereas exploitability descent defines a per-player exploitability, and each player independently descends their own exploitability, this algorithm performs gradient descent on the exploitability of the strategy profile ($\mathcal{L}_\textrm{exp}(\boldsymbol{x})$ above), across all agents.

\section{Theoretical Properties of Fictitious Play}
\label{app:theory}

Appendix~\ref{app:blotto} discusses the relation between the FP variants of BRPI described in Section~\ref{methods:brpi} and Stochastic Fictitious Play~\cite{Fudenberg93SFP} (SFP). We now investigate the convergence properties of SFP, extending analysis done for two-player games to many-player games (3 or more players). Our key theoretical result is that SFP converges to an $\epsilon$-coarse correlated equilibrium ($\epsilon$-CCE) in such settings (Theorem~\ref{sfp_eps_cce}). We make use of the same notation introduced in Appendix~\ref{l_sect_appendix_defs_notation}.

We first introduce some notation we use. A Quantal best response equilibrium (QRE) with parameter $\lambda$ defines a fixed point for policy $\pi$ such that for all $i$
\begin{align*}
    \pi^i(a^i) &= \texttt{softmax}\Big(\frac{r^i_{\pi^{-i}}}{\lambda}\Big) =  \frac{\exp\big(\frac{r^i_{\pi^{-i}}(a_i)}{\lambda}\big)}{\sum \limits_{j} \exp\big(\frac{r^i_{\pi^{-i}}(a_j)}{\lambda}\big)}
\end{align*}
where $\lambda$ is an inverse temperature parameter and $r^i_{\pi^{-i}}$ is a vector containing player $i$'s rewards for each action $a_i$ given the remaining players play $\pi^{-i}$. The softmax can be rewritten as follows:
\begin{align*}
    \frac{\exp\big(\frac{r^i_{\pi^{-i}}(a_i)}{\lambda}\big)}{\sum \limits_{j} \exp\big(\frac{r^i_{\pi^{-i}}(a_j)}{\lambda}\big)} = \argmax \limits_{p^i} \left[ \langle p^i , r^i_{\pi^{-i}} \rangle + \lambda h^i(p^i) \right]
\end{align*}
\noindent where the function $h^i$ is the entropy. The QRE can then be rewritten as:
\begin{align*}
   \max \limits_{p^i} \left[ \langle p^i , r^i_{\pi^{-i}} \rangle + \lambda h^i(p^i) \right] = \langle \pi^i , r^i_{\pi^{-i}} \rangle + \lambda h^i(\pi^i).
\end{align*}
We leverage this final formulation of the QRE in our analysis.

Fictitious Play with best responses computed against this entropy regularized payoff is referred to as \emph{Stochastic} Fictitious Play. Note that without the entropy term, a best response is likely a pure strategy whereas with the entropy term, mixed strategies are more common. The probability of sampling an action $a$ according to this best response is
$$P\left(a = \argmax_{a^i} \frac{r^i_{\pi^{-i}}(a^i)}{\lambda} + \epsilon_i \right)$$ 
\noindent where $\epsilon_i$ follow a Gumbel distribution ($\mu=0$ and $\beta=1$ see the original paper on SFP~\cite{Fudenberg93SFP}).

Finally, let $h(p^i)$ denote the entropy regularized payoff of a policy $p^i$. We use the following shorthand notation for the Fenchel conjugate of $h$: $h^*(y) = \max \limits_{p} \left[ \langle p , y\rangle + h(p) \right]$. Note that $h^*(y) = \langle p^* , y\rangle + h(p^*)$ where $p^* = \argmax \limits_{p} \left[ \langle p , y\rangle + h(p) \right]$ and so therefore, $\frac{d h^*(y)}{dy} = p^*$. This property is used in proving the regret minimizing property of Continuous Time SFP below.

We first recap known results of Discrete Time Fictitious Play and Continuous Time Fictitious Play. As a warm-up, we then prove convergence of Continuous Time Fictitious Play to a CCE.

Next, we review results for Discrete Time Stochastic Fictitious Play and Continuous Time Stochastic Fictitious Play. Lastly, we prove convergence of Continuous Time Stochastic Fictitious Play to an $\epsilon$-CCE.\footnote{This result generalizes the result in the warm-up.}

{\bf Discrete Time Fictitious Play:} Discrete Time Fictitious Play~\cite{robinson1951iterative} is probably the oldest algorithm to learn a Nash equilibrium in a zero-sum two-player game. The convergence rate is $O\big(t^{-\frac{1}{|A_0| + |A_1| - 2}}\big)$~\cite{shapiro1958} and it has been conjectured in~\cite{karlin1959mathematical} that the actual rate of convergence of Discrete Time Fictitious Play is $O(t^{-\frac{1}{2}})$ (which matches~\cite{shapiro1958}'s lower bound in the 2 action case). A strong form of this conjecture has been disproved in~\cite{daskalakis2014counter}.

However Discrete Time Fictitious Play (sometimes referred to as Follow the Leader) is not a regret minimizing algorithm in the worst case (see a counter example in~\cite{de2014follow}). A solution to this problem is to add a regularization term such as entropy in the best response to get the regret minimizing property (\textit{i.e.} a Follow the {\it Regularized} Leader algorithm).

{\bf Continuous time Fictitious Play:} The integral version of Continuous Time FP (CFP) is:
\begin{align*}
    \pi^i_t =  \frac{1}{t}\int \limits_{s=0}^t b^{i}_s ds \quad \textrm{ where $\forall i$ and $t \ge 1$, $b^{i}_t = \argmax_{p^i} \Big\langle p^i , \frac{1}{t}\int_{s=0}^t r^i_{b^{-i}_s} ds \Big\rangle$},
\end{align*}
\noindent with $b^i_t$ being arbitrary for $t<1$.
A straightforward Lyapunov analysis~\cite{harris1998rate} shows that CFP (in two-player zero-sum) results in a descent on the exploitability $\phi(\pi) = \sum_{i=1}^N \max_{p^i} \langle p^i , r^i_{\pi^{-i}} \rangle - r^i_{\pi}$. In addition, CFP is known to converge to a CCE in \textbf{two} player~\cite{ostrovski2013payoff} games.

\emph{Our Contribution:} \textbf{Many-Player CFP is Regret Minimizing $\implies$ Convergence to CCE}

We now extend convergence of CFP to a CCE in \textbf{many}-player games. Let $r^i_s$ be a measurable reward stream and let the CFP process be:
\begin{align*}
    \pi^i_t =  \frac{1}{t}\int \limits_{s=0}^t b^{i}_s ds \quad \textrm{ where $\forall i, b^{i}_t = \argmax_{p^i} \Big\langle p^i , \frac{1}{t} \int \limits_{s=0}^t r^i_{s} ds \Big\rangle$}.
\end{align*}

Section~\ref{eps_cce_from_regret} relates regret to coarse correlated equilibria, so we can prove convergence to a CCE by proving the following regret is sub-linear:
$$Reg((b^i_s)_{s\leq t}) = \max_{p^i} \int \limits_{s=0}^t \langle p^i , r^i_s \rangle ds - \int \limits_{s=0}^t \langle b^{i}_s , r^i_s \rangle ds.$$

For all $t \geq 1$ we have:
\begin{align*}
    \frac{d}{dt} \left[ \max_{p^i}  \Big\langle p^i , \int \limits_{s=0}^t r^i_s ds \Big\rangle \right] &= \Big\langle b^i_t , \frac{d}{dt} \int \limits_{s=0}^t r^i_s ds \Big\rangle = \langle b^i_t , r^i_t \rangle.
\end{align*}

We conclude by noticing that:
\begin{align*}
    &\int \limits_{t=1}^T \frac{d}{dt} \left[ \max_{p^i}  \Big\langle p^i , \int \limits_{s=0}^t r^i_s ds \Big\rangle \right] dt = \int \limits_{t=1}^T \langle b^i_t , r^i_t \rangle dt = \int \limits_{t=0}^T \langle b^i_t , r^i_t \rangle dt - \int \limits_{t=0}^1 \langle b^i_t , r^i_t \rangle dt \\
    &= \max_{p^i} \int \limits_{t=0}^T \langle p^i , r^i_t \rangle dt - \max_{p^i} \int \limits_{t=0}^1 \langle p^i , r^i_t \rangle dt.
\end{align*}

This implies that:
\begin{align*}
     \max_{p^i} \int \limits_{t=0}^T \langle p^i , r^i_t \rangle dt - \int \limits_{t=0}^T \langle b^i_t , r^i_t \rangle dt = \max_{p^i} \int \limits_{t=0}^1 \langle p^i , r^i_t \rangle dt - \int \limits_0^1 \langle b^i_t , r^i_t \rangle dt.
\end{align*}

Hence we have $Reg((b^i_s)_{s\leq t}) = O(1)$. This implies that the average joint strategy $\frac{1}{T} \int \limits_{0}^T b_t dt$ converges to a CCE (where $b_t(a_1, \dots, a_N) = b^1_t(a_1)\times \dots \times b^N_t(a_N)$).

{\bf Discrete time Stochastic Fictitious Play:}
Discrete Time Fictitious play has been comprehensively studied~\cite{fudenberg1998theory}. This book shows that Discrete time Stochastic Fictitious Play converges to an $\epsilon$-CCE (this is implied by $\epsilon$-Hannan consistency).

{\bf Continuous time Stochastic Fictitious Play: }
The integral version of Continuous Time Stochastic FP (CSFP) is:
\begin{align*}
    \pi^i_t =  \frac{1}{t}\int \limits_0^t b^{i}_s ds \quad \textrm{ where $\forall i, b^{i}_t = \argmax_{p^i} \Big\langle p^i , \frac{1}{t}\int_0^t r^i_{b^{-i}_s} ds  + \lambda h^i(p^i)\Big\rangle$}.
\end{align*}

CSFP is known to converge to a QRE and
$\phi_\lambda(\pi) = \sum_{i=1}^{N=2} \max_{p^i} \langle p^i , r^i_{\pi^{-i}} + \lambda h^i(p^i)\rangle - [r^i_{\pi} + \lambda h^i(\pi^i)] = \sum_{i=1}^{N=2} \max_{p^i} \langle p^i , r^i_{\pi^{-i}}  + \lambda h^i(p^i)\rangle - \lambda h^i(\pi^i)$  is a Lyapunov function of the CFP dynamical system~\cite{hofbauer2002global}. Note that the term $\sum_{i=1}^{N=2} r^i_{\pi} = 0$ because this is a zero-sum game.

\emph{Our Contribution:} \textbf{Many-Player CSFP Achieves Bounded Regret $\implies$ Convergence to $\epsilon$-CCE}

We now present a regret minimization property for CSFP in \textbf{many}-player games, which we use to show convergence to an $\epsilon$-CCE.  As before, let $r^i_s$ be a measurable reward stream and let the CSFP process be:
\begin{align}
    \pi^i_t =  \frac{1}{t}\int \limits_0^t b^{i}_s ds \quad \textrm{ where $\forall i, b^{i}_t = \argmax_{p^i} \Big\langle p^i , \frac{1}{t}\int_0^t r^i_{s} ds  + \lambda h^i(p^i)\Big\rangle$}. \label{sfp_pi}
\end{align}

Note, this process averages best responses to the historical sequence of entropy regularized rewards. We seek to show that the regret of this process with respect to the \emph{un}-regularized rewards grows, at worst, linearly in $T$ with coefficient dependent on the regularization coefficient, $\lambda$. The result should recover the standard CFP result when $\lambda=0$.

The proof proceeds by decomposing the regret with respect to the \emph{un}-regularized rewards into two terms: regret with respect to the regularized rewards ($O(1)$) and an $O(T)$ term dependent on $\lambda$ as desired.

Bounding the first term, the regret term, is accomplished by first deriving the derivative of the maximum entropy regularized payoff given historical play up to time $t$ and then recovering the regret over the entire time horizon $T$ by fundamental theorem of calculus.

\begin{theorem}
\label{sfp_eps_cce}
CSFP converges to an $\epsilon$-CCE.
\end{theorem}
\begin{proof}
The regret $Reg((b^i)_{s \le t}) = \max_{p^i} \int \limits_{s=0}^T \left[ \langle p^i ,  r^i_s \rangle \right]ds - \int \limits_{s=0}^T [\langle b^i_s, r^i_s\rangle]ds$ is
\begin{align*}
    &\leq \max_{p^i} \int \limits_{s=0}^T \left[ \langle p^i ,  r^i_s \rangle + \textcolor{blue}{\lambda h^i(p^i) - \lambda h^i(p^i)} \right]ds - \int \limits_{s=0}^T [\langle b^i_s, r^i_s\rangle]ds \\
    &\leq \max_{p^i} \int \limits_{s=0}^T \left[ \langle p^i ,  r^i_s \rangle + \lambda h^i(p^i) \right]ds - \textcolor{blue}{T \lambda \min_{p^i} h^i(p^i)} - \int \limits_{s=0}^T [\langle b^i_s, r^i_s\rangle]ds \\
    &\leq \max_{p^i} \int \limits_{s=0}^T \left[ \langle p^i ,  r^i_s \rangle + \lambda h^i(p^i) \right]ds - T \lambda \min_{p^i} h^i(p^i) - \int \limits_{s=0}^T [\langle b^i_s, r^i_s\rangle]ds + \int \limits_{s=0}^T [\textcolor{blue}{\lambda h^i(b^i_s) - \lambda h^i(b^i_s)}]ds \\
    &\leq \max_{p^i} \int \limits_{s=0}^T \left[ \langle p^i ,  r^i_s \rangle + \lambda h^i(p^i)\right]ds   - \int \limits_{s=0}^T [\langle b^i_s, r^i_s\rangle + \textcolor{blue}{\lambda h^i(b^i_s)}]ds + \int \limits_{s=0}^T \lambda h^i(b^i_s) ds - T \lambda \min_{p^i} h^i(p^i) \\
    &\leq \max_{p^i} \int \limits_{s=0}^T \left[ \langle p^i ,  r^i_s \rangle + \lambda h^i(p^i)\right]ds   - \int \limits_{s=0}^T [\langle b^i_s, r^i_s\rangle + \lambda h^i(b^i_s)]ds + T \lambda [\textcolor{blue}{\max_{p^i} h^i(p^i)} - \min_{p^i} h^i(p^i)] \\
    &\stackrel{L\ref{qre_o1}}{\le} O(1) + T \lambda [\textcolor{blue}{\max_{p^i} h^i(p^i)} - \min_{p^i} h^i(p^i)].
\end{align*}

Section~\ref{eps_cce_from_regret} relates regret to coarse correlated equilibria. We thus conclude that CSFP converges to an $\epsilon$-CCE with $\epsilon \leq \lambda [\max_{p^i} h^i(p^i)- \min_{p^i} h^i(p^i)]$.

\end{proof}

\begin{lemma}
\label{dt_max}
The derivative of the max entropy regularized payoff up to time $t$ is given by the entropy regularized payoff of the best response:
$\frac{d}{dt} \max_{p^i} \int \limits_{s=0}^t \left[ \langle p^i ,  r^i_s \rangle + \lambda h^i(p^i) \right] ds = \langle b^i_t, r^i_t\rangle + \lambda h^i(b^i_t)$.
\end{lemma}

\begin{proof}
The derivative is derived by first computing the derivative of the maximum payoff considering the average reward (effectively higher entropy regularization):
\begin{align}
    &\frac{d}{dt} \max_{p^i} \left[ \langle p^i , \left[\frac{1}{t}\int \limits_{s=0}^t r^i_s ds \right] \rangle + \lambda h^i(p^i) \right] = \lambda \frac{d}{dt} \max_{p^i} \left[  \langle p^i , \left[\frac{1}{\lambda t}\int \limits_{s=0}^t r^i_s ds \right] \rangle + h^i(p^i) \right] \label{dtmax}
    \\ &=\lambda \frac{d}{dt} h^{*i}(\frac{1}{\lambda t}\int \limits_{s=0}^t r^i_s ds) \stackrel{\textcolor{blue}{{\frac{dh^{*i}(y)}{dy} \frac{dy}{dt}}}}{=} \lambda \langle b^i_t, \frac{d}{dt} \left[\frac{1}{\lambda t}\int \limits_{s=0}^t r^i_s ds\right] \rangle = \langle b^i_t, \frac{d}{dt} \left[\frac{1}{t}\int \limits_{s=0}^t r^i_s ds\right] \rangle \nonumber
    \\ &= \langle b^i_t, \left[-\frac{1}{t^2}\int \limits_{s=0}^t r^i_s ds + \frac{1}{t}r^i_t\right] \rangle = \frac{1}{t} \left[ \langle b^i_t, r^i_t \rangle - \langle b^i_t, \left[\frac{1}{t}\int \limits_{s=0}^t r^i_s ds\right] \rangle \right] \nonumber
    \\ &= \frac{1}{t} \left[ \langle b^i_t, r^i_t\rangle + \lambda h^i(b^i_t) - \langle b^i_t, \left[\frac{1}{t}\int \limits_{s=0}^t r^i_s ds\right] \rangle - \lambda h^i(b^i_t) \right] \nonumber
    \\ &= \frac{1}{t} \left[ \langle b^i_t, r^i_t\rangle + \lambda h^i(b^i_t) - \max_{p^i} \left[ \langle p^i , \left[\frac{1}{t}\int \limits_{s=0}^t r^i_s ds \right] \rangle + \lambda h^i(p^i) \right] \right] \label{dtmax_alt}
\end{align}
where we highlight the use of a special property of the Fenchel conjugate in the second line.

Rearranging~\eqref{dtmax} and~\eqref{dtmax_alt} and then multiplying both sides by $t$ we find:
\begin{align*}
    &t\frac{d}{dt} \max_{p^i} \left[ \langle p^i , \left[\frac{1}{t}\int \limits_{s=0}^t r^i_s ds \right] \rangle + \lambda h^i(p^i) \right] + \max_{p^i} \left[ \langle p^i , \left[\frac{1}{t}\int \limits_{s=0}^t r^i_s ds \right] \rangle + \lambda h^i(p^i) \right] =  \langle b^i_t, r^i_t\rangle + \lambda h^i(b^i_t).
\end{align*}

As a result we obtain:
\begin{align*}
    & \frac{d}{dt} \max_{p^i} \int \limits_{s=0}^t \left[ \langle p^i ,  r^i_s \rangle + \lambda h^i(p^i) \right] ds\\
    &= \frac{d}{dt} \frac{t}{t} \max_{p^i} \int \limits_{s=0}^t \left[ \langle p^i ,  r^i_s \rangle + \lambda h^i(p^i) \right] ds\\
    &= \frac{d}{dt} \frac{t}{t} \max_{p^i} \left[ \Big\langle p^i ,  \int \limits_{s=0}^t r^i_s ds \Big\rangle + \int \limits_{s=0}^t \lambda h^i(p^i) ds \right]\\
    &= \frac{d}{dt} t \max_{p^i} \left[ \Big\langle p^i , \frac{1}{t} \int \limits_{s=0}^t r^i_s ds \Big\rangle + \frac{1}{t} \int \limits_{s=0}^t \lambda h^i(p^i) ds \right]\\
    &= \frac{d}{dt} t \max_{p^i} \left[ \Big\langle p^i , \frac{1}{t} \int \limits_{s=0}^t r^i_s ds \Big\rangle + \lambda h^i(p^i) \right]\\
    &=t \frac{d}{dt} \max_{p^i} \left[ \langle p^i , \left[\frac{1}{t}\int \limits_{s=0}^t r^i_s ds \right] \rangle + \lambda h^i(p^i) \right]+ \left[\frac{d}{dt} t\right]\max_{p^i} \left[ \langle p^i , \left[\frac{1}{t}\int \limits_{s=0}^t r^i_s ds \right] \rangle + \lambda h^i(p^i) \right]\\
    &= \langle b^i_t, r^i_t\rangle + \lambda h^i(b^i_t).
\end{align*}
\end{proof}

\begin{lemma}
\label{qre_o1}
The regret of the stochastic best response process with respect to the entropy regularized payoffs,
\begin{align*}
    &\max_{p^i} \int \limits_{s=0}^T \left[ \langle p^i ,  r^i_s \rangle + \lambda h^i(p^i) \right]ds - \int \limits_{s=0}^T [\langle b^i_s, r^i_s\rangle + \lambda h^i(b^i_s)]ds, \quad \text{ is $O(1)$.}
\end{align*}
\end{lemma}

\begin{proof}
Integrating Lemma~\ref{dt_max} from $1$ to $T$ and decomposing we find
\begin{align*}
    \int_1^T \frac{d}{dt} [\max_{p^i} \int \limits_{s=0}^t \left[ \langle p^i ,  r^i_s \rangle + \lambda h^i(p^i) \right] ds]dt &= \int_1^T [\langle b^i_t, r^i_t\rangle + \lambda h^i(b^i_t)] dt\\
    &= \int \limits_{s=0}^T [\langle b^i_s, r^i_s\rangle + \lambda h^i(b^i_s)]ds - \int \limits_{s=0}^1 [\langle b^i_s, r^i_s\rangle + \lambda h^i(b^i_s)]ds.
\end{align*}
Also, note that, by fundamental theorem of calculus, the integral can also be expressed as
\begin{align*}
    &\int_1^T \frac{d}{dt} [\max_{p^i} \int \limits_{s=0}^t \left[ \langle p^i ,  r^i_s \rangle + \lambda h^i(p^i) \right] ds]dt\\
    &=\max_{p^i} \int \limits_{s=0}^T \left[ \langle p^i ,  r^i_s \rangle + \lambda h^i(p^i) \right]ds - \max_{p^i} \int \limits_{s=0}^1 \left[ \langle p^i ,  r^i_s \rangle + \lambda h^i(p^i) \right]ds.
\end{align*}

Rearranging the terms relates the regret to a definite integral from $0$ to $1$
\begin{align*}
    &\max_{p^i} \int \limits_{s=0}^T \left[ \langle p^i ,  r^i_s \rangle + \lambda h^i(p^i) \right]ds - \int \limits_{s=0}^T [\langle b^i_s, r^i_s\rangle + \lambda h^i(b^i_s)]ds
    \\ &= \max_{p^i} \int \limits_{s=0}^1 \left[ \langle p^i ,  r^i_s \rangle + \lambda h^i(p^i) \right]ds - \int \limits_{s=0}^1 [\langle b^i_s, r^i_s\rangle + \lambda h^i(b^i_s)]ds
\end{align*}
which is $O(1)$ with respect to the time horizon $T$.
\end{proof}

\section{Size Estimates for Diplomacy}
\label{app:size-estimates}

One of the issues that make Diplomacy a difficult AI challenge is the sheer size of the game. We estimate the size of the game of Diplomacy based on No-Press games in the human dataset~\cite{paquette2019no}. This dataset consists of 21,831 games. We play through each game, and inspect how many legal actions were available for each player at each turn of the game.

Some of the games in the dataset are unusually short, with a draw agreed after only a couple of turns. This is usually done to cancel a game on websites without the functionality to do so. We do not attempt to filter such games from our data; as a result, the size estimates here are biased downward.

\subsection{Legal Joint Actions per Turn}

Diplomacy turns are in one of three phases: the \textit{movement} (or \textit{Diplomacy}) phase, the \textit{retreats} phase and the \textit{adjustments} phase. The majority of gameplay is in the movement phase, while the retreats and adjustments phases mostly handle the effects of the movement phase. So we consider the size of this movement phase.

In the first turn of diplomacy, there are 22 units on the board, 3 for most players and 4 for Russia. Each unit has approximately 10 legal actions, which can be selected independently of one another. The total number of possibilities for this first turn is $10^{22.3}$. 

As the game progresses, the number of units on the board increases to a maximum of 34. Additionally, when units are closer together, there are more opportunities to play \textit{support} moves and as a result the number of legal actions per unit grows. The largest movements phase in our data had a total of $10^{64.3}$ legal action combinations across the 7 players. The median number of possibilities in movement phases is $10^{45.8}$.

Many of these different combinations of actions lead to the same states, for example the action \MAR~{\bf support} \PAR $\rightarrow$ \BUR has no affect on the adjudication if the owner of the unit is Paris doesn't take the action \PAR $\rightarrow$ \BUR, or if the movement to \BUR is unopposed, or if another unit moves to \MAR, cutting the support.

\subsection{Estimate of the Game Tree Size}

We estimate the size of the game tree from a single game by considering the product of the number of legal actions available at each turn of the game. For example, for a game where there were 4 options on the first turn, 2 options on the second, 3 on the third, we estimate the size as $4 \times 2 \times 3 = 12$. The median size was $10^{896.8}$.

Note that as Diplomacy has no repetition rule or turn limit, the game tree is technically infinite. The purpose of our estimate is to give a rough sense of the number of possibilities the agents must consider in any given game. We report the median as the arithmetic mean is dominated by the largest value, $10^{7478}$, which comes from an exceptionally long game (that game lasted 157 movement phases, whereas the median game length was 20 movement phases). When long games are truncated to include only their first 20 movement phases, the median size is $10^{867.8}$, and the maximum is $10^{1006.7}$.

\end{document}